\documentclass{article}

\usepackage[nonatbib, final]{neurips_2022}
\usepackage[utf8]{inputenc} % allow utf-8 input
\usepackage[T1]{fontenc}    % use 8-bit T1 fonts
\usepackage{hyperref}       % hyperlinks
\usepackage{url}            % simple URL typesetting
\usepackage{booktabs}       % professional-quality tables
\usepackage{amsfonts}       % blackboard math symbols
\usepackage{nicefrac}       % compact symbols for 1/2, etc.
\usepackage{microtype}      % microtypography
\usepackage{soul}           %  hyphenatable spacing out (letterspacing), underlining, striking out, etc.,

\usepackage{setting}
\usepackage{setting2}
\graphicspath{ {Figures/} }

\title{Integral Probability Metrics PAC-Bayes Bounds}

\author{
  Ron Amit\\
  Technion - Israel Institute of Technology\\
  The Viterbi Faculty of Electrical Engineering\\
  \texttt{ronamit@campus.technion.ac.il} \\
  \And
  Baruch Epstein\\
  Technion - Israel Institute of Technology\\
  The Viterbi Faculty of Electrical Engineering\\
  \texttt{baruch.epstein@gmail.com} \\
  \And
  Shay Moran\\
  Technion - Israel Institute of Technology\\
 Faculty of Mathematics\\
 The Taub Faculty of
Computer Science\\
Google Research, Israel\\
  \texttt{smoran@technion.ac.il} \\
\And
  Ron Meir\\
  Technion - Israel Institute of Technology\\
  The Viterbi Faculty of Electrical Engineering\\
  \texttt{rmeir@ee.technion.ac.il} \\
}

\begin{document}

\maketitle

\begin{abstract}
We present a PAC-Bayes-style generalization bound which enables the replacement of the KL-divergence with a variety of \emph{Integral Probability Metrics} (IPM). We provide instances of this bound with the IPM being the \emph{total variation metric}
and the \emph{Wasserstein distance}. A notable feature of the obtained bounds is that they naturally interpolate
between classical uniform convergence bounds in the worst case (when the prior and posterior are far away from each other), and improved bounds in favorable cases (when the posterior and prior are close). This illustrates the possibility of reinforcing classical generalization bounds with algorithm- and data-dependent components, thus making them more suitable to analyze algorithms that use a large hypothesis space.
\end{abstract}

\section{Introduction and Related Work}

Classical statistical learning theory is based on a worst-case perspective 
which can be too pessimistic to model practical machine learning. 
In reality, data is rarely worst-case, and experiments demonstrate learning tasks that
are solved with much less data than predicted by traditional theory. A primary manifestation of the traditional worst-case perspective is demonstrated by {\it uniform convergence} (UC); 
 a genre of generalization bounds which form the backbone of the classical theory \citep{vapnik1999nature}. 
    These bounds guarantee that the generalization gap of \emph{all} hypotheses in the output-space of the algorithm
    \emph{simultaneously} vanish as the training-set size grows.
    The key algorithmic insight these bounds provide is summarized by the \emph{Empirical Risk Minimization} principle
    (ERM), which asserts that it suffices to output \emph{any} hypothesis in the class which minimizes the empirical risk.

Consequently, UC arguments provide non-trivial guarantees only if the hypothesis class used by the algorithm is {\it restricted} (e.g.\ has low Rademacher complexity or bounded VC dimension).
    In contrast, practical learning approaches such as deep learning algorithms use huge hypothesis classes
    whose VC dimensions rapidly increase with the size and depth of the underlying network.
    Hence, the rate guaranteed by UC arguments is often much slower than the rate observed in practice \citep*{Rethinking_Generalization, neyshabur2017exploring, nagarajan2019uniform, bachmann2021uniform}.

A further shortcoming of UC bounds, and the associated ERM principle, is that they are algorithm- and data-independent;\footnote{More precisely,
UC bounds only depend on the hypothesis space.}
    that is, they do not utilize beneficial properties of the data and/or the algorithm.
    For example, in practice, regularized algorithms often perform better than Empirical Risk Minimization, 
    but this cannot be expressed by UC bounds and the ERM principle.

The PAC-Bayes (PB) framework is a  prominent example of a theoretical framework that does not require the UC property.
This framework was pioneered by \citet{Shawe-Taylor_Williamson_First_PAC-Bayes} and \cite{McAllester_98} and developed in later papers, e.g.\ \citep{McAllester_03, Seeger_02, Maurer_04, catoni2007pac, lever2013tighter}; see \cite{guedj2019primer} and \citet{alquier2021user} for extensive surveys. 
\cite{Begin_16} introduced a general strategy that produces PB bounds from change-of-measure inequalities leading to bounds based on the Rényi’s $\alpha$-divergence, and \citet{alquier2018simpler, ohnishi2021novel, picard2022change} further extended PB bounds to other Csiszár’s $f$-divergences.

% These bounds provide an explicit and flexible trade-off between model complexity and empirical data fit. Moreover, by depending on a prior distribution over the hypothesis space, they naturally call for incorporating prior knowledge into the bound.

PB theorems consider the generalization performance of stochastic predictors. These bounds are non-uniform\footnote{I.e.,~PB bounds apply even in learning problems without uniform convergence (Definition \ref{def:UC}).} by nature, and are algorithm and data-dependent.
They are usually based on a complexity term that depends on the Kullback-Leibler (KL) divergence between a data-dependent posterior distribution and a data-independent prior distribution.\footnote{But see \citet{rivasplata2020pac} for data-dependent priors.}

There are additional notable works on data and algorithm-dependent guarantees. 
The classical work of \citet{bousquet2002stability} and \citet{xu2012robustness} studied generalization guarantees that depend on data and algorithm-dependent stability measures. 
A further line of recent papers tries to incorporate noise robustness/resilience. In \citet{miyaguchi2019pac}, a PAC-Bayes transportation bound is used to measure the contribution of randomization to PB. This is done via optimal transport and Lipschitzness, based on the usual KL-PB bound.
The work of \citet{wei2019data} uses data-dependent Lipschitz smoothness to improve margin bounds, and \citet{nagarajan2019deterministic} passes from standard PB to a deterministic bound by assuming noise-resilience on the training data. This property translates to the test data, implying that good training smoothness leads to good test smoothness. Finally, \citet{yang2019fast} measure data-dependent smoothness around each hypothesis (for each sample) and merge Catoni's bound \citep{catoni2007pac} with Rademacher theory, to obtain fast rates.

Recent work by \citet{lopez2018generalization, wang2019information,rodriguez2021tighter, zhang2021optimal, aminian2022tighter} and \citet{neu2022generalization} proved information-theoretic bounds on the \emph{expected} generalization gap using the Wasserstein and the
total-variation (TV) distances.
Our work is within the PB framework, and therefore enjoys the following advantages: \textit{(i)} The bounds are ``in high probability''{{\placeholdera} over the sample} rather than in expectation. \textit{(ii)} PB bounds are sample dependent, i.e.,~bound the generalization gap for a specific sample-dependent posterior, while information-theoretic bounds are formulated as expectation over all sample sets, thereby providing a basis for empirical algorithms, e.g., \citet{alquier2021user, dziugaite2017computing}. 
 \textit{(iii)} The reference measure in PB can be any sample-independent distribution, while information-theoretic bounds consider a specific reference.
Our work introduces uniform convergence assumptions, while 
the above-mentioned papers each used different assumptions. 
Recently, \citet{chee:hal-03262687} proposed PB bounds with the entropy regularized optimal transport distance for an online-learning setting with a finite class.

The optimal transport interpretation of the Wasserstein distance has been used recently in other contexts to derive generalization bounds.
\citet{chuang2021measuring} proposed a bound that uses a data-dependent complexity measure, evaluated via the Wasserstein distance of independently sampled subsets of the training data in the feature space. \citet{hou2022instance} used the principles of optimal transport to derive an instance-based bound based on the local Lipschitz regularity of the learned prediction function in the data space.

In the modern deep learning regime, measures of the hypothesis class complexity used in UC bounds, such as the VC dimension or Rademacher complexity, are
enormous, making the bounds extremely loose for any reasonable number of samples, as opposed to PB bounds 
\citep{dziugaite2017computing, jiang2019fantastic}. However, these complexity measures often have closed-form formulas for models such as neural networks, which show explicitly the effect of the model architecture (number of layer, activation functions etc.). This in contrast to PB bounds, in which the dependence on the hypothesis class is less explicit (but see \citet{anthony1999neural,neyshabur2015norm} for exceptions for neural networks). 
Therefore, we believe that extension of UC bounds to incorporate data- and algorithm-dependence can facilitate the design of better performing architectures. A further advantage of PB bounds is their non-uniformity (the generalization gap bound depends on the learning output), hence we can use the bound as a minimization objective for a structural minimization algorithm, where the complexity term acts as a regularizer. 
In cases where the hypothesis class is very large, but we have some prior knowledge on which hypotheses are more likely to have low population loss (e.g.\ prefer simpler hypothesis as suggested by Occam’s Razor), then in PB one can inject this knowledge as the prior distribution, effectively lowering the generalization bound. 

% \todo{Reviewer 2 - 2x57 - (put this here or later? in the discussion?) We agree that utilizing UC bounds in the over-parameterized regime may lead to very
% loose bounds. Thank you for the reference to (Jiang et al. 2020), it clearly demonstrates
% this point. We will highlight this in a future revision.}

% \todo{Reviewer 2 - 2x57 - We will improve the literature review on algorithm-dependent guarantees. Thank you
% for the suggestion.}

Can the rich theory of UC bounds be extended to help explain generalization with modern large scale models? Can this theory be used to prove data and algorithm dependent guarantees? 
In this paper, we take a step in the direction of answering these questions positively.
To achieve this goal, we show a new technique to incorporate UC bounds within the PAC-Bayes framework.
We prove new PB bounds with Integral Probability Metric (IPM) \citep{muller1997stochastic} to measure distances between distributions, rather than the standard KL or $f$-divergences used so far. 
Specifically, we focus on utilizing two specific IPMs: the total variation and Wasserstein metrics
This IPM framework allows greater flexibility, as it does not require the support of the posterior to be a subset of the prior's support (absolute continuity) as in standard KL-PB bounds, and it applies to deterministic as well as stochastic prior and posterior distributions.
In fact, the IPM-based PB bounds we introduce match, at worst, the rate of the UC bound used. Recently,
\citet{livni2020limitation} showed that the classical KL-PB theorem \underline{cannot} imply meaningful distribution-free generalization bounds for 1-dimensional linear classification. In contrast, our derived IPM-PB bounds do imply such bounds, because linear classifiers satisfy uniform convergence.

% However, since linear classifiers satisfy UC, our derived IPM-PB bounds  

% in a simple 1-dimensional linear classification problem, there exists a data distribution where the KL-PB bound is arbitrarily large. However, the UC bound, and therefore the derived IPM-PB bounds are bounded. Re

We note that the work of \citet{Audibert03PAC, audibert2007combining} showed a different approach to utilizing the UC assumption to derive PAC-Bayes bounds. Their work assumed a UC property to utilize the generic chaining technique, resulting in more refined, variance-sensitive, bounds.
In contrast to our work, their bound is not fully empirical, and the assumed UC bound is not used by the resulting bound.

The Total Variation PAC-Bayes (TVPB) bound (Thm.~\ref{thm:TVPB}) applies in any setting where uniform convergence (UC) (Def. \ref{def:UC}) holds, an assumption that is satisfied by many natural learning problems. For example, in binary problems, UC holds with a rate of $O(\sqrt{\nicefrac{\VC}{m}})$, where $\VC$ is the VC dimension \citep{vapnik1999nature}.
As observed in practice, for large models of deep neural networks with very large VC dimensions the learning rate on natural datasets is often much faster than predicted UC bounds.
To explain this gap, we must turn to data and algorithm-dependent bounds.
The TVPB improves the gap bound to be $O\left( \sqrt{\nf{\VC\DTV(Q,P)}{m}}\right)$, effectively multiplying the VC dimension by the total variation distance of the posterior from the prior. Intuitively, simpler posteriors (closer to the prior assumed before observing the data) lead to a better generalization gap.
Compared to the vanilla KL-PB, the TV distance can be small even in cases where the KL distance can be very large, and in any case, the bound only improves over the original UC bound. In addition, the TV-PB bounds incorporate properties of $\calH$ via the VC dim. We also explore settings where the generalization gap function exhibits a certain smoothness property, and show a PB bound with the Wasserstein metric (Thm. \ref{thm:WPB}),
We analyze this smoothness property and show an explicit Wasserstein-based bound in the finite hypothesis class setting and in a linear regression setting.
In the latter setting, we show that a standard UC bound can be improved by a factor of $O(\sqrt{W_1(Q,P)})$, where $W_1(Q,P)$ is the $1^{\mbox{st}}$ order Wasserstein distance between posterior and prior over the unit-sphere. 
{{\placeholdera}We conduct a numerical simulation to demonstrate the improvement of the Wasserstein PB bound over the UC bound, and, in cases of narrow prior distributions, over the KL-PB bound. The experiment also investigates the case of non-randomized predictors by setting the prior and posterior as Dirac delta measures, which $f$-divergence based PB bounds are unable to use.}  

Figure \ref{fig:Diagram} illustrates graphically the organization of the claims in the paper.

\begin{figure}[t] \centering{
\includegraphics[scale=0.75]{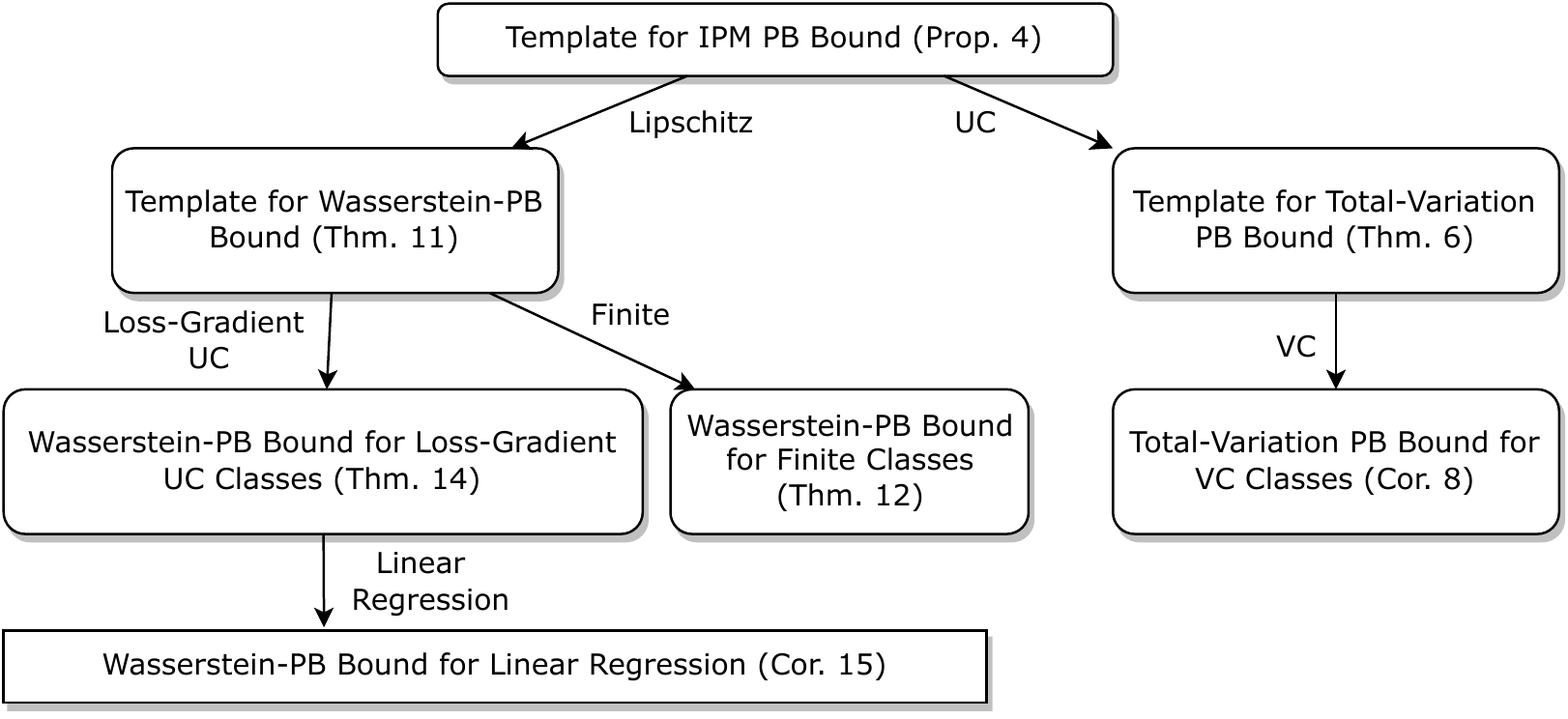}}
	\caption{Claims tree diagram. $A \to B$ means that claim $B$ is a special case of claim $A$, under additional assumptions on the learning problem. For full description of the assumptions, see the corresponding claims.}
	\label{fig:Diagram}
\end{figure}

 \section{Preliminaries} \label{sec:Prelim}
%%%%%%%%%%%%%%%%%%%%%%%%%%%%%%%%%%%%%%%

%%%%%%%%%%%%%%%%%%%%%%%%%%%%%%%%%%%%
\subsection{The Learning Problem}
%%%%%%%%%%%%%%%%%%%%%%%%%%%%%%$$

We begin with a short description a standard supervised learning task. 
% \textcolor{red}{We never really consider the definition of PAC learning. What appears below is the standard (hypothesis-class based) definition of a general supervised learning task with losses in $[0,1]$. I would remove the reference to PAC learning.}
Consider a domain $\calZ$, \footnote{This formulation allows for greater generality than the standard $\mathcal{Z}=\mathcal{X}\times\mathcal{Y}$
and mis-classification loss setting. In particular, it can describe
a number of unsupervised learning problems (\cite{Seldin_Tishby_Unsup_PB}).}
a distribution $\calD$ over $\calZ$, a hypothesis set $\calH$, and finally, a loss function $\ell:\mathcal{H}\times\mathcal{Z}\rightarrow [0,1]$.
% \textcolor{red}{I am not sure we are expected to use this measure theoretic language, but if we do then I think it is required that the loss function is measurable w.r.t to the product of the sigma algebras (which is a sigma algebra over $\mathcal{H}\times\mathcal{Z}$). I personally would probably not use this language here. Maybe we can find a reference to some book/paper which contains standard measure theoretic assumptions.}
The tuple $(\calD, \calH, \ell)$ defines a learning problem:
The learning algorithm receives as input a training set $S=\left\{ z_{i}\right\} _{i=1}^{m}\in \mathcal{Z}^m$ sampled i.i.d
from $\calD$ and selects an hypothesis $h \in \calH$.
The performance of $h$ is measured by the \textit{expected risk},  $L_{\calD}(h)\defeq\bbE_{z\sim\calD} \ell\left(h,z\right)$.
While the expected risk is unavailable to algorithm, the \textit{empirical risk}, $\Lhat_{S}(h) \defeq{\displaystyle }\frac{1}{m}\sum_{i=1}^{m} \ell \left(h,z_{i}\right)$, can be evaluated using the training data.
The generalization gap is defined by $\Delta_{S}(h) \defeq L_{\calD}(h)-\Lhat_{S}(h)$.

Several of our results will assume that the learning problem $(\calD, \calH, \ell)$ satisfies the \textit{uniform convergence property}; i.e.\ the existence of a uniform upper bound on the generalization gap which applies simultaneously to all hypotheses in $\calH$.

\begin{defn}[Uniform convergence, \citep{vapnik2015uniform}] \label{def:UC}
The learning problem $(\calD, \calH, \ell)$ satisfies the uniform convergence property, if there exists a bound function $\uc\br{m, \delta'} > 0$ s.t.\ for any $m \in \N^+$, $\delta \in (0,1)$ we have
 \begin{equation} 
    \Prob \brc*{\abs{\Delta_S(h)} \leq \uc\br{m, \delta}, \forall h \in \calH} \geq 1 - \delta \quad \text{ and } \quad \uc\br{m, \delta} \underset{m \to \infty}{\to} 0.
 \end{equation}
\end{defn}
UC type bounds are a major part of the foundations of theoretical machine learning.
Unfortunately, they 
suffers from a few drawbacks. First, currently known bounds
tend to be extremely loose in many cases, most notably
for deep networks. Second, the setup does not provide a natural way to encode prior knowledge into the bounds, particularly when dealing
with deep networks - the hypothesis set is usually rich enough to
express all relevant functions, and the training algorithms
that might utilize some prior knowledge are not themselves a part
of UC based bounds. 
Finally, UC bounds are usually not data-dependent, a property which is critical to explain the generalization of DNNs on real-world data  \citep{Rethinking_Generalization, nagarajan2019uniform}. %

%%%%%%%%%%%%%%%%%%%%%%%%%%%%%%%%%%%%
\subsection{PAC-Bayes Bounds}
%%%%%%%%%%%%%%%%%%%%%%%%%%%%%%$$

% PAC-Bayes bounds study the performance of learning algorithm that define post.

Let $\calM(\calH)$ denote the set of all probability measures over $\calH$.
For any probability measure $Q \in \calM(\calH)$, we define the expected loss, empirical loss and generalization gap by
\begin{align}
L_{\calD}(Q) & \defeq \Expct hQL_{\calD}(h) \quad;\quad 
\Lhat_{S}(Q) \defeq \Expct hQ \Lhat_{S}(h)\quad ;\quad
\Delta_{S}(Q) \defeq L_{\calD}(Q)-\Lhat_{S}(Q).
\end{align}
PAC-Bayes theory bounds the expected loss, simultaneously for all ``posterior" (sample-dependent) probability measures $Q \in \calM(\calH)$, with high probability over the samples $S \sim \calD^m$,
given any ``prior" (sample-independent) probability measure $P \in \calM(\calH)$.
A key feature of most PAC-Bayes bounds is their dependence
on the KL divergence between the two distributions $P,Q$, $\KL QP \defeq \int_{\calH} \ln(\frac{\dd Q}{\dd P}) \dd Q$, where $\frac{\dd Q}{\dd P}$ is the Radon–Nikodym derivative of $Q$ w.r.t.~$P$.
While KL is a natural measure of divergence between probability distributions,
it restricts the applicability of the resulting bounds to cases where
the support of $Q$ is contained in the support of $P$.
The following bound was introduced by \cite{McAllester_98}.
\begin{prop}[Classical KL-PB Bound]
\label{prop:PACBayesBound}
For any prior $P \in \calM(\calH)$ and $\delta\in (0,1)$, with probability
at least $1-\delta$ over the samples $S \sim \calD^m$, for all $Q \in \calM(\calH)$, we have
\begin{align}
\Delta_S(Q) \le \sqrt{\frac{\KL QP+\ln(\nf{m}{\delta})}{2(m-1)}}.
\end{align}
% In fact, we also have a stronger result (that leads to the result above via Jensen's inequality)
% \begin{align} \label{eq:ExDelta2upbound}
% \Expct hQ\left(\Delta_S^{2}(h)\right) \le \frac{\KL{Q}{P} + \ln(\nf{m}{\delta})}{2(m-1)}.
% \end{align}
\end{prop}

\section{A Template for IPM PAC-Bayes Bounds}
 %====================================================================
 %====================================================================
 \subsection{General IPM-PB Bound}
 %====================================================================

\begin{defn}[Integral Probability Metric] \citep{muller1997stochastic, sriperumbudur2009integral, sriperumbudur2012empirical} \label{def:IPM}
The Integral Probability Metric (IPM) between two probability measures $P$ and $Q$ over $\calH$ is defined as
\begin{align}  
    \gamma_\calF(Q, P) \defeq \sup_{f \in \calF} \abs*{\int_\calH f \dd P - \int_\calH f \dd Q} 
    = \sup_{f \in \calF} \abs*{\Expct{h}{P}\brs*{f(h)} -  \Expct{h}{Q}\brs*{f(h)}} ,
\end{align}
where $\calF$ is a set of real-valued bounded functions $\calH \to \R$.
\end{defn}

By definition, IPM distance measures are symmetric and non-negative.
Note that the KL-divergence is not a special case of IPM, rather it belongs to the family of $f$-divergences, that intersect with IPM only at the Total-Variation \citep{sriperumbudur2009integral, sriperumbudur2012empirical}.

The following proposition assumes that for any fixed sample $S' \in \calZ^m$, the function $f_{S'}(h) \defeq 2(m-1)\Delta_{S'}^{2}(h)$ is a member of a family of functions that depend on the sample, denoted $\calF_{S'}$.
Thus, the IPM-PB bound allows us to `convert' some knowledge we have about the properties of the generalization gap function $\Delta_S(h)$ to a generalization bound.

Since we do not specify yet the collection of function families $\brc*{\calF_S}_{S \in \calZ^m}$, the bound does not convey an explicit rate, and it should rather be seen as a \textbf{template}.
In the next sections, we will derive explicit bounds with specific IPMs divergences. Namely, we will derive a total-variation distance based bound by selecting a collection of bounded function sets, and a Wasserstein distance based bound, by selecting a collection of Lipschitz function sets.

\begin{prop}[Template for IPM PB Bound]
\label{prop:PACBayesBound_IPM}
For any fixed dataset $S' \in \calZ^m, m \in \N^+$, let $\calF_{S'}$ be a family of bounded and measurable functions $\calH \to \R$.
Assume that
 for any number of samples, $m$, and sample $S' \in \calZ^m$, the function $2(m-1)\Delta_{S'}^{2}(\cdot)$ is in $\calF_{S'}$. 
Then for any prior $P \in \calM(\calH)$ and $\delta\in (0,1)$, with probability
at least $1-\delta$ over the samples $S \sim \calD^m$, for all $Q \in \calM(\calH)$, we have
\begin{align}
\Delta_S(Q) \leq& \sqrt{\frac{\gamma_{\calF_S}(Q, P)  +\ln(\nf{m}{\delta})}{2(m-1)}}. 
\end{align}
\end{prop}

The proof is in Appendix \ref{sect:proof_thm:PACBayesBound_IPM}. 
The main idea is to use the IPM definition and the assumption as a change-of-measure inequality, instead of the variational formula by \citet{Donsker_Varadhan_75}, which is used in the classical KL-PB bound proof. 
The rest of the proof is similar to the classical derivation \citep{McAllester_03, SSS}. 

Note that, similarly to the classical KL-PB bound, Proposition \ref{prop:PACBayesBound_IPM} does not require the UC property to hold.
However, in the next sections we will see that assuming an existence of a UC bound $\uc\br{m,\delta}$, and selecting a particular collection of function families $\brc*{\calF_S}_{S \in \calZ^m}$, will result in explicit bounds that can improve upon the worst-case nature of the original $\uc\br{m,\delta}$ bound. 
\subsection{Template for Seeger Type IPM PAC-Bayes Bound}
%%%%%%%%%%%%%%%%%%%%%%%%%%%%%%%%%%%%%%%%%%%%%%%%%%%%%%%%%%%%
The work of \cite{Seeger_02} and
\citet{Maurer_04} presented a different form of the PAC-Bayes theorem with fast $O(\nf{1}{m})$ rate as the dominant term, if the empirical risk is low.
% , bounding
% a ``KL generalization gap'' rather than $\Delta_{S}(h)$

We denote by $\kld{p}{q}$ the KL divergence between two Bernoulli distributions $\calB(p)$ and $\calB(q)$,
$p,q \in [0,1]$, that is, 
\begin{align}
  \kld{p}{q} \defeq
     \begin{cases}
       p\ln \br*{\frac{p}{q}} + (1-p)\ln \br*{\frac{1-p}{1-q}}, &  \text{if $q \notin \brc{0,1}$}\\
       0, &  \text{else if $p=q$}\\
        \text{undefined}, 
       &  \text{else.}
     \end{cases}
\end{align}
% \ra{it is not a problem since  if $L_D(h) = 0$ then $L_S(h) = 0$ a.s.} \todo{fix the proofs for the kl-gap to separately handle the case $L_D(h) = 0$ for which we have gap=0 w.p.~1 }

For any $h \in \calH$, define the relative entropy of the empirical risk with respect to the expected one as 
$\DeltaKL_{S}(h) \defeq \kld{\Lhat_{S}(h)}{L_{D}(h)}$. 
Similarly, and overloading notations, we define 
for any distribution $Q \in \calM(\calH)$,
$
    \DeltaKL_{S}(Q) \defeq \kld{\Lhat_{S}(Q)}{L_{D}(Q)}.
$

% \begin{prop}\label{prop:Maurer}
% (Theorem 5 in \cite{Maurer_04}) For any probability distribution
% $P$ over $\calH$ and for any $\delta \in (0,1)$, w.p.\
% at least $1-\delta$ over $S \sim \calD^{m}$ and all
% distributions $Q$, we have
% \begin{align}
% \DeltaKL_{S}(Q) & \le\frac{\KL QP+ \ln(\nf{2\sqrt{m}}{\delta})}{m}. \label{eq:Maurer_bound}
% \end{align}
% \end{prop}

 By replacing the KL-based change-of-measure inequality step in the proof of 
 \cite[Thm. 5]{Maurer_04} with the IPM property (Def. \ref{def:IPM}) we get a similar bound for IPM measures (see a detailed proof in Appendix \ref{sect:proof_thm:PACBayesBound_IPM_Seeger}).
 
\begin{prop}[Template for Seeger Type IPM PAC-Bayes Bound] \label{prop:Seeger_IPM_PB}
Assume $f_{S}(h)\defeq m \cdot \DeltaKL_{S}(h)  \in \calF_S$. Then for any prior $P \in \calM(\calH)$ and $\delta\in (0,1)$, with probability
at least $1-\delta$ over the samples $S \sim \calD^m$, for all $Q \in \calM(\calH)$, we have
\begin{align}
\DeltaKL_{S}(Q) & \leq \frac{\gamma_{\calF_S}(Q, P) + \ln(\nf{{2\sqrt{m}}}{\delta})}{m}.
\end{align}
 By applying the Refined Pinsker's relaxation (\cite{McAllester_03}, Eq.~6) we can immediately derive the following, looser, but easier to interpret bound
\textbf{\begin{align}
    \Delta_{S}(Q) & \le \sqrt{2\Lhat_{S}(Q) 
    \frac{\gamma_{\calF_S}(Q, P) + \ln(\nf{{2\sqrt{m}}}{\delta})}{m}}
    + 2 \frac{\gamma_{\calF_S}(Q, P) + \ln(\nf{{2\sqrt{m}}}{\delta})}{m}.
\end{align}}
\end{prop}

% \begin{prop}[IPM Seeger's Type PAC-Bayes Template Bound] \label{prop:Seeger_IPM_PB_gap}
% Assume $f_{S}(h)\defeq m \cdot \kld{L_S(h)}{L_D(h)} \in \calF_S$. We have that with probability at least
% $1-\delta$ over the samples $S \sim \calD^m$, the following inequality holds for all $Q \in \calM(\calH)$
% \begin{align}
% \Delta_{S}(Q) & \le \sqrt{2\Lhat_{S}(Q) 
% \frac{\gamma_{\calF_S}(Q, P) + \ln(\nf{{2\sqrt{m}}}{\delta})}{m}}
% + 2 \frac{\gamma_{\calF_S}(Q, P) + \ln(\nf{{2\sqrt{m}}}{\delta})}{m}.
% \end{align}
% \end{prop}
When $\Lhat_{S}(Q)$ is small (as is typical with  modern deep networks), the final term determines
the convergence rate. 
We defer the investigation of PB bounds derived from Prop.~\ref{prop:Seeger_IPM_PB} to Appendix \ref{sect:appendox_seeger}. In the following sections we focus on investigating the implication of the Template IPM PB Bound of Prop.~\ref{prop:PACBayesBound_IPM}.

%%%%%%%%%%%%%%%%%%%%%%%%%%%%%%%%%%%%%%%%%%%%%%%%%%%%%%%%%

%%%%%%%%%%%%%%%%%%%%%%%%%%%%%%%%%%%%%%%%%%%%%%%%%%%%%%%%
\section{Total-Variation PAC-Bayes Bounds}
%%%%%%%%%%%%%%%%%%%%%%%%%%%%%%%%%%%%%%%%%%%%%%%%%%%%%%%%
In this section, we investigate a PB bound with the total-variation (TV) distance, $\DTV(Q,P) \defeq \sup_{A \in  \Sigma_{\calH}} \abs*{P(A) - Q(A)}$, where $\Sigma_{\calH}$ is the standard Borel sigma-algebra associated with $\calH$.
The TV distance can be described as an IPM with the family of functions 
\begin{equation} \label{eq:F_bound_innfty}
     \calFinfty_M \defeq \brc*{f:\calH \to [0,\infty), \norm{f}_\infty \leq M },
\end{equation}
for any $M \geq 0$.
To see this, note that
\begin{align} \label{eq:TVeqGamma}
        \gamma_{\calFinfty_M}(Q, P) = \sup_{f \in \calFinfty_M} \abs*{\int_\calH f \dd P - \int_\calH f \dd Q}
        \textrel{(i)}{=}  M \cdot \sup_{A \in  \Sigma_{\calH}} \abs*{P(A) - Q(A)}
        =  M \DTV(Q,P),
\end{align}
    where equality (\textit{i)} holds since in the supremum it suffices to take the class of indicator functions $\brc*{M \cdot \mathbb{1}_A(h), A \in  \Sigma_{\calH}}$, since the functions in $\calFinfty_M$ are bounded in $[0, M]$. 

\begin{thm}[Template for Total-Variation PB Bound]
\label{thm:TVPB}
Assume that there exists some uniform convergence bound $\uc\br{m, \delta'}$ (Definition \ref{def:UC}), 
then, for any prior $P \in \calM(\calH)$ and
$\delta \in \left(0,1\right)$, with probability of
at least $1-\delta$ over samples $S \sim \calD^m$, for all $Q \in \calM(\calH)$, we have
\begin{align}\label{eq:TVBound}
\Delta_S(Q) \le\sqrt{ \uc^2\br{m, \nf{\delta}{2}}  \DTV(Q,P) + \frac{ \ln(\nf{2 m}{\delta})}{2(m-1)}}.
\end{align}
\end{thm}
The proof (Appendix \ref{sect:proof_thm_TVPB}) follows directly from the general IPM PB bound (Prop. \ref{prop:PACBayesBound_IPM}) and the uniform convergence assumption, using a union bound argument. This bound can be seen as a template to be used to derive explicit PB bounds, by plugging in existing UC bounds. Note that while we require the existence of UC bound, the resulting bound is nonuniform (since it depends on the data-dependent posterior).

Compared to the original UC bound, $\uc\br{m,\delta}$, the bound in \eqref{eq:TVBound} is roughly multiplied by a factor of $\sqrt{\DTV(Q,P)} \in [0,1]$, ensuring tighter guarantees, especially if the posterior is close to the prior.

For example, consider a binary classification case, with the zero-one loss function and $\VC$ class ${\cal H}$. The well-known UC theorem states that the generalization gap converges uniformly at a rate $O \br*{\sqrt{\nf{\VC}{m}}}$.
\begin{prop}   
% [\citet{lafferty2008concentration}, Corollary 7.78]
[VC Bound, \citet{boucheron2005theory}] \label{prop:vc}
There exists some universal constant
\footnote{{\placeholdera}The bound of Cor.~\ref{prop:vc} originates from \citet{talagrand1994sharper}.
As far as we know, there is no explicit value of the universal constant in the literature.
Obtaining the constant involves careful computations of covering numbers and using the chaining method (e.g., based on Thm.~1.16 and 1.17 in \citet{lugosi2002pattern}).
Since our focus was not on the numerical evaluation of the bounds, we did not include this in our work. We note that there are other VC-type bounds with explicit constants, but with an extra $\log(m)$ factor (e.g., \citet{vapnik1999nature}, Sect 3.4). }
$c > 0 $ s.t.\ for any $\delta \in (0,1)$ we have
 \begin{equation} 
    \Prob \brc*{\Delta_S(h) \leq c \sqrt{\frac{\VC +  \ln(\nf{1}{\delta})}{m} }, \forall h \in \calH} \geq 1 - \delta .
 \end{equation}
\end{prop}
% \RM{Since our results include $\log m$ terms we may as well use the standard VC bounds that include $\log m$}

Using Thm.\ \ref{thm:TVPB}, we derive the following algorithm and data-dependent bound.
\begin{cor} [Total-Variation PB Bound for VC Classes]
\label{cor:TV_VC}
Consider a binary classification problem, with the zero-one loss, and hypothesis class $\calH$, with finite VC dimension, $\VC$.
There exists some universal constant $c > 0 $ s.t.\ for any prior $P \in \calM(\calH)$ and $\delta \in \left(0,1\right)$, with probability of
at least $1-\delta$ over samples $S \sim \calD^m$, for all $Q \in \calM(\calH)$, we have
\begin{align}
\Delta_S(Q) 
\le \sqrt{c \frac{\VC +
\ln(\nf{1}{\delta})}{m} \DTV(Q,P) + \frac{ \ln(\nf{ m}{\delta})}{2(m-1)}}. 
\end{align}
\end{cor}
Compared to the UC bound of Prop.~\ref{prop:vc}, Cor.~\ref{cor:TV_VC} multiplies the dominant term of the bound by a nonuniform (data and algorithm-dependent) factor of $\sqrt{\DTV(Q,P)}$, which is guaranteed to tighten the bound. 

Note that the total-variation based bound of \citet{aminian2022tighter} and \citet{rodriguez2021tighter} assume Lipschitz loss function, while our TV bound allows non continuous loss functions such as the zero-one loss.
The TV based bound of \citet{wang2019information} (Thm.~1) is not directly comparable, since the empirical risk term is multiplied by a factor that goes to infinity for TV distance that goes to 1.

%%%%%%%%%%%%%%%%%%%%%%%%%%%%%%%%%%%%
\section{Wasserstein PAC-Bayes Bounds} \label{sec:Wasserstein-PAC-Bayes-Bound}
%%%%%%%%%%%%%%%%%%%%%%%%%%%%%%%%%%%%
%%%%%%%%%%%%%%%%%%%%%%%%%%%%%%%%%%%%%%%%%%%%%%%%%%%%%%%%
\subsection{Template for Wasserstein-PB Bound}\label{sec:TemplateWasser}
%%%%%%%%%%%%%%%%%%%%%%%%%%%%%%%%%%%%%%%%%%%%%%%%%%%%%%%%
In this section, we provide a PAC-Bayes generalization bound with the Wasserstein metric between posterior and prior and a certain smoothness parameter of the generalization gap function. We explore learning settings for which $\calH$ can be paired with a distance metric $\rho : \calH \times \calH \to \R_{\geq 0}$ s.t.\ $\left(\calH,\rho\right)$ is a Polish metric space (complete, separable metric space).\footnote{ See \cite{Villani06}
Ch. 1, for a discussion of this assumption.}
Given the distance metric $\rho$, we can define the  Wasserstein distance between any two probability measures on $\calH$.
\begin{defn}[Wasserstein Distance]
For any two probability measures $P,Q$
on $\calH$ with finite first moment, the $1^{\mbox{st}}$ order
Wasserstein distance is 
\begin{align}
W_1(Q,P) & \defeq \inf_{\gamma\in\Gamma(Q,P)}\int_{\calH\times\calH}\rho(h,h') \dd \gamma(h,h'),
\end{align}
where $\Gamma(Q,P)$ denotes the set of all couplings of
$Q$ and $P$, that is, the set of all joint measure on $\calH\times\calH$
whose marginals are $Q$ and $P$.
\end{defn}
% In this paper, we will only use $p=1$, for which the following duality holds.

The following proposition gives a dual representation for the first-order Wasserstein distance. 
\begin{prop}[Kantorovich-Rubinstein Duality \citep{Villani06}] \label{prop: KR_duality}
For any $0 \leq K $, and any two probability measures $P, Q \in \calM(\calH)$,
\begin{align} \label{eq:K_KR_duality}
K \cdot W_1(Q,P) = \sup_{f \in \calFLip_K} \abs*{ \Expct hP\brs*{f(h)} - \Expct hQ\brs*{f(h)}} ,
\end{align}
where $\calFLip_K$ is the set of $K$-Lipschitz functions w.r.t.\ $\rho(h,h')$, i.e. functions that satisfy
\begin{equation}
    \sup_{(h,h') \in \calH^2} \frac{\abs*{f(h) - f(h')}}{\rho(h,h')} \leq K .
\end{equation}
\end{prop}

We can write the Kantorovich-Rubinstein duality \eqref{eq:K_KR_duality} using IPM formulation (Def. \ref{def:IPM}),
\begin{align} \label{eq:Duality_Short}
K \cdot W_1(Q,P) = \gamma_{\calFLip_K} (P,Q).
\end{align}

Using this duality we will prove the following bound.
\begin{thm}[Template for Wasserstein-PB Bound]
\label{thm:WPB}
Assume that for any $\delta' \in (0,1]$, w.p.\ at least $1-\delta'$ over the sampling $S \sim \calD^m$, the squared generalization gap function, $\Delta_{S}^{2}(\cdot)$, is $K$-Lipschitz w.r.t.\ the metric $\rho$ with some $K = K\br{m, \delta'}$.
Then, for any prior $P \in \calM(\calH)$ and $\delta \in \left(0,1\right)$, with probability of
at least $1-\delta$ over samples $S \sim \calD^m$, for all $Q \in \calM(\calH)$, we have
\begin{align}
\Delta_S(Q) \le \sqrt{K(m, \nf{\delta}{2}) W_1(Q,P) + \frac{ \ln(\nf{2 m}{\delta})}{2(m-1)}}.
\end{align}
\end{thm}
The proof (Appendix \ref{sect:proof_thm:WPB}) follows directly from Proposition \ref{prop:PACBayesBound_IPM} and the assumption, via the union bound.
% The Lipschitz constant of $\Delta_S^2(\cdot)$, $K_S$, is a random variable where the randomness stems from the sampling $S \sim \calD^m$. 
Theorem \ref{thm:WPB} can be seen as a template to be used for deriving Wasserstein-PB bounds in various learning settings where the $\Delta_S^2(\cdot)$ is $K$-Lipschitz with high probability (over samples), where the rate of $K = K(m, \nf{\delta}{2})$ should be $O(\nf{1}{m})$ to ensure a factor $O(\nf{1}{\sqrt{m}})$ multiplying the divergence between posterior and prior, as in the KL-PB bound.

Such a result can be challenging to prove since it requires uniform convergence of the slope between any two hypotheses in $\calH$. 
Next, we will show specific learning settings where this property holds and the resulting generalization bounds.

We note that recent work \citet{neu2022generalization} also establishes a Wasserstein-based information-theoretic generalization bound. This work assumed infinitely smooth loss functions and established bounds on the expected generalization gap, rather than high-probability bounds. In fact, \citet{neu2022generalization} noted that obtaining such bounds is an open problem. Observe, though, that our bound depends on the Lipschitz constant $K = K(m,\delta)$ which needs to assessed; see sections \ref{sec:Wasser_finite} and \ref{sec:Wasser_regress} for specific examples. The general problem remains open.

At a more pragmatic level, we note that learning algorithms derived from minimizing Wasserstein based PB bounds have an added benefit of more stable optimization compared to KL based approaches, due to lower gradient variance, as noted in \cite{WGAN}, and this distance measure can be approximated efficiently from finite samples \citep{Weed_Bach_17,Cuturi_Sinkhorn}
%Here we can mention meta-learning and the benefit of Wasserstein: " the Wasserstein distance is also known to be more convenient from an optimization point-of-view (\cite{Cuturi_Sinkhorn, Solomon_15, WGAN, Weed_Bach_17}")
%%%%%%%%%%%%%%%%%%%%%%%%%%%%%%%%%%%%%%%%%%%%%%%%%%%%%%%%
\subsection{Wasserstein-PB Bound for Finite Classes}\label{sec:Wasser_finite}
%%%%%%%%%%%%%%%%%%%%%%%%%%%%%%%%%%%%%%%%%%%%%%%%%%%%%%%%
We first investigate the simple case of a finite hypothesis class with a loss function $\ell$ which is $G$-Lipschitz w.r.t.\ the metric $\rho$.
Note that for finite classes, UC always holds.
We derive the following bound from Thm.~\ref{thm:WPB}.

\begin{thm}[Wasserstein-PB Bound for Finite Classes] \label{thm:WPB_finite}
 Let $\calH$ be a finite hypothesis class.
 Assume that for any fixed $z \in \calZ$, $\ell(h,z)$ is a $\LossLip$-Lipschitz function in $h$ w.r.t~the metric $\rho$.
 Then for any prior $P \in \calM(\calH)$ and $\delta \in \left(0,1\right)$, with probability of
at least $1-\delta$ over samples $S \sim \calD^m$, for all $Q \in \calM(\calH)$, we have
\begin{align}
\Delta_S(Q) \le \sqrt{\frac{8 \LossLip\log(\nf{4 \abs{\calH}}{\delta}) }{m} W_1(Q,P) + \frac{ \ln(\nf{2 m}{\delta})}{2(m-1)}}. 
\end{align}
\end{thm}
The proof (Appendix \ref{sect:proof_KS_finite_H}) makes use of standard union bound arguments, Hoeffding's concentration inequality and the template Wasserstein-PB bound (Thm.~\ref{thm:WPB}).
Notice that compared to the standard UC bound for finite classes, Thm.~\ref{thm:WPB_finite} multiplies the bound by a nonuniform factor of $ \sqrt{G W_1(Q,P)}$.

 %%%%%%%%%%%%%%%%%%%%%%%%%%%%%%%%%%%%%%%%%%%%%%%%%%%%%%%%
\subsection{Wasserstein-PB Bound for Loss-Gradient UC Classes}\label{sec:Wasser_regress}
%%%%%%%%%%%%%%%%%%%%%%%%%%%%%%%%%%%%%%%%%%%%%%%%%%%%%%%%
In this section we show a Wasserstein-PB bound for learning problems $(\calD, \calH, \ell)$, with $\calH \subset \R^d$, for some dimension $d \in \N^+$, that satisfy the standard UC property, and, additionally, satisfy UC property for the loss gradient, as defined below.

 \begin{defn}[Loss-Gradient UC Property] \label{def:uc_grad}
A learning problem $(\calD, \calH, \ell)$, with $\calH \subset \R^d$,
is said to satisfy the \textbf{loss-gradient UC property}, if: 
{{\placeholdera}
\textit{(i)} the loss function $\ell(h,z)$ is differentiable w.r.t.~h on $\interior(\calH) \times \calZ$ and continuous w.r.t.~$h$ on $\calH \times \calZ$.
}
\textit{(ii)} The problem satisfies the uniform convergence property (Def.~\ref{def:UC}). \textit{(iii)} The empirical average of the loss gradient converges uniformly in $L_2$ norm sense to its mean. I.e., there exists a bound function $\ucg\br{m,\delta} > 0$, s.t.\ for any $\delta \in (0,1)$ we have $\ucg\br{m,\delta} \underset{m \to \infty}{\to} 0$ and for any $m \in \N^+$,
\begin{align}
   \Prob \br*{\norm{ \E_{z\sim\calD}  \nabla_h\ell(h,z) - \frac{1}{m}\sum_{i=1}^{m}
      \nabla_h\ell(h,z_i) }_2 \leq \ucg\br{m,\delta}, \forall h \in \interior(\calH)} \geq 1 - \delta, 
\end{align}
where $\nabla_h\ell(h,z)$ denotes the gradient of $\nabla_h\ell(h,z)$ w.r.t.\ $h$, for a fixed $z \in \calZ$, and $\interior(\calH) $ is the interior of $\calH$.
We call $\ucg$, the UC bound of the loss gradient.
\end{defn}

\begin{thm}[Wasserstein-PB Bound for Loss-Gradient UC Classes] \label{thm:WPB_uc_grad}
Let $(\calH, \rho)$ be a metric space such that $\calH \subset \R^d$ is a closed and convex set, and $\rho$ is the $L_2$ distance.
% Assume that for any $z \in \calZ$ and $h \in \interior(\calH)$, $\ell(h,z)$ is continuously differentiable w.r.t~$h$.
Assume the learning problem satisfies the loss-gradient UC property (Def. \ref{def:uc_grad}), with UC bound $\uc$, and a UC bound of the loss gradient, $\ucg$.
Then for any prior $P \in \calM(\calH)$ and $\delta \in \left(0,1\right)$, with probability of
at least $1-\delta$ over samples $S \sim \calD^m$, for all $Q \in \calM(\calH)$, we have
\begin{align}
\Delta_S(Q) \le \sqrt{ 2 \cdot \uc\br{m, \nf{\delta}{4}} \cdot  \ucg\br{m, \nf{\delta}{4}} \cdot W_1(Q,P) + \frac{ \ln(\nf{2 m}{\delta})}{2(m-1)}}. 
\end{align}
\end{thm}
The proof (Appendix \ref{sect:proof_thm:WPB_uc_grad}) is derived from the assumptions and the template Wasserstein-PB bound (Thm.~\ref{thm:WPB}).
In learning problems that satisfy the loss-gradient UC property, we often have $\uc\br{m, \nf{\delta}{4}}, \ucg\br{m, \nf{\delta}{4}} \in O(\nicefrac{1}{\sqrt{m}})$, and then the resulting bound is $O(\nicefrac{1}{\sqrt{m}})$.
We provide full analysis that shows such a rate for the following linear regression example. 

% \todo{
% \textit{It would be good to include an explicit comparison of the new bounds to those of, e.g., Neu and Lugosi (2022). (If space is an issue, at least include the bound of the latter in the appendix).}
% We will expand the discussion related to Neu and Lugosi (2022) [ Cor. 8 -$R \beta  d /n$, $\beta$ assumption] in a future revision.
% We emphasize, though, that their result holds in expectation, as opposed to our high
% probability results.}

 %%%%%%%%%%%%%%%%%%%%%%%%%%%%%%%%%%%%%%%%%%%%%%%%%%%%%%%%

\subsection{Linear Regression Example}
%%%%%%%%%%%%%%%%%%%%%%%%%%%%%%%%%%%%%%%%%%%%%%%%%%%%%%%%

Based on the Wasserstein-PB Bound for Loss-Gradient UC Classes (Thm.~\ref{thm:WPB_uc_grad}), we derive the following corollary. 
\begin{cor}[Wasserstein-PB Bound for Linear Regression]\label{cor:wpb_regression}
Consider a data distribution of pairs $z =(x,y)$, where $x$ is sampled from an unknown distribution supported on a $d$-dimensional ball of radius $r >0$, $\Br \defeq \brc{x \in \R^d : \norm{x}_2 \leq r}$, and the target, $y=f(x)$, is set by an unknown, possibly random, target function $f : \Br \to [-1,1]$.
% {{\placeholdera}The hypothesis space
% $\calH$ is $\R^d$
% , and the loss function is $\ell(x, y, h) = \min \brc*{\frac{1}{4} (h^\top x - y)^2, 1}$}.
The hypothesis space
$\calH$ is $\Bir$
, and the loss function is $\ell(x, y, h) = \frac{1}{4} (h^\top x - y)^2$.
% {\color{red}The hypothesis space
% $\calH$ is $\R^d$
% , and the loss function is $\ell(x, y, h) = \min \brc*{\frac{1}{4} (h^\top x - y)^2, 1}$}.
%  , and the loss function is $\ell(x, y, h) = \min \brc*{\frac{1}{4} (\hhat^\top x - y)^2, 1}$, where $\hhat \defeq \frac{1}{r \norm{h}_2 } h $ (the projection of $h$ to a sphere of radius $\nf{1}{r}$).
% the $\nf{1}{r}$-sphere $\Sdir$, and the loss function is $\ell(x, y, h) = \frac{1}{4} (h^\top x - y)^2$.
Then, for any prior $P \in \calM(\calH)$ and $\delta \in \left(0,1\right)$, with probability 
at least $1-\delta$ over samples $S \sim \calD^m$, for all $Q \in \calM(\calH)$,
\begin{align}
\Delta_S(Q) \le \sqrt{ 2   \uc\br{m, \nf{\delta}{4}} \cdot  \ucg\br{m, \nf{\delta}{4}} \cdot W_1(Q,P) + \frac{ \ln(\nf{2 m}{\delta})}{2(m-1)}},
\end{align}
where $W_1(Q,P)$ denotes the $1^{\mbox{st}}$ order
Wasserstein distance with the $L_2$ metric,
\begin{equation}
    \uc(m, \delta) \in O\br*{\sqrt{\frac{
    d \br*{1 + \ln(\nf{1}{\delta}})}{m}}} \text{, and }\ucg(m, \delta) \in O\br*{r \sqrt{\frac{
    d \br*{1 + \ln(\nf{1}{\delta}})}{m}}} 
 . 
\end{equation}
\end{cor}
The full expression of the bound appears in the theorem's proof (Appendix \ref{sect:proof_cor_wpb_regression}).
Ignoring logarithmic factors we obtained a UC bound of $\tildO \br*{\sqrt{\frac{d}{m}}}$, and a Wasserstein-PB bound of $\tildO \br*{\sqrt{r W_1(Q,P) \frac{d}{m}}}$. Note that from Thm. \ref{thm:TVPB} we can also deduce a TV-PB bound of order $\tildO  \br*{\sqrt{\DTV(Q,P) \frac{d}{m} }}$.
In comparison, the standard KL-PB bound is of order $\tildO \br*{\sqrt{\frac{\KL{Q}{P}}{m}}}$. 
{{\placeholdera}
The TV-PB bound is, at worst, roughly the same as the UC bound, since $\DTV(Q,P) \leq 1$. 
Note that since $Q$ and $P$ are distributions over a sphere of radius $\nf{1}{r}$, then $r W_1(Q,P) \leq 2$. Hence, the Wasserstein-PB is also, at worst, roughly the same as the UC bound. However, the KL-PB bound can be either tighter or looser, depending on $P$ and $Q$. In cases where the mass of the posterior $Q$ is concentrated in a region of the hypothesis set where the prior $P$ is arbitrarily small, then the KL divergence can be arbitrarily large, making the KL-PB bound extremely loose compared to the UC, TV-PB, and Wasserstein-PB bounds.
The numerical experiment described in Appendix \ref{sect:expriment} demonstrates this by investigating different prior distributions with different widths. In particular, for posteriors and priors that are Dirac delta distributions (i.e., deterministic predictors), we show that the Wasserstein-PB considerably improves over the UC bound, while the KL-PB bound is undefined. 
We therefore demonstrated 
non-vacuous guarantees for deterministic models within the PB framework without requiring additional derandomization steps. 
 }% end color

{{\placeholdera}We can compare the bound of Thm.~\ref{thm:WPB_uc_grad} to the bound of Corollary 8, in \citet{neu2022generalization}, which is also dependent on the Wasserstein distance between a data-dependent output (posterior) and a base measure (prior). The bound of Thm.~\ref{thm:WPB_uc_grad} is different in the sense that 
\textit{(i)} It holds with high probability instead of in expectation. \textit{(ii)} Instead of assuming infinitely-smooth loss function with $\beta$-bounded directional derivatives, Thm.~\ref{thm:WPB_uc_grad} assumes a loss-gradient UC class. \textit{(iii)} The bound scales as $\tildO \br*{\sqrt{W_1 \cdot \uc \cdot \ucg }}$ instead of $\tildO \br*{\sqrt{W_2 \cdot \frac{d \beta}{m}}}$. 

In our linear regression example, Cor.~\ref{cor:wpb_regression} scales as $\tildO  \br*{\sqrt{W_1 \cdot \uc \cdot \ucg}} = \tildO  \br*{\sqrt{r  \sqrt{\frac{d}{m}} r \sqrt{\frac{ d}{m}}}} = \tildO \br*{r \sqrt{\frac{d}{m}}}$. The loss is infinitely-smooth with $\beta \in O(1 + r + r^2)$, and therefore Cor.~8 of \citet{neu2022generalization} scales as $\tildO  \br*{\sqrt{W_2 \cdot \frac{d \beta}{m}}} = \tildO  \br*{ r \sqrt{\frac{d}{m} (1 + r +  r^2)}}$, i.e.,~looser by a factor of $\tildO  \br*{ \sqrt{1 + r +  r^2}}$ compared to Cor.~\ref{cor:wpb_regression}.
}

%%%%%%%%%%%%%%%%%%%%%%%%%%%%%%%%%%%%%%%%%%%%%
%%%%%%%%%%%%%%%%%%%%%%%%%%%%%%%%%%%%%%%%%%%%%%%%%%%%%%%%
\section{Discussion} \label{sec:Discussion}
%%%%%%%%%%%%%%%%%%%%%%%%%%%%%%%%%%%%%%%%%%%%%%%%%%%%%%%%
% \todo{in terms of weaknesses, the fact that uniform convergence is required in most bounds is not ideal. Especially, since the new bounds depend on the uniform convergence error which can be large for Neural Networks, for example. After all, having non-vacuous bounds for large classes was one of the main motivations behind the paper.}
% \todo{discuss - bound loosens (VC - loose, but localized UC) , compare to Neu,  future work}

We have presented high-probability PB bounds based on integral probability metrics, that extend standard PB bounds based on KL divergence and more recent $f$-divergence and $\alpha$-divergence based bounds, to a new class of distances. Our bounds interpolate between classic UC bounds and PB bounds, by allowing data- and algorithm-dependent complexity terms. As in all PB results, our bounds suggest improved rates when the PB posterior is close to the prior. While we have extended high-probability PB bounds for IPMs to novel distance measures, it is still an open question to do so without the UC assumption. 

Possible directions for future research include: \emph{(i)} Deriving high probability IPM-PB bounds (e.g.\ Wasserstein or TV based), without global UC assumptions (possibly using localization based approaches, {{\placeholdera}e.g.~, Local Rademacher complexities \citep{koltchinskii2004rademacher, bartlett2005local} are computed only on a subset of hypotheses with small empirical risk}). 
{{\placeholdera}This may allow non-vacuous bounds for large-scale models where global UC-based bounds are extremely vacuous.}
\emph{(ii)} Derivation of algorithms that utilize the bounds as minimization objectives. Based on the optimization advantages of Wasserstein based costs mentioned in Sec.\ \ref{sec:TemplateWasser}, these could lead to enhanced practical utility. Such an advantage could play an important role in meta-learning schemes where PB methods have been widely used in recent years \citep{Amit_Meir_18}.

%%%%%%%%%%%%%%%%%%%%%%%%%%%%%%%%%%%%%%%%%%%%%%%%%%%%%%%%%%%%

\subsubsection*{Acknowledgments}

We thank Nadav Merlis and Daniel Soudry for helpful discussions of this work, and the anonymous reviewers for their helpful comments. 
Shay Moran is a Robert J.\ Shillman Fellow; he acknowledges support by ISF grant 1225/20, by BSF grant 2018385, by an Azrieli Faculty Fellowship, by Israel PBC-VATAT, by the Technion Center for Machine Learning and Intelligent Systems (MLIS), and by the the European Union (ERC, GENERALIZATION, 101039692). Views and opinions expressed are however those of the author(s) only and do not necessarily reflect those of the European Union or the European Research Council Executive Agency. Neither the European Union nor the granting authority can be held responsible for them.
% SM is a Robert J.\ Shillman Fellow; he acknowledges support by ERC grant 802599, by ISF grant 1225/20, by BSF grant 2018385, by an Azrieli Faculty Fellowship, by Israel PBC-VATAT, and by the Technion Center for Machine Learning and Intelligent Systems (MLIS).
The work of Ron Meir is partially supported by ISF grant  1693/22, by the Ollendorff Center of the Viterbi ECE Faculty
at the Technion, and by the Skillman chair in biomedical sciences.

%%%%%%%%%%%%%%%%%%%%%%%%%%%%%%%%%%%%%%%%%%%%%%%%%%%%%%%%%%%%
% \newpage
% \bibliographystyle{apalike}
\bibliography{bibfile.bib}

\begin{thebibliography}{}

\bibitem [\protect \citeauthoryear {%
Alquier%
}{%
Alquier%
}{%
{\protect \APACyear {2021}}%
}]{%
alquier2021user}
\APACinsertmetastar {%
alquier2021user}%
\begin{APACrefauthors}%
Alquier, P.%
\end{APACrefauthors}%
\unskip\
\newblock
\APACrefYearMonthDay{2021}{}{}.
\newblock
{\BBOQ}\APACrefatitle {User-friendly introduction to {PAC-Bayes} bounds}
  {User-friendly introduction to {PAC-Bayes} bounds}.{\BBCQ}
\newblock
\APACjournalVolNumPages{arXiv preprint arXiv:2110.11216}{}{}{}.
\PrintBackRefs{\CurrentBib}

\bibitem [\protect \citeauthoryear {%
Alquier%
\ \BBA {} Guedj%
}{%
Alquier%
\ \BBA {} Guedj%
}{%
{\protect \APACyear {2018}}%
}]{%
alquier2018simpler}
\APACinsertmetastar {%
alquier2018simpler}%
\begin{APACrefauthors}%
Alquier, P.%
\BCBT {}\ \BBA {} Guedj, B.%
\end{APACrefauthors}%
\unskip\
\newblock
\APACrefYearMonthDay{2018}{}{}.
\newblock
{\BBOQ}\APACrefatitle {Simpler {PAC-Bayesian} bounds for hostile data} {Simpler
  {PAC-Bayesian} bounds for hostile data}.{\BBCQ}
\newblock
\APACjournalVolNumPages{Machine Learning}{107}{5}{887--902}.
\PrintBackRefs{\CurrentBib}

\bibitem [\protect \citeauthoryear {%
Aminian%
, Bu%
, Wornell%
\BCBL {}\ \BBA {} Rodrigues%
}{%
Aminian%
\ \protect \BOthers {.}}{%
{\protect \APACyear {2022}}%
}]{%
aminian2022tighter}
\APACinsertmetastar {%
aminian2022tighter}%
\begin{APACrefauthors}%
Aminian, G.%
, Bu, Y.%
, Wornell, G.%
\BCBL {}\ \BBA {} Rodrigues, M.%
\end{APACrefauthors}%
\unskip\
\newblock
\APACrefYearMonthDay{2022}{}{}.
\newblock
{\BBOQ}\APACrefatitle {Tighter Expected Generalization Error Bounds via
  Convexity of Information Measures} {Tighter expected generalization error
  bounds via convexity of information measures}.{\BBCQ}
\newblock
\APACjournalVolNumPages{arXiv preprint arXiv:2202.12150}{}{}{}.
\PrintBackRefs{\CurrentBib}

\bibitem [\protect \citeauthoryear {%
Amit%
\ \BBA {} Meir%
}{%
Amit%
\ \BBA {} Meir%
}{%
{\protect \APACyear {2018}}%
}]{%
Amit_Meir_18}
\APACinsertmetastar {%
Amit_Meir_18}%
\begin{APACrefauthors}%
Amit, R.%
\BCBT {}\ \BBA {} Meir, R.%
\end{APACrefauthors}%
\unskip\
\newblock
\APACrefYearMonthDay{2018}{}{}.
\newblock
{\BBOQ}\APACrefatitle {Meta-learning by adjusting priors based on extended
  {PAC-Bayes} theory} {Meta-learning by adjusting priors based on extended
  {PAC-Bayes} theory}.{\BBCQ}
\newblock
\BIn{} \APACrefbtitle {International Conference on Machine Learning ({{ICML}})}
  {International conference on machine learning ({{ICML}})}\ (\BPGS\ 205--214).
\PrintBackRefs{\CurrentBib}

\bibitem [\protect \citeauthoryear {%
Anthony%
\ \BBA {} Bartlett%
}{%
Anthony%
\ \BBA {} Bartlett%
}{%
{\protect \APACyear {1999}}%
}]{%
anthony1999neural}
\APACinsertmetastar {%
anthony1999neural}%
\begin{APACrefauthors}%
Anthony, M.%
\BCBT {}\ \BBA {} Bartlett, P\BPBI L.%
\end{APACrefauthors}%
\unskip\
\newblock
\APACrefYear{1999}.
\newblock
\APACrefbtitle {Neural network learning: Theoretical foundations} {Neural
  network learning: Theoretical foundations}\ (\BVOL~9).
\newblock
\APACaddressPublisher{}{cambridge university press Cambridge}.
\PrintBackRefs{\CurrentBib}

\bibitem [\protect \citeauthoryear {%
Arjovsky%
, Chintala%
\BCBL {}\ \BBA {} Bottou%
}{%
Arjovsky%
\ \protect \BOthers {.}}{%
{\protect \APACyear {2017}}%
}]{%
WGAN}
\APACinsertmetastar {%
WGAN}%
\begin{APACrefauthors}%
Arjovsky, M.%
, Chintala, S.%
\BCBL {}\ \BBA {} Bottou, L.%
\end{APACrefauthors}%
\unskip\
\newblock
\APACrefYearMonthDay{2017}{}{}.
\newblock
{\BBOQ}\APACrefatitle {Wasserstein generative adversarial networks}
  {Wasserstein generative adversarial networks}.{\BBCQ}
\newblock
\BIn{} \APACrefbtitle {International conference on machine learning ({ICML})}
  {International conference on machine learning ({ICML})}\ (\BPGS\ 214--223).
\PrintBackRefs{\CurrentBib}

\bibitem [\protect \citeauthoryear {%
Audibert%
\ \BBA {} Bousquet%
}{%
Audibert%
\ \BBA {} Bousquet%
}{%
{\protect \APACyear {2003}}%
}]{%
Audibert03PAC}
\APACinsertmetastar {%
Audibert03PAC}%
\begin{APACrefauthors}%
Audibert, J\BHBI Y.%
\BCBT {}\ \BBA {} Bousquet, O.%
\end{APACrefauthors}%
\unskip\
\newblock
\APACrefYearMonthDay{2003}{}{}.
\newblock
{\BBOQ}\APACrefatitle {{PAC-Bayesian} Generic Chaining} {{PAC-Bayesian} generic
  chaining}.{\BBCQ}
\newblock
\BIn{} S.~Thrun, L.~Saul\BCBL {}\ \BBA {} B.~Sch\"{o}lkopf\ (\BEDS),
  \APACrefbtitle {Advances in Neural Information Processing Systems
  ({NeurIPS})} {Advances in neural information processing systems ({NeurIPS})}\
  (\BVOL~16).
\newblock
\APACaddressPublisher{}{MIT Press}.
\PrintBackRefs{\CurrentBib}

\bibitem [\protect \citeauthoryear {%
Audibert%
\ \BBA {} Bousquet%
}{%
Audibert%
\ \BBA {} Bousquet%
}{%
{\protect \APACyear {2007}}%
}]{%
audibert2007combining}
\APACinsertmetastar {%
audibert2007combining}%
\begin{APACrefauthors}%
Audibert, J\BHBI Y.%
\BCBT {}\ \BBA {} Bousquet, O.%
\end{APACrefauthors}%
\unskip\
\newblock
\APACrefYearMonthDay{2007}{}{}.
\newblock
{\BBOQ}\APACrefatitle {Combining {PAC-Bayesian} and Generic Chaining Bounds.}
  {Combining {PAC-Bayesian} and generic chaining bounds.}{\BBCQ}
\newblock
\APACjournalVolNumPages{Journal of Machine Learning Research}{8}{4}{}.
\PrintBackRefs{\CurrentBib}

\bibitem [\protect \citeauthoryear {%
Bachmann%
, Moosavi-Dezfooli%
\BCBL {}\ \BBA {} Hofmann%
}{%
Bachmann%
\ \protect \BOthers {.}}{%
{\protect \APACyear {2021}}%
}]{%
bachmann2021uniform}
\APACinsertmetastar {%
bachmann2021uniform}%
\begin{APACrefauthors}%
Bachmann, G.%
, Moosavi-Dezfooli, S\BHBI M.%
\BCBL {}\ \BBA {} Hofmann, T.%
\end{APACrefauthors}%
\unskip\
\newblock
\APACrefYearMonthDay{2021}{}{}.
\newblock
{\BBOQ}\APACrefatitle {Uniform convergence, adversarial spheres and a simple
  remedy} {Uniform convergence, adversarial spheres and a simple
  remedy}.{\BBCQ}
\newblock
\BIn{} \APACrefbtitle {International Conference on Machine Learning ({ICML})}
  {International conference on machine learning ({ICML})}\ (\BPGS\ 490--499).
\PrintBackRefs{\CurrentBib}

\bibitem [\protect \citeauthoryear {%
Bartlett%
, Bousquet%
\BCBL {}\ \BBA {} Mendelson%
}{%
Bartlett%
\ \protect \BOthers {.}}{%
{\protect \APACyear {2005}}%
}]{%
bartlett2005local}
\APACinsertmetastar {%
bartlett2005local}%
\begin{APACrefauthors}%
Bartlett, P\BPBI L.%
, Bousquet, O.%
\BCBL {}\ \BBA {} Mendelson, S.%
\end{APACrefauthors}%
\unskip\
\newblock
\APACrefYearMonthDay{2005}{}{}.
\newblock
{\BBOQ}\APACrefatitle {Local {Rademacher} complexities} {Local {Rademacher}
  complexities}.{\BBCQ}
\newblock
\APACjournalVolNumPages{The Annals of Statistics}{33}{4}{1497--1537}.
\PrintBackRefs{\CurrentBib}

\bibitem [\protect \citeauthoryear {%
Bastani%
, Simchi-Levi%
\BCBL {}\ \BBA {} Zhu%
}{%
Bastani%
\ \protect \BOthers {.}}{%
{\protect \APACyear {2021}}%
}]{%
bastani2019meta}
\APACinsertmetastar {%
bastani2019meta}%
\begin{APACrefauthors}%
Bastani, H.%
, Simchi-Levi, D.%
\BCBL {}\ \BBA {} Zhu, R.%
\end{APACrefauthors}%
\unskip\
\newblock
\APACrefYearMonthDay{2021}{}{}.
\newblock
{\BBOQ}\APACrefatitle {Meta dynamic pricing: Transfer learning across
  experiments} {Meta dynamic pricing: Transfer learning across
  experiments}.{\BBCQ}
\newblock
\APACjournalVolNumPages{Management Science}{}{}{}.
\PrintBackRefs{\CurrentBib}

\bibitem [\protect \citeauthoryear {%
Begin%
, Germain%
, Laviolette%
\BCBL {}\ \BBA {} Roy%
}{%
Begin%
\ \protect \BOthers {.}}{%
{\protect \APACyear {2016}}%
}]{%
Begin_16}
\APACinsertmetastar {%
Begin_16}%
\begin{APACrefauthors}%
Begin, L.%
, Germain, P.%
, Laviolette, F.%
\BCBL {}\ \BBA {} Roy, J\BHBI F.%
\end{APACrefauthors}%
\unskip\
\newblock
\APACrefYearMonthDay{2016}{}{}.
\newblock
{\BBOQ}\APACrefatitle {{PAC-B}ayesian Bounds based on the Rényi Divergence}
  {{PAC-B}ayesian bounds based on the rényi divergence}.{\BBCQ}
\newblock
\APACjournalVolNumPages{Proceedings of the 19th International Conference on
  Artificial Intelligence and Statistics (AISTATS)}{}{}{}.
\PrintBackRefs{\CurrentBib}

\bibitem [\protect \citeauthoryear {%
Boucheron%
, Bousquet%
\BCBL {}\ \BBA {} Lugosi%
}{%
Boucheron%
\ \protect \BOthers {.}}{%
{\protect \APACyear {2005}}%
}]{%
boucheron2005theory}
\APACinsertmetastar {%
boucheron2005theory}%
\begin{APACrefauthors}%
Boucheron, S.%
, Bousquet, O.%
\BCBL {}\ \BBA {} Lugosi, G.%
\end{APACrefauthors}%
\unskip\
\newblock
\APACrefYearMonthDay{2005}{}{}.
\newblock
{\BBOQ}\APACrefatitle {Theory of classification: A survey of some recent
  advances} {Theory of classification: A survey of some recent
  advances}.{\BBCQ}
\newblock
\APACjournalVolNumPages{ESAIM: probability and statistics}{9}{}{323--375}.
\PrintBackRefs{\CurrentBib}

\bibitem [\protect \citeauthoryear {%
Bousquet%
\ \BBA {} Elisseeff%
}{%
Bousquet%
\ \BBA {} Elisseeff%
}{%
{\protect \APACyear {2002}}%
}]{%
bousquet2002stability}
\APACinsertmetastar {%
bousquet2002stability}%
\begin{APACrefauthors}%
Bousquet, O.%
\BCBT {}\ \BBA {} Elisseeff, A.%
\end{APACrefauthors}%
\unskip\
\newblock
\APACrefYearMonthDay{2002}{}{}.
\newblock
{\BBOQ}\APACrefatitle {Stability and generalization} {Stability and
  generalization}.{\BBCQ}
\newblock
\APACjournalVolNumPages{The Journal of Machine Learning
  Research}{2}{}{499--526}.
\PrintBackRefs{\CurrentBib}

\bibitem [\protect \citeauthoryear {%
Catoni%
}{%
Catoni%
}{%
{\protect \APACyear {2007}}%
}]{%
catoni2007pac}
\APACinsertmetastar {%
catoni2007pac}%
\begin{APACrefauthors}%
Catoni, O.%
\end{APACrefauthors}%
\unskip\
\newblock
\APACrefYearMonthDay{2007}{}{}.
\newblock
{\BBOQ}\APACrefatitle {{PAC-B}ayesian supervised classification: the
  thermodynamics of statistical learning} {{PAC-B}ayesian supervised
  classification: the thermodynamics of statistical learning}.{\BBCQ}
\newblock
\APACjournalVolNumPages{arXiv preprint arXiv:0712.0248}{}{}{}.
\PrintBackRefs{\CurrentBib}

\bibitem [\protect \citeauthoryear {%
Chee%
\ \BBA {} Loustau%
}{%
Chee%
\ \BBA {} Loustau%
}{%
{\protect \APACyear {2021}}%
}]{%
chee:hal-03262687}
\APACinsertmetastar {%
chee:hal-03262687}%
\begin{APACrefauthors}%
Chee, A.%
\BCBT {}\ \BBA {} Loustau, S.%
\end{APACrefauthors}%
\unskip\
\newblock
\APACrefYearMonthDay{2021}{}{}.
\newblock
{\BBOQ}\APACrefatitle {{Learning with BOT - Bregman and Optimal Transport
  divergences}} {{Learning with BOT - Bregman and Optimal Transport
  divergences}}.{\BBCQ}
\newblock
\APACjournalVolNumPages{hal-03262687v2}{}{}{}.
\PrintBackRefs{\CurrentBib}

\bibitem [\protect \citeauthoryear {%
Chuang%
, Mroueh%
, Greenewald%
, Torralba%
\BCBL {}\ \BBA {} Jegelka%
}{%
Chuang%
\ \protect \BOthers {.}}{%
{\protect \APACyear {2021}}%
}]{%
chuang2021measuring}
\APACinsertmetastar {%
chuang2021measuring}%
\begin{APACrefauthors}%
Chuang, C\BHBI Y.%
, Mroueh, Y.%
, Greenewald, K.%
, Torralba, A.%
\BCBL {}\ \BBA {} Jegelka, S.%
\end{APACrefauthors}%
\unskip\
\newblock
\APACrefYearMonthDay{2021}{}{}.
\newblock
{\BBOQ}\APACrefatitle {Measuring generalization with optimal transport}
  {Measuring generalization with optimal transport}.{\BBCQ}
\newblock
\APACjournalVolNumPages{Advances in Neural Information Processing
  Systems}{34}{}{8294--8306}.
\PrintBackRefs{\CurrentBib}

\bibitem [\protect \citeauthoryear {%
Clement%
\ \BBA {} Desch%
}{%
Clement%
\ \BBA {} Desch%
}{%
{\protect \APACyear {2008}}%
}]{%
clement2008elementary}
\APACinsertmetastar {%
clement2008elementary}%
\begin{APACrefauthors}%
Clement, P.%
\BCBT {}\ \BBA {} Desch, W.%
\end{APACrefauthors}%
\unskip\
\newblock
\APACrefYearMonthDay{2008}{}{}.
\newblock
{\BBOQ}\APACrefatitle {An elementary proof of the triangle inequality for the
  {W}asserstein metric} {An elementary proof of the triangle inequality for the
  {W}asserstein metric}.{\BBCQ}
\newblock
\APACjournalVolNumPages{Proceedings of the American Mathematical
  Society}{136}{1}{333--339}.
\PrintBackRefs{\CurrentBib}

\bibitem [\protect \citeauthoryear {%
Cuturi%
}{%
Cuturi%
}{%
{\protect \APACyear {2013}}%
}]{%
Cuturi_Sinkhorn}
\APACinsertmetastar {%
Cuturi_Sinkhorn}%
\begin{APACrefauthors}%
Cuturi, M.%
\end{APACrefauthors}%
\unskip\
\newblock
\APACrefYearMonthDay{2013}{}{}.
\newblock
{\BBOQ}\APACrefatitle {Sinkhorn distances: Lightspeed computation of optimal
  transport} {Sinkhorn distances: Lightspeed computation of optimal
  transport}.{\BBCQ}
\newblock
\APACjournalVolNumPages{Advances in neural information processing systems
  {(NeurIPS}}{26}{}{}.
\PrintBackRefs{\CurrentBib}

\bibitem [\protect \citeauthoryear {%
Dembo%
\ \BBA {} Zeitouni%
}{%
Dembo%
\ \BBA {} Zeitouni%
}{%
{\protect \APACyear {2009}}%
}]{%
dembo2009ldp}
\APACinsertmetastar {%
dembo2009ldp}%
\begin{APACrefauthors}%
Dembo, A.%
\BCBT {}\ \BBA {} Zeitouni, O.%
\end{APACrefauthors}%
\unskip\
\newblock
\APACrefYearMonthDay{2009}{}{}.
\newblock
{\BBOQ}\APACrefatitle {LDP for Finite Dimensional Spaces} {Ldp for finite
  dimensional spaces}.{\BBCQ}
\newblock
\BIn{} \APACrefbtitle {Large deviations techniques and applications} {Large
  deviations techniques and applications}\ (\BPGS\ 11--70).
\newblock
\APACaddressPublisher{}{Springer}.
\PrintBackRefs{\CurrentBib}

\bibitem [\protect \citeauthoryear {%
Donsker%
\ \BBA {} Varadhan%
}{%
Donsker%
\ \BBA {} Varadhan%
}{%
{\protect \APACyear {1975}}%
}]{%
Donsker_Varadhan_75}
\APACinsertmetastar {%
Donsker_Varadhan_75}%
\begin{APACrefauthors}%
Donsker, M\BPBI D.%
\BCBT {}\ \BBA {} Varadhan, S\BPBI R\BPBI S.%
\end{APACrefauthors}%
\unskip\
\newblock
\APACrefYearMonthDay{1975}{}{}.
\newblock
{\BBOQ}\APACrefatitle {Asymptotic evaluation of certain Markov process
  expectations for large time.} {Asymptotic evaluation of certain markov
  process expectations for large time.}{\BBCQ}
\newblock
\APACjournalVolNumPages{Comm. Pure Appl. Math.}{}{}{}.
\PrintBackRefs{\CurrentBib}

\bibitem [\protect \citeauthoryear {%
Dziugaite%
\ \BBA {} Roy%
}{%
Dziugaite%
\ \BBA {} Roy%
}{%
{\protect \APACyear {2017}}%
}]{%
dziugaite2017computing}
\APACinsertmetastar {%
dziugaite2017computing}%
\begin{APACrefauthors}%
Dziugaite, G\BPBI K.%
\BCBT {}\ \BBA {} Roy, D\BPBI M.%
\end{APACrefauthors}%
\unskip\
\newblock
\APACrefYearMonthDay{2017}{}{}.
\newblock
{\BBOQ}\APACrefatitle {Computing Nonvacuous Generalization Bounds for Deep
  (Stochastic) Neural Networks with Many More Parameters than Training Data}
  {Computing nonvacuous generalization bounds for deep (stochastic) neural
  networks with many more parameters than training data}.{\BBCQ}
\newblock
\BIn{} \APACrefbtitle {Conference on Uncertainty in Artificial Intelligence,
  ({UAI}).} {Conference on uncertainty in artificial intelligence, ({UAI}).}
\PrintBackRefs{\CurrentBib}

\bibitem [\protect \citeauthoryear {%
Givens%
\ \BBA {} Shortt%
}{%
Givens%
\ \BBA {} Shortt%
}{%
{\protect \APACyear {1984}}%
}]{%
givens1984class}
\APACinsertmetastar {%
givens1984class}%
\begin{APACrefauthors}%
Givens, C\BPBI R.%
\BCBT {}\ \BBA {} Shortt, R\BPBI M.%
\end{APACrefauthors}%
\unskip\
\newblock
\APACrefYearMonthDay{1984}{}{}.
\newblock
{\BBOQ}\APACrefatitle {A class of {W}asserstein metrics for probability
  distributions.} {A class of {W}asserstein metrics for probability
  distributions.}{\BBCQ}
\newblock
\APACjournalVolNumPages{Michigan Mathematical Journal}{31}{2}{231--240}.
\PrintBackRefs{\CurrentBib}

\bibitem [\protect \citeauthoryear {%
Guedj%
}{%
Guedj%
}{%
{\protect \APACyear {2019}}%
}]{%
guedj2019primer}
\APACinsertmetastar {%
guedj2019primer}%
\begin{APACrefauthors}%
Guedj, B.%
\end{APACrefauthors}%
\unskip\
\newblock
\APACrefYearMonthDay{2019}{}{}.
\newblock
{\BBOQ}\APACrefatitle {A primer on {PAC-Bayesian} learning} {A primer on
  {PAC-Bayesian} learning}.{\BBCQ}
\newblock
\APACjournalVolNumPages{arXiv preprint arXiv:1901.05353}{}{}{}.
\PrintBackRefs{\CurrentBib}

\bibitem [\protect \citeauthoryear {%
Hou%
, Kassraie%
, Kratsios%
, Rothfuss%
\BCBL {}\ \BBA {} Krause%
}{%
Hou%
\ \protect \BOthers {.}}{%
{\protect \APACyear {2022}}%
}]{%
hou2022instance}
\APACinsertmetastar {%
hou2022instance}%
\begin{APACrefauthors}%
Hou, S.%
, Kassraie, P.%
, Kratsios, A.%
, Rothfuss, J.%
\BCBL {}\ \BBA {} Krause, A.%
\end{APACrefauthors}%
\unskip\
\newblock
\APACrefYearMonthDay{2022}{}{}.
\newblock
{\BBOQ}\APACrefatitle {Instance-Dependent Generalization Bounds via Optimal
  Transport} {Instance-dependent generalization bounds via optimal
  transport}.{\BBCQ}
\newblock
\APACjournalVolNumPages{arXiv preprint arXiv:2211.01258}{}{}{}.
\PrintBackRefs{\CurrentBib}

\bibitem [\protect \citeauthoryear {%
Hsu%
, Kakade%
\BCBL {}\ \BBA {} Zhang%
}{%
Hsu%
\ \protect \BOthers {.}}{%
{\protect \APACyear {2012}}%
}]{%
hsu2012tail}
\APACinsertmetastar {%
hsu2012tail}%
\begin{APACrefauthors}%
Hsu, D.%
, Kakade, S.%
\BCBL {}\ \BBA {} Zhang, T.%
\end{APACrefauthors}%
\unskip\
\newblock
\APACrefYearMonthDay{2012}{}{}.
\newblock
{\BBOQ}\APACrefatitle {A tail inequality for quadratic forms of subgaussian
  random vectors} {A tail inequality for quadratic forms of subgaussian random
  vectors}.{\BBCQ}
\newblock
\APACjournalVolNumPages{Electronic Communications in Probability}{17}{}{1--6}.
\PrintBackRefs{\CurrentBib}

\bibitem [\protect \citeauthoryear {%
Jiang%
, Neyshabur%
, Mobahi%
, Krishnan%
\BCBL {}\ \BBA {} Bengio%
}{%
Jiang%
\ \protect \BOthers {.}}{%
{\protect \APACyear {2019}}%
}]{%
jiang2019fantastic}
\APACinsertmetastar {%
jiang2019fantastic}%
\begin{APACrefauthors}%
Jiang, Y.%
, Neyshabur, B.%
, Mobahi, H.%
, Krishnan, D.%
\BCBL {}\ \BBA {} Bengio, S.%
\end{APACrefauthors}%
\unskip\
\newblock
\APACrefYearMonthDay{2019}{}{}.
\newblock
{\BBOQ}\APACrefatitle {Fantastic generalization measures and where to find
  them} {Fantastic generalization measures and where to find them}.{\BBCQ}
\newblock
\APACjournalVolNumPages{arXiv preprint arXiv:1912.02178}{}{}{}.
\PrintBackRefs{\CurrentBib}

\bibitem [\protect \citeauthoryear {%
D.~Kingma%
\ \BBA {} Ba%
}{%
D.~Kingma%
\ \BBA {} Ba%
}{%
{\protect \APACyear {2015}}%
}]{%
kingma2014adam}
\APACinsertmetastar {%
kingma2014adam}%
\begin{APACrefauthors}%
Kingma, D.%
\BCBT {}\ \BBA {} Ba, J.%
\end{APACrefauthors}%
\unskip\
\newblock
\APACrefYearMonthDay{2015}{}{}.
\newblock
{\BBOQ}\APACrefatitle {Adam: A method for stochastic optimization} {Adam: A
  method for stochastic optimization}.{\BBCQ}
\newblock
\BIn{} \APACrefbtitle {International Conference on Learning Representations
  ({ICLR}).} {International conference on learning representations ({ICLR}).}
\PrintBackRefs{\CurrentBib}

\bibitem [\protect \citeauthoryear {%
D\BPBI P.~Kingma%
\ \BBA {} Welling%
}{%
D\BPBI P.~Kingma%
\ \BBA {} Welling%
}{%
{\protect \APACyear {2013}}%
}]{%
kingma2013auto}
\APACinsertmetastar {%
kingma2013auto}%
\begin{APACrefauthors}%
Kingma, D\BPBI P.%
\BCBT {}\ \BBA {} Welling, M.%
\end{APACrefauthors}%
\unskip\
\newblock
\APACrefYearMonthDay{2013}{}{}.
\newblock
{\BBOQ}\APACrefatitle {Auto-encoding variational {Bayes}} {Auto-encoding
  variational {Bayes}}.{\BBCQ}
\newblock
\BIn{} \APACrefbtitle {International Conference on Learning Representations
  ({ICLR}).} {International conference on learning representations ({ICLR}).}
\PrintBackRefs{\CurrentBib}

\bibitem [\protect \citeauthoryear {%
Koltchinskii%
\ \BBA {} Panchenko%
}{%
Koltchinskii%
\ \BBA {} Panchenko%
}{%
{\protect \APACyear {2004}}%
}]{%
koltchinskii2004rademacher}
\APACinsertmetastar {%
koltchinskii2004rademacher}%
\begin{APACrefauthors}%
Koltchinskii, V.%
\BCBT {}\ \BBA {} Panchenko, D.%
\end{APACrefauthors}%
\unskip\
\newblock
\APACrefYearMonthDay{2004}{}{}.
\newblock
{\BBOQ}\APACrefatitle {Rademacher processes and bounding the risk of function
  learning} {Rademacher processes and bounding the risk of function
  learning}.{\BBCQ}
\newblock
\APACjournalVolNumPages{arXiv preprint math/0405338}{}{}{}.
\PrintBackRefs{\CurrentBib}

\bibitem [\protect \citeauthoryear {%
Lever%
, Laviolette%
\BCBL {}\ \BBA {} Shawe-Taylor%
}{%
Lever%
\ \protect \BOthers {.}}{%
{\protect \APACyear {2013}}%
}]{%
lever2013tighter}
\APACinsertmetastar {%
lever2013tighter}%
\begin{APACrefauthors}%
Lever, G.%
, Laviolette, F.%
\BCBL {}\ \BBA {} Shawe-Taylor, J.%
\end{APACrefauthors}%
\unskip\
\newblock
\APACrefYearMonthDay{2013}{}{}.
\newblock
{\BBOQ}\APACrefatitle {Tighter {PAC-B}ayes bounds through
  distribution-dependent priors.} {Tighter {PAC-B}ayes bounds through
  distribution-dependent priors.}{\BBCQ}
\newblock
\APACjournalVolNumPages{Theor. Comput. Sci.}{473}{Feb}{4--28}.
\PrintBackRefs{\CurrentBib}

\bibitem [\protect \citeauthoryear {%
Livni%
\ \BBA {} Moran%
}{%
Livni%
\ \BBA {} Moran%
}{%
{\protect \APACyear {2020}}%
}]{%
livni2020limitation}
\APACinsertmetastar {%
livni2020limitation}%
\begin{APACrefauthors}%
Livni, R.%
\BCBT {}\ \BBA {} Moran, S.%
\end{APACrefauthors}%
\unskip\
\newblock
\APACrefYearMonthDay{2020}{}{}.
\newblock
{\BBOQ}\APACrefatitle {A limitation of the {PAC-Bayes} framework} {A limitation
  of the {PAC-Bayes} framework}.{\BBCQ}
\newblock
\APACjournalVolNumPages{Advances in Neural Information Processing Systems
  ({NeurIPS})}{33}{}{20543--20553}.
\PrintBackRefs{\CurrentBib}

\bibitem [\protect \citeauthoryear {%
Lopez%
\ \BBA {} Jog%
}{%
Lopez%
\ \BBA {} Jog%
}{%
{\protect \APACyear {2018}}%
}]{%
lopez2018generalization}
\APACinsertmetastar {%
lopez2018generalization}%
\begin{APACrefauthors}%
Lopez, A\BPBI T.%
\BCBT {}\ \BBA {} Jog, V.%
\end{APACrefauthors}%
\unskip\
\newblock
\APACrefYearMonthDay{2018}{}{}.
\newblock
{\BBOQ}\APACrefatitle {Generalization error bounds using Wasserstein distances}
  {Generalization error bounds using wasserstein distances}.{\BBCQ}
\newblock
\BIn{} \APACrefbtitle {2018 IEEE Information Theory Workshop ({ITW})} {2018
  ieee information theory workshop ({ITW})}\ (\BPGS\ 1--5).
\PrintBackRefs{\CurrentBib}

\bibitem [\protect \citeauthoryear {%
Lugosi%
}{%
Lugosi%
}{%
{\protect \APACyear {2002}}%
}]{%
lugosi2002pattern}
\APACinsertmetastar {%
lugosi2002pattern}%
\begin{APACrefauthors}%
Lugosi, G.%
\end{APACrefauthors}%
\unskip\
\newblock
\APACrefYearMonthDay{2002}{}{}.
\newblock
{\BBOQ}\APACrefatitle {Pattern classification and learning theory} {Pattern
  classification and learning theory}.{\BBCQ}
\newblock
\BIn{} \APACrefbtitle {Principles of nonparametric learning} {Principles of
  nonparametric learning}\ (\BPGS\ 1--56).
\newblock
\APACaddressPublisher{}{Springer}.
\PrintBackRefs{\CurrentBib}

\bibitem [\protect \citeauthoryear {%
Mardia%
, Jiao%
, Tánczos%
, Nowak%
\BCBL {}\ \BBA {} Weissman%
}{%
Mardia%
\ \protect \BOthers {.}}{%
{\protect \APACyear {2019}}%
}]{%
Mardia19}
\APACinsertmetastar {%
Mardia19}%
\begin{APACrefauthors}%
Mardia, J.%
, Jiao, J.%
, Tánczos, E.%
, Nowak, R\BPBI D.%
\BCBL {}\ \BBA {} Weissman, T.%
\end{APACrefauthors}%
\unskip\
\newblock
\APACrefYearMonthDay{2019}{11}{}.
\newblock
{\BBOQ}\APACrefatitle {{Concentration inequalities for the empirical
  distribution of discrete distributions: beyond the method of types}}
  {{Concentration inequalities for the empirical distribution of discrete
  distributions: beyond the method of types}}.{\BBCQ}
\newblock
\APACjournalVolNumPages{Information and Inference: A Journal of the
  IMA}{9}{4}{813-850}.
\PrintBackRefs{\CurrentBib}

\bibitem [\protect \citeauthoryear {%
Maurer%
}{%
Maurer%
}{%
{\protect \APACyear {2004}}%
}]{%
Maurer_04}
\APACinsertmetastar {%
Maurer_04}%
\begin{APACrefauthors}%
Maurer, A.%
\end{APACrefauthors}%
\unskip\
\newblock
\APACrefYearMonthDay{2004}{}{}.
\newblock
{\BBOQ}\APACrefatitle {A Note on the {PAC B}ayesian Theorem} {A note on the
  {PAC B}ayesian theorem}.{\BBCQ}
\newblock
\APACjournalVolNumPages{CoRR abs/cs/0411099}{}{}{}.
\PrintBackRefs{\CurrentBib}

\bibitem [\protect \citeauthoryear {%
McAllester%
}{%
McAllester%
}{%
{\protect \APACyear {1998}}%
}]{%
McAllester_98}
\APACinsertmetastar {%
McAllester_98}%
\begin{APACrefauthors}%
McAllester, D.%
\end{APACrefauthors}%
\unskip\
\newblock
\APACrefYearMonthDay{1998}{}{}.
\newblock
{\BBOQ}\APACrefatitle {Some {PAC-B}ayesian theorems} {Some {PAC-B}ayesian
  theorems}.{\BBCQ}
\newblock
\APACjournalVolNumPages{Conference on Learning Theory (COLT)}{}{}{}.
\PrintBackRefs{\CurrentBib}

\bibitem [\protect \citeauthoryear {%
McAllester%
}{%
McAllester%
}{%
{\protect \APACyear {2003}}%
}]{%
McAllester_03}
\APACinsertmetastar {%
McAllester_03}%
\begin{APACrefauthors}%
McAllester, D.%
\end{APACrefauthors}%
\unskip\
\newblock
\APACrefYearMonthDay{2003}{}{}.
\newblock
{\BBOQ}\APACrefatitle {Simplified {PAC-B}ayesian margin bounds} {Simplified
  {PAC-B}ayesian margin bounds}.{\BBCQ}
\newblock
\APACjournalVolNumPages{Conference on Learning Theory (COLT)}{}{}{}.
\PrintBackRefs{\CurrentBib}

\bibitem [\protect \citeauthoryear {%
Miyaguchi%
}{%
Miyaguchi%
}{%
{\protect \APACyear {2019}}%
}]{%
miyaguchi2019pac}
\APACinsertmetastar {%
miyaguchi2019pac}%
\begin{APACrefauthors}%
Miyaguchi, K.%
\end{APACrefauthors}%
\unskip\
\newblock
\APACrefYearMonthDay{2019}{}{}.
\newblock
{\BBOQ}\APACrefatitle {{PAC-B}ayesian Transportation Bound} {{PAC-B}ayesian
  transportation bound}.{\BBCQ}
\newblock
\APACjournalVolNumPages{arXiv preprint arXiv:1905.13435}{}{}{}.
\PrintBackRefs{\CurrentBib}

\bibitem [\protect \citeauthoryear {%
M{\"u}ller%
}{%
M{\"u}ller%
}{%
{\protect \APACyear {1997}}%
}]{%
muller1997stochastic}
\APACinsertmetastar {%
muller1997stochastic}%
\begin{APACrefauthors}%
M{\"u}ller, A.%
\end{APACrefauthors}%
\unskip\
\newblock
\APACrefYearMonthDay{1997}{}{}.
\newblock
{\BBOQ}\APACrefatitle {Stochastic orders generated by integrals: a unified
  study} {Stochastic orders generated by integrals: a unified study}.{\BBCQ}
\newblock
\APACjournalVolNumPages{Advances in Applied probability}{29}{2}{414--428}.
\PrintBackRefs{\CurrentBib}

\bibitem [\protect \citeauthoryear {%
Nagarajan%
\ \BBA {} Kolter%
}{%
Nagarajan%
\ \BBA {} Kolter%
}{%
{\protect \APACyear {2019}}%
{\protect \APACexlab {{\protect \BCnt {1}}}}}]{%
nagarajan2019deterministic}
\APACinsertmetastar {%
nagarajan2019deterministic}%
\begin{APACrefauthors}%
Nagarajan, V.%
\BCBT {}\ \BBA {} Kolter, J\BPBI Z.%
\end{APACrefauthors}%
\unskip\
\newblock
\APACrefYearMonthDay{2019{\protect \BCnt {1}}}{}{}.
\newblock
{\BBOQ}\APACrefatitle {Deterministic {PAC-B}ayesian generalization bounds for
  deep networks via generalizing noise-resilience} {Deterministic
  {PAC-B}ayesian generalization bounds for deep networks via generalizing
  noise-resilience}.{\BBCQ}
\newblock
\APACjournalVolNumPages{arXiv preprint arXiv:1905.13344}{}{}{}.
\PrintBackRefs{\CurrentBib}

\bibitem [\protect \citeauthoryear {%
Nagarajan%
\ \BBA {} Kolter%
}{%
Nagarajan%
\ \BBA {} Kolter%
}{%
{\protect \APACyear {2019}}%
{\protect \APACexlab {{\protect \BCnt {2}}}}}]{%
nagarajan2019uniform}
\APACinsertmetastar {%
nagarajan2019uniform}%
\begin{APACrefauthors}%
Nagarajan, V.%
\BCBT {}\ \BBA {} Kolter, J\BPBI Z.%
\end{APACrefauthors}%
\unskip\
\newblock
\APACrefYearMonthDay{2019{\protect \BCnt {2}}}{}{}.
\newblock
{\BBOQ}\APACrefatitle {Uniform convergence may be unable to explain
  generalization in deep learning} {Uniform convergence may be unable to
  explain generalization in deep learning}.{\BBCQ}
\newblock
\APACjournalVolNumPages{Advances in Neural Information Processing Systems
  ({NeurIPS})}{32}{}{}.
\PrintBackRefs{\CurrentBib}

\bibitem [\protect \citeauthoryear {%
Neu%
\ \BBA {} Lugosi%
}{%
Neu%
\ \BBA {} Lugosi%
}{%
{\protect \APACyear {2022}}%
}]{%
neu2022generalization}
\APACinsertmetastar {%
neu2022generalization}%
\begin{APACrefauthors}%
Neu, G.%
\BCBT {}\ \BBA {} Lugosi, G.%
\end{APACrefauthors}%
\unskip\
\newblock
\APACrefYearMonthDay{2022}{}{}.
\newblock
{\BBOQ}\APACrefatitle {Generalization Bounds via Convex Analysis}
  {Generalization bounds via convex analysis}.{\BBCQ}
\newblock
\APACjournalVolNumPages{arXiv preprint arXiv:2202.04985}{}{}{}.
\PrintBackRefs{\CurrentBib}

\bibitem [\protect \citeauthoryear {%
Neyshabur%
, Bhojanapalli%
, McAllester%
\BCBL {}\ \BBA {} Srebro%
}{%
Neyshabur%
\ \protect \BOthers {.}}{%
{\protect \APACyear {2017}}%
}]{%
neyshabur2017exploring}
\APACinsertmetastar {%
neyshabur2017exploring}%
\begin{APACrefauthors}%
Neyshabur, B.%
, Bhojanapalli, S.%
, McAllester, D.%
\BCBL {}\ \BBA {} Srebro, N.%
\end{APACrefauthors}%
\unskip\
\newblock
\APACrefYearMonthDay{2017}{}{}.
\newblock
{\BBOQ}\APACrefatitle {Exploring generalization in deep learning} {Exploring
  generalization in deep learning}.{\BBCQ}
\newblock
\APACjournalVolNumPages{Advances in Neural Information Processing Systems
  ({NeurIPS})}{30}{}{}.
\PrintBackRefs{\CurrentBib}

\bibitem [\protect \citeauthoryear {%
Neyshabur%
, Tomioka%
\BCBL {}\ \BBA {} Srebro%
}{%
Neyshabur%
\ \protect \BOthers {.}}{%
{\protect \APACyear {2015}}%
}]{%
neyshabur2015norm}
\APACinsertmetastar {%
neyshabur2015norm}%
\begin{APACrefauthors}%
Neyshabur, B.%
, Tomioka, R.%
\BCBL {}\ \BBA {} Srebro, N.%
\end{APACrefauthors}%
\unskip\
\newblock
\APACrefYearMonthDay{2015}{}{}.
\newblock
{\BBOQ}\APACrefatitle {Norm-based capacity control in neural networks}
  {Norm-based capacity control in neural networks}.{\BBCQ}
\newblock
\BIn{} \APACrefbtitle {Conference on Learning Theory} {Conference on learning
  theory}\ (\BPGS\ 1376--1401).
\PrintBackRefs{\CurrentBib}

\bibitem [\protect \citeauthoryear {%
Ohnishi%
\ \BBA {} Honorio%
}{%
Ohnishi%
\ \BBA {} Honorio%
}{%
{\protect \APACyear {2021}}%
}]{%
ohnishi2021novel}
\APACinsertmetastar {%
ohnishi2021novel}%
\begin{APACrefauthors}%
Ohnishi, Y.%
\BCBT {}\ \BBA {} Honorio, J.%
\end{APACrefauthors}%
\unskip\
\newblock
\APACrefYearMonthDay{2021}{}{}.
\newblock
{\BBOQ}\APACrefatitle {Novel change of measure inequalities with applications
  to {PAC-Bayes}ian bounds and Monte Carlo estimation} {Novel change of measure
  inequalities with applications to {PAC-Bayes}ian bounds and monte carlo
  estimation}.{\BBCQ}
\newblock
\BIn{} \APACrefbtitle {International Conference on Artificial Intelligence and
  Statistics} {International conference on artificial intelligence and
  statistics}\ (\BPGS\ 1711--1719).
\PrintBackRefs{\CurrentBib}

\bibitem [\protect \citeauthoryear {%
Paszke%
\ \protect \BOthers {.}}{%
Paszke%
\ \protect \BOthers {.}}{%
{\protect \APACyear {2019}}%
}]{%
paszke2019pytorch}
\APACinsertmetastar {%
paszke2019pytorch}%
\begin{APACrefauthors}%
Paszke, A.%
, Gross, S.%
, Massa, F.%
, Lerer, A.%
, Bradbury, J.%
, Chanan, G.%
\BDBL {}others%
\end{APACrefauthors}%
\unskip\
\newblock
\APACrefYearMonthDay{2019}{}{}.
\newblock
{\BBOQ}\APACrefatitle {Pytorch: An imperative style, high-performance deep
  learning library} {Pytorch: An imperative style, high-performance deep
  learning library}.{\BBCQ}
\newblock
\APACjournalVolNumPages{Advances in Neural Information Processing Systems
  ({NeurIPS})}{32}{}{}.
\PrintBackRefs{\CurrentBib}

\bibitem [\protect \citeauthoryear {%
Picard-Weibel%
\ \BBA {} Guedj%
}{%
Picard-Weibel%
\ \BBA {} Guedj%
}{%
{\protect \APACyear {2022}}%
}]{%
picard2022change}
\APACinsertmetastar {%
picard2022change}%
\begin{APACrefauthors}%
Picard-Weibel, A.%
\BCBT {}\ \BBA {} Guedj, B.%
\end{APACrefauthors}%
\unskip\
\newblock
\APACrefYearMonthDay{2022}{}{}.
\newblock
{\BBOQ}\APACrefatitle {On change of measure inequalities for $f$-divergences}
  {On change of measure inequalities for $f$-divergences}.{\BBCQ}
\newblock
\APACjournalVolNumPages{arXiv preprint arXiv:2202.05568}{}{}{}.
\PrintBackRefs{\CurrentBib}

\bibitem [\protect \citeauthoryear {%
Rivasplata%
, Kuzborskij%
, Szepesv{\'a}ri%
\BCBL {}\ \BBA {} Shawe-Taylor%
}{%
Rivasplata%
\ \protect \BOthers {.}}{%
{\protect \APACyear {2020}}%
}]{%
rivasplata2020pac}
\APACinsertmetastar {%
rivasplata2020pac}%
\begin{APACrefauthors}%
Rivasplata, O.%
, Kuzborskij, I.%
, Szepesv{\'a}ri, C.%
\BCBL {}\ \BBA {} Shawe-Taylor, J.%
\end{APACrefauthors}%
\unskip\
\newblock
\APACrefYearMonthDay{2020}{}{}.
\newblock
{\BBOQ}\APACrefatitle {{PAC-Bayes} analysis beyond the usual bounds}
  {{PAC-Bayes} analysis beyond the usual bounds}.{\BBCQ}
\newblock
\APACjournalVolNumPages{Advances in Neural Information Processing Systems
  ({NeurIPS})}{33}{}{16833--16845}.
\PrintBackRefs{\CurrentBib}

\bibitem [\protect \citeauthoryear {%
Rodr{\'\i}guez~G{\'a}lvez%
, Bassi%
, Thobaben%
\BCBL {}\ \BBA {} Skoglund%
}{%
Rodr{\'\i}guez~G{\'a}lvez%
\ \protect \BOthers {.}}{%
{\protect \APACyear {2021}}%
}]{%
rodriguez2021tighter}
\APACinsertmetastar {%
rodriguez2021tighter}%
\begin{APACrefauthors}%
Rodr{\'\i}guez~G{\'a}lvez, B.%
, Bassi, G.%
, Thobaben, R.%
\BCBL {}\ \BBA {} Skoglund, M.%
\end{APACrefauthors}%
\unskip\
\newblock
\APACrefYearMonthDay{2021}{}{}.
\newblock
{\BBOQ}\APACrefatitle {Tighter expected generalization error bounds via
  {W}asserstein distance} {Tighter expected generalization error bounds via
  {W}asserstein distance}.{\BBCQ}
\newblock
\APACjournalVolNumPages{Advances in Neural Information Processing Systems
  ({NeurIPS})}{34}{}{}.
\PrintBackRefs{\CurrentBib}

\bibitem [\protect \citeauthoryear {%
Seeger%
}{%
Seeger%
}{%
{\protect \APACyear {2002}}%
}]{%
Seeger_02}
\APACinsertmetastar {%
Seeger_02}%
\begin{APACrefauthors}%
Seeger, M.%
\end{APACrefauthors}%
\unskip\
\newblock
\APACrefYearMonthDay{2002}{}{}.
\newblock
{\BBOQ}\APACrefatitle {{PAC-B}ayesian generalisation bounds for Gaussian
  processes} {{PAC-B}ayesian generalisation bounds for gaussian
  processes}.{\BBCQ}
\newblock
\APACjournalVolNumPages{Journal of Machine Learning Research, 3:233-269}{}{}{}.
\PrintBackRefs{\CurrentBib}

\bibitem [\protect \citeauthoryear {%
Seldin%
\ \BBA {} Tishby%
}{%
Seldin%
\ \BBA {} Tishby%
}{%
{\protect \APACyear {2010}}%
}]{%
Seldin_Tishby_Unsup_PB}
\APACinsertmetastar {%
Seldin_Tishby_Unsup_PB}%
\begin{APACrefauthors}%
Seldin, Y.%
\BCBT {}\ \BBA {} Tishby, N.%
\end{APACrefauthors}%
\unskip\
\newblock
\APACrefYearMonthDay{2010}{}{}.
\newblock
{\BBOQ}\APACrefatitle {{PAC-B}ayesian Analysis of Co-clustering and Beyond}
  {{PAC-B}ayesian analysis of co-clustering and beyond}.{\BBCQ}
\newblock
\APACjournalVolNumPages{JMLR 2010}{}{}{}.
\PrintBackRefs{\CurrentBib}

\bibitem [\protect \citeauthoryear {%
Shalev-Shwartz%
\ \BBA {} Ben-David%
}{%
Shalev-Shwartz%
\ \BBA {} Ben-David%
}{%
{\protect \APACyear {2014}}%
}]{%
SSS}
\APACinsertmetastar {%
SSS}%
\begin{APACrefauthors}%
Shalev-Shwartz, S.%
\BCBT {}\ \BBA {} Ben-David, S.%
\end{APACrefauthors}%
\unskip\
\newblock
\APACrefYear{2014}.
\newblock
\APACrefbtitle {Understanding Machine Learning: From Theory to Algorithms}
  {Understanding machine learning: From theory to algorithms}.
\newblock
\APACaddressPublisher{}{Cambridge University Press}.
\PrintBackRefs{\CurrentBib}

\bibitem [\protect \citeauthoryear {%
Shawe-Taylor%
\ \BBA {} Williamson%
}{%
Shawe-Taylor%
\ \BBA {} Williamson%
}{%
{\protect \APACyear {1997}}%
}]{%
Shawe-Taylor_Williamson_First_PAC-Bayes}
\APACinsertmetastar {%
Shawe-Taylor_Williamson_First_PAC-Bayes}%
\begin{APACrefauthors}%
Shawe-Taylor, J.%
\BCBT {}\ \BBA {} Williamson, R\BPBI C.%
\end{APACrefauthors}%
\unskip\
\newblock
\APACrefYearMonthDay{1997}{}{}.
\newblock
{\BBOQ}\APACrefatitle {A {PAC} analysis of a {B}ayesian estimator} {A {PAC}
  analysis of a {B}ayesian estimator}.{\BBCQ}
\newblock
\APACjournalVolNumPages{Proceedings of the International Conference on
  Computational Learning Theory (COLT)}{}{}{}.
\PrintBackRefs{\CurrentBib}

\bibitem [\protect \citeauthoryear {%
Sriperumbudur%
, Fukumizu%
, Gretton%
, Sch{\"o}lkopf%
\BCBL {}\ \BBA {} Lanckriet%
}{%
Sriperumbudur%
\ \protect \BOthers {.}}{%
{\protect \APACyear {2009}}%
}]{%
sriperumbudur2009integral}
\APACinsertmetastar {%
sriperumbudur2009integral}%
\begin{APACrefauthors}%
Sriperumbudur, B\BPBI K.%
, Fukumizu, K.%
, Gretton, A.%
, Sch{\"o}lkopf, B.%
\BCBL {}\ \BBA {} Lanckriet, G\BPBI R.%
\end{APACrefauthors}%
\unskip\
\newblock
\APACrefYearMonthDay{2009}{}{}.
\newblock
{\BBOQ}\APACrefatitle {On integral probability metrics, $\Phi$-divergences and
  binary classification} {On integral probability metrics, $\phi$-divergences
  and binary classification}.{\BBCQ}
\newblock
\APACjournalVolNumPages{arXiv preprint arXiv:0901.2698}{}{}{}.
\PrintBackRefs{\CurrentBib}

\bibitem [\protect \citeauthoryear {%
Sriperumbudur%
, Fukumizu%
, Gretton%
, Sch{\"o}lkopf%
\BCBL {}\ \BBA {} Lanckriet%
}{%
Sriperumbudur%
\ \protect \BOthers {.}}{%
{\protect \APACyear {2012}}%
}]{%
sriperumbudur2012empirical}
\APACinsertmetastar {%
sriperumbudur2012empirical}%
\begin{APACrefauthors}%
Sriperumbudur, B\BPBI K.%
, Fukumizu, K.%
, Gretton, A.%
, Sch{\"o}lkopf, B.%
\BCBL {}\ \BBA {} Lanckriet, G\BPBI R.%
\end{APACrefauthors}%
\unskip\
\newblock
\APACrefYearMonthDay{2012}{}{}.
\newblock
{\BBOQ}\APACrefatitle {On the empirical estimation of integral probability
  metrics} {On the empirical estimation of integral probability
  metrics}.{\BBCQ}
\newblock
\APACjournalVolNumPages{Electronic Journal of Statistics}{6}{}{1550--1599}.
\PrintBackRefs{\CurrentBib}

\bibitem [\protect \citeauthoryear {%
Talagrand%
}{%
Talagrand%
}{%
{\protect \APACyear {1994}}%
}]{%
talagrand1994sharper}
\APACinsertmetastar {%
talagrand1994sharper}%
\begin{APACrefauthors}%
Talagrand, M.%
\end{APACrefauthors}%
\unskip\
\newblock
\APACrefYearMonthDay{1994}{}{}.
\newblock
{\BBOQ}\APACrefatitle {Sharper bounds for Gaussian and empirical processes}
  {Sharper bounds for gaussian and empirical processes}.{\BBCQ}
\newblock
\APACjournalVolNumPages{The Annals of Probability}{}{}{28--76}.
\PrintBackRefs{\CurrentBib}

\bibitem [\protect \citeauthoryear {%
Thorpe%
}{%
Thorpe%
}{%
{\protect \APACyear {2018}}%
}]{%
thorpe2018introduction}
\APACinsertmetastar {%
thorpe2018introduction}%
\begin{APACrefauthors}%
Thorpe, M.%
\end{APACrefauthors}%
\unskip\
\newblock
\APACrefYearMonthDay{2018}{}{}.
\newblock
{\BBOQ}\APACrefatitle {Introduction to optimal transport} {Introduction to
  optimal transport}.{\BBCQ}
\newblock
\APACjournalVolNumPages{Notes of Course at University of Cambridge.}{}{}{}.
\PrintBackRefs{\CurrentBib}

\bibitem [\protect \citeauthoryear {%
Vapnik%
}{%
Vapnik%
}{%
{\protect \APACyear {1999}}%
}]{%
vapnik1999nature}
\APACinsertmetastar {%
vapnik1999nature}%
\begin{APACrefauthors}%
Vapnik, V.%
\end{APACrefauthors}%
\unskip\
\newblock
\APACrefYear{1999}.
\newblock
\APACrefbtitle {The nature of statistical learning theory} {The nature of
  statistical learning theory}.
\newblock
\APACaddressPublisher{}{Springer science \& business media}.
\PrintBackRefs{\CurrentBib}

\bibitem [\protect \citeauthoryear {%
Vapnik%
\ \BBA {} Chervonenkis%
}{%
Vapnik%
\ \BBA {} Chervonenkis%
}{%
{\protect \APACyear {2015}}%
}]{%
vapnik2015uniform}
\APACinsertmetastar {%
vapnik2015uniform}%
\begin{APACrefauthors}%
Vapnik, V.%
\BCBT {}\ \BBA {} Chervonenkis, A\BPBI Y.%
\end{APACrefauthors}%
\unskip\
\newblock
\APACrefYearMonthDay{2015}{}{}.
\newblock
{\BBOQ}\APACrefatitle {On the uniform convergence of relative frequencies of
  events to their probabilities} {On the uniform convergence of relative
  frequencies of events to their probabilities}.{\BBCQ}
\newblock
\BIn{} \APACrefbtitle {Measures of complexity} {Measures of complexity}\
  (\BPGS\ 11--30).
\newblock
\APACaddressPublisher{}{Springer}.
\PrintBackRefs{\CurrentBib}

\bibitem [\protect \citeauthoryear {%
Villani%
}{%
Villani%
}{%
{\protect \APACyear {2006}}%
}]{%
Villani06}
\APACinsertmetastar {%
Villani06}%
\begin{APACrefauthors}%
Villani, C.%
\end{APACrefauthors}%
\unskip\
\newblock
\APACrefYear{2006}.
\newblock
\APACrefbtitle {Optimal transport: old and new} {Optimal transport: old and
  new}.
\newblock
\APACaddressPublisher{}{Springer Science and Business Media}.
\PrintBackRefs{\CurrentBib}

\bibitem [\protect \citeauthoryear {%
Wainwright%
}{%
Wainwright%
}{%
{\protect \APACyear {2019}}%
}]{%
wainwright2019high}
\APACinsertmetastar {%
wainwright2019high}%
\begin{APACrefauthors}%
Wainwright, M\BPBI J.%
\end{APACrefauthors}%
\unskip\
\newblock
\APACrefYear{2019}.
\newblock
\APACrefbtitle {High-dimensional statistics: A non-asymptotic viewpoint}
  {High-dimensional statistics: A non-asymptotic viewpoint}\ (\BVOL~48).
\newblock
\APACaddressPublisher{}{Cambridge University Press}.
\PrintBackRefs{\CurrentBib}

\bibitem [\protect \citeauthoryear {%
Wang%
, Diaz%
, Santos~Filho%
\BCBL {}\ \BBA {} Calmon%
}{%
Wang%
\ \protect \BOthers {.}}{%
{\protect \APACyear {2019}}%
}]{%
wang2019information}
\APACinsertmetastar {%
wang2019information}%
\begin{APACrefauthors}%
Wang, H.%
, Diaz, M.%
, Santos~Filho, J\BPBI C\BPBI S.%
\BCBL {}\ \BBA {} Calmon, F\BPBI P.%
\end{APACrefauthors}%
\unskip\
\newblock
\APACrefYearMonthDay{2019}{}{}.
\newblock
{\BBOQ}\APACrefatitle {An information-theoretic view of generalization via
  {W}asserstein distance} {An information-theoretic view of generalization via
  {W}asserstein distance}.{\BBCQ}
\newblock
\BIn{} \APACrefbtitle {2019 {IEEE} International Symposium on Information
  Theory ({ISIT})} {2019 {IEEE} international symposium on information theory
  ({ISIT})}\ (\BPGS\ 577--581).
\PrintBackRefs{\CurrentBib}

\bibitem [\protect \citeauthoryear {%
Weed%
\ \BBA {} Bach%
}{%
Weed%
\ \BBA {} Bach%
}{%
{\protect \APACyear {2017}}%
}]{%
Weed_Bach_17}
\APACinsertmetastar {%
Weed_Bach_17}%
\begin{APACrefauthors}%
Weed, J.%
\BCBT {}\ \BBA {} Bach, F.%
\end{APACrefauthors}%
\unskip\
\newblock
\APACrefYearMonthDay{2017}{}{}.
\newblock
{\BBOQ}\APACrefatitle {Sharp asymptotic and finite-sample rates of convergence
  of empirical measures in {W}asserstein distance} {Sharp asymptotic and
  finite-sample rates of convergence of empirical measures in {W}asserstein
  distance}.{\BBCQ}
\newblock
\APACjournalVolNumPages{Bernoulli 25 (4A), 2620-2648}{}{}{}.
\PrintBackRefs{\CurrentBib}

\bibitem [\protect \citeauthoryear {%
Wei%
\ \BBA {} Ma%
}{%
Wei%
\ \BBA {} Ma%
}{%
{\protect \APACyear {2019}}%
}]{%
wei2019data}
\APACinsertmetastar {%
wei2019data}%
\begin{APACrefauthors}%
Wei, C.%
\BCBT {}\ \BBA {} Ma, T.%
\end{APACrefauthors}%
\unskip\
\newblock
\APACrefYearMonthDay{2019}{}{}.
\newblock
{\BBOQ}\APACrefatitle {Data-dependent Sample Complexity of Deep Neural Networks
  via Lipschitz Augmentation} {Data-dependent sample complexity of deep neural
  networks via lipschitz augmentation}.{\BBCQ}
\newblock
\APACjournalVolNumPages{arXiv preprint arXiv:1905.03684}{}{}{}.
\PrintBackRefs{\CurrentBib}

\bibitem [\protect \citeauthoryear {%
Xu%
\ \BBA {} Mannor%
}{%
Xu%
\ \BBA {} Mannor%
}{%
{\protect \APACyear {2012}}%
}]{%
xu2012robustness}
\APACinsertmetastar {%
xu2012robustness}%
\begin{APACrefauthors}%
Xu, H.%
\BCBT {}\ \BBA {} Mannor, S.%
\end{APACrefauthors}%
\unskip\
\newblock
\APACrefYearMonthDay{2012}{}{}.
\newblock
{\BBOQ}\APACrefatitle {Robustness and generalization} {Robustness and
  generalization}.{\BBCQ}
\newblock
\APACjournalVolNumPages{Machine learning}{86}{3}{391--423}.
\PrintBackRefs{\CurrentBib}

\bibitem [\protect \citeauthoryear {%
Yang%
, Sun%
\BCBL {}\ \BBA {} Roy%
}{%
Yang%
\ \protect \BOthers {.}}{%
{\protect \APACyear {2019}}%
}]{%
yang2019fast}
\APACinsertmetastar {%
yang2019fast}%
\begin{APACrefauthors}%
Yang, J.%
, Sun, S.%
\BCBL {}\ \BBA {} Roy, D\BPBI M.%
\end{APACrefauthors}%
\unskip\
\newblock
\APACrefYearMonthDay{2019}{}{}.
\newblock
{\BBOQ}\APACrefatitle {Fast-rate {PAC-B}ayes Generalization Bounds via Shifted
  {Rademacher} Processes} {Fast-rate {PAC-B}ayes generalization bounds via
  shifted {Rademacher} processes}.{\BBCQ}
\newblock
\APACjournalVolNumPages{arXiv preprint arXiv:1908.07585}{}{}{}.
\PrintBackRefs{\CurrentBib}

\bibitem [\protect \citeauthoryear {%
C.~Zhang%
, Bengio%
, Hardt%
, Recht%
\BCBL {}\ \BBA {} Vinyals%
}{%
C.~Zhang%
\ \protect \BOthers {.}}{%
{\protect \APACyear {2017}}%
}]{%
Rethinking_Generalization}
\APACinsertmetastar {%
Rethinking_Generalization}%
\begin{APACrefauthors}%
Zhang, C.%
, Bengio, S.%
, Hardt, M.%
, Recht, B.%
\BCBL {}\ \BBA {} Vinyals, O.%
\end{APACrefauthors}%
\unskip\
\newblock
\APACrefYearMonthDay{2017}{}{}.
\newblock
{\BBOQ}\APACrefatitle {Understanding deep learning requires rethinking
  generalization} {Understanding deep learning requires rethinking
  generalization}.{\BBCQ}
\newblock
\APACjournalVolNumPages{International Conference on Learning Representations
  ({ICLR})}{}{}{}.
\PrintBackRefs{\CurrentBib}

\bibitem [\protect \citeauthoryear {%
J.~Zhang%
, Liu%
\BCBL {}\ \BBA {} Tao%
}{%
J.~Zhang%
\ \protect \BOthers {.}}{%
{\protect \APACyear {2021}}%
}]{%
zhang2021optimal}
\APACinsertmetastar {%
zhang2021optimal}%
\begin{APACrefauthors}%
Zhang, J.%
, Liu, T.%
\BCBL {}\ \BBA {} Tao, D.%
\end{APACrefauthors}%
\unskip\
\newblock
\APACrefYearMonthDay{2021}{}{}.
\newblock
{\BBOQ}\APACrefatitle {An Optimal Transport Analysis on Generalization in Deep
  Learning} {An optimal transport analysis on generalization in deep
  learning}.{\BBCQ}
\newblock
\APACjournalVolNumPages{IEEE Transactions on Neural Networks and Learning
  Systems}{}{}{}.
\PrintBackRefs{\CurrentBib}

\end{thebibliography}

%%%%%%%%%%%%%%%%%%%%%%%%%%%%%%%%%%%%%%%%%%%%%%%%%%%%%%%%%%%%
\appendix
%%%%%%%%

%%%%%%%%%%%%%%%%%%%%%%%%%%%%%%%%%%%%%%%%%%%%%%%%%%%%%%%%%%%%
\section*{Checklist}
\begin{enumerate}

\item For all authors...
\begin{enumerate}
  \item Do the main claims made in the abstract and introduction accurately reflect the paper's contributions and scope?
    \answerYes{}
  \item Did you describe the limitations of your work?
 \answerYes{}
  \item Did you discuss any potential negative societal impacts of your work?
    \answerNA{}
  \item Have you read the ethics review guidelines and ensured that your paper conforms to them?
    \answerYes{}
\end{enumerate}

\item If you are including theoretical results...
\begin{enumerate}
  \item Did you state the full set of assumptions of all theoretical results?
  \answerYes{}
        \item Did you include complete proofs of all theoretical results?
     \answerYes{}
\end{enumerate}

\item If you ran experiments...
\begin{enumerate}
  \item Did you include the code, data, and instructions needed to reproduce the main experimental results (either in the supplemental material or as a URL)?
  \answerYes{}
  \item Did you specify all the training details (e.g., data splits, hyperparameters, how they were chosen)?
 \answerYes{}
        \item Did you report error bars (e.g., with respect to the random seed after running experiments multiple times)?
   \answerYes{}
        \item Did you include the total amount of compute and the type of resources used (e.g., type of GPUs, internal cluster, or cloud provider)?
  \answerNA{}
\end{enumerate}

\item If you are using existing assets (e.g., code, data, models) or curating/releasing new assets...
\begin{enumerate}
  \item If your work uses existing assets, did you cite the creators?
  \answerNA{}
  \item Did you mention the license of the assets?
    \answerNA{}
  \item Did you include any new assets either in the supplemental material or as a URL?
  \answerNA{}
  \item Did you discuss whether and how consent was obtained from people whose data you're using/curating?
  \answerNA{}
  \item Did you discuss whether the data you are using/curating contains personally identifiable information or offensive content?
 \answerNA{}
\end{enumerate}

\item If you used crowdsourcing or conducted research with human subjects...
\begin{enumerate}
  \item Did you include the full text of instructions given to participants and screenshots, if applicable?
  \answerNA{}
  \item Did you describe any potential participant risks, with links to Institutional Review Board (IRB) approvals, if applicable?
  \answerNA{}
  \item Did you include the estimated hourly wage paid to participants and the total amount spent on participant compensation?
  \answerNA{}
\end{enumerate}

\end{enumerate}

% \newpage
%%%%%%%%%%%%%%%%%%%%%%%%%%%%%%%%%%%%%%%%%%%%%%%%%%%%%%%%%%%%

\newpage 
\section{Appendix: Proofs} 
%%%%%%%%%%%%%%%%%%%%%%%%%%%%%%%%%%%%%%%%%%%%%%%%%%%%%%%%%%%%
\subsection{Proof of the IPM-PB Bound (Prop. \ref{prop:PACBayesBound_IPM})} \label{sect:proof_thm:PACBayesBound_IPM}
%%%%%%%%%%%%%%%%%%%%%%%%%%%%%%%%%%%%%%%%%%%%%%%%%%%%%%%%%%%%

\begin{proof}
The proof follows a similar structure as the classical derivation \citep{McAllester_03, SSS}, except for replacing the change-of-measure inequality. 

For any sample $S \in \calZ^{m}$, consider the function
\begin{equation} 
    f_{S}(h)\defeq 2(m-1)\Delta_{S}^{2}(h). \nn
\end{equation}

Since we assume $f_S\in \calF_{S}$, Definition \ref{def:IPM} implies that for any pair of probability measures $P,Q \in \calM(\calH)$
\begin{align} 
\Expct hQ\brs*{f_{S}(h)} - \Expct hP \brs*{f_{S}(h)}  & \le \gamma_{\calF_{S}}(Q, P). \nn
\end{align}
Therefore, by the monotonicity of $\exp(\cdot)$ we have 
\begin{align} 
\exp\br*{\Expct hQ\brs*{f_{S}(h)}- \gamma_{\calF_{S}}(Q, P)} & \le \exp\br*{\Expct hP\brs*{f_{S}(h)}} \\
 & \le\Expct hP\left[\exp\br*{f_{S}(h)}\right]. 
\end{align}
where the last inequality is by the convexity of $\exp(\cdot)$, and by Jensen's inequality.

Taking the supremum over $Q \in \calM(\calH)$, and an expectation over samples $S \sim \calD^m$ we have that for any $P \in \calM(\calH)$
\begin{align} \label{eq:proof_before_Hoeffding-1}
\E_{S \sim \calD^m} \sup_{Q} \brc*{ \exp\br*{\Expct hQ\brs*{f_{S}(h)}-\gamma_{\calF_S}(Q, P)}} & \le \E_{S \sim \calD^m} \sup_{Q} \brc*{ \Expct hP \exp\br*{f_{S}(h) }} \\
 & =\E_{S \sim \calD^m}  \Expct hP  \exp\br*{f_{S}(h)} \nn \\
 & = \Expct hP\E_{S \sim \calD^m}\exp\br*{f_{S}(h)}, \nn
\end{align}
where the last equality is obtained by the prior's independence from the sample, and from Fubini's theorem.
We recall that by Hoeffding's inequality, for any $h \in \calH$, 
\begin{align}
\mathbb{P}_{S \sim \calD^m}\left(\Delta_S(h)>\uc\right) & \le e^{-2m\uc^{2}}, \nn
\end{align}
which, by Lemma 5 of \cite{McAllester_03}, this imply.
\begin{equation} \label{eq:bound_by_m}
   \bbE_{S \sim \calD^m} \exp\br*{f_{S}(h)}  = \bbE_{S \sim \calD^m} \exp\br*{2(m-1)\Delta_S^{2}(h)} \le m.
\end{equation}

Inequalities \eqref{eq:proof_before_Hoeffding-1} and \eqref{eq:bound_by_m} imply 
\begin{align}
\E_{S \sim \calD^m}  \sup_{Q} \brc*{\exp\br*{\Expct hQ\brs*{f_{S}(h)}-\gamma_{\calF_S}(Q, P)}}  & \le m. \nn
\end{align}
Therefore, by Markov's inequality, for any $t >0$ we have
\begin{align}
\mathbb{P}_{S \sim \calD^m}\left( \sup_{Q} \brc*{\exp\br*{\Expct hQ\brs*{f_{S}(h)} - \gamma_{\calF_S}(Q, P)}}  \geq t\right) & \le\frac{m}{t}. \nn
\end{align}
Or, equivalently,
\begin{align}
\mathbb{P}_{S \sim \calD^m}\left( \ln \br*{\sup_{Q} \brc*{\exp\br*{\Expct hQ\brs*{f_{S}(h)} - \gamma_{\calF_S}(Q, P)}}} \ge \ln(t) \right) & \le\frac{m}{t}. \nn
\end{align}

By Lem.~\ref{lem:sup_f}, the $\ln(\cdot)$ and $\sup(\cdot)$ operations are interchangeable, and therefore
% \begin{align*}
%     \ln\br*{\sup_{a \in A}{a}} = \sup_{a \in A}{\ln(a)} \text{ for } A = \brc*{\exp \br*{ \E_{h \sim Q}[f_S(h)] - \gamma_{\calF_S}(Q, P) } : Q \in \calM(\calH)}
% \end{align*}
% and therefore

\begin{align}
\mathbb{P}_{S \sim \calD^m}\left( \sup_{Q} \brc*{\Expct hQ\brs*{f_{S}(h)} - \gamma_{\calF_S}(Q,P)} \ge \ln(t) \right) & \le\frac{m}{t}. \nn
\end{align}

Let $\delta \in (0,1)$, we set $t= \frac{m}{\delta}$, and plug in $f_{S}(h) \defeq2(m-1) \Delta_{S}^{2}(h)$ to get
\begin{align}
\mathbb{P}_{S \sim \calD^m}\left(  \sup_{Q} \brc*{\Expct hQ\left(2(m-1)\Delta_{S}^{2}(h)\right) - \gamma_{\calF_S}(Q, P)} \ge \ln (\nf{m}{\delta}) \right) & \le \delta. \nn
\end{align}
Therefore, the complementary event satisfies
\begin{align}
\mathbb{P}_{S \sim \calD^m}\left(  \sup_{Q} \brc*{\Expct hQ\left(2(m-1)\Delta_{S}^{2}(h)\right) - \gamma_{\calF_S}(Q, P)} < \ln (\nf{m}{\delta}) \right) & \ge 1 -\delta. \nn
\end{align}
I.e.,~for any $P \in \calM(\calH)$, with a probability of at least
$1-\delta$ over the samples $S \sim \calD^m$, the following inequality holds for all $Q \in \calM(\calH)$
\begin{align}
\Expct hQ\left(\Delta_S^{2}(h)\right) & < \frac{\gamma_{\calF_S}(Q, P) + \ln(\nf{m}{\delta})}{2(m-1)}. \nn
\end{align}

Jensen's inequality implies that
\begin{align}
 \left(\Expct hQ\Delta_S(h)\right)^{2}\le\Expct hQ\left(\Delta_S^{2}(h)\right) & \le \frac{\gamma_\calF(Q, P) + \ln(\nf{m}{\delta})}{2(m-1)}. \nn
\end{align}
The proof is concluded by taking the square root of both sides.
\end{proof}

%%%%%%%%%%%%%%%%%%%%%%%%%%%%%%%%%%%%%%%%%%%%%%%%%%%%%%%%%%%%
\subsection{Proof of the Seeger's Type IPM-PB Bound (Prop. \ref{prop:Seeger_IPM_PB})} \label{sect:proof_thm:PACBayesBound_IPM_Seeger}
%%%%%%%%%%%%%%%%%%%%%%%%%%%%%%%%%%%%%%%%%%%%%%%%%%%%%%%%%%%%

\begin{proof}
As in the proof of Prop.~\ref{prop:PACBayesBound_IPM}, we follow a similar structure as the classical derivation \citep{McAllester_03, Maurer_04}, except replacing the change-of-measure inequality. 

For any sample $S \in \calZ^{m}$, consider the function on $\calH$
\begin{equation}
    f_{S}(h)\defeq m  \cdot \kld{\Lhat_{S}(h)}{L_D(h)} . \nn
\end{equation}

Note that $f_S(\cdot)$ is almost surely well-defined, since if $L_D(h)=0$, then $\Lhat(h) \textrel{a.s.}{=} 0$. 

Since we assume $f_S\in \calF_{S}$, Definition \ref{def:IPM} implies that for any pair of probability measures $P,Q \in \calM(\calH)$
\begin{align}
\Expct hQ\brs*{f_{S}(h)} - \Expct hP \brs*{f_{S}(h)}  & \le \gamma_{\calF_{S}}(Q, P). \nn
\end{align}
Therefore, by the monotonicity of $\exp(\cdot)$ we have 
\begin{align}
\exp\br*{\Expct hQ\brs*{f_{S}(h)}- \gamma_{\calF_{S}}(Q, P)} & \le \exp\br*{\Expct hP\brs*{f_{S}(h)}} \nn \\
 & \le\Expct hP\left[\exp\br*{f_{S}(h)}\right]. \nn
\end{align}
where the last inequality is by the convexity of $\exp(\cdot)$ and by Jensen's inequality.

Taking the supremum over $Q \in \calM(\calH)$, and an expectation over samples $S \sim \calD^m$ we have for any $P \in \calM(\calH)$
\begin{align} \label{eq:proof_sgr}
\E_{S \sim \calD^m} \sup_Q \brc*{\exp\br*{\Expct hQ\brs*{f_{S}(h)}-\gamma_{\calF_S}(Q,P)}} & \le \E_{S \sim \calD^m} \sup_Q \brc*{\Expct hP \exp\br*{f_{S}(h)}} \\
 & =\E_{S \sim \calD^m}  \Expct hP  \exp\br*{f_{S}(h)} \nn \\
 & = \Expct hP \E_{S \sim \calD^m}\exp\br*{f_{S}(h)} ,\nn
\end{align}
where the last equality is obtained by the prior's independence from the sample, and from Fubini's theorem.
Using \citet{Maurer_04}, Thm. 1 we have
\begin{align} \label{eq:bound_by_m_sgr}
    \E_{S \sim \calD^m} \exp\br*{f_{S}(h)} &= \E_{S \sim \calD^m}\exp\br*{m  \cdot \kld{\Lhat_{S}(h)}{L_D(h)}}\\
    &\leq 2 \sqrt{m} . \nn
\end{align}
 
Inequalities \eqref{eq:proof_sgr} and \eqref{eq:bound_by_m_sgr} imply 
\begin{align}
\E_{S \sim \calD^m} \sup_Q \brc*{ \exp\br*{\Expct hQ\brs*{f_{S}(h)}-\gamma_{\calF_S}(Q, P)}}  & \le 2\sqrt{m}. \nn
\end{align}
Therefore, by Markov's inequality, for any $t >0$ we have
\begin{align}
\mathbb{P}_{S \sim \calD^m} \left( \sup_Q  \brc*{ \exp\br*{\Expct hQ\brs*{f_{S}(h)} - \gamma_{\calF_S}(Q, P)} } \ge t \right) & \le\frac{2\sqrt{m}}{t}, \nn
\end{align}
or, equivalently,
\begin{align}
\mathbb{P}_{S \sim \calD^m}\left( \ln \br*{\sup_Q  \brc*{ \exp\br*{\Expct hQ\brs*{f_{S}(h)} - \gamma_{\calF_S}(Q, P)} } }\ge \ln (t)\right) & \le\frac{2\sqrt{m}}{t}. \nn
\end{align}

By Lem.~\ref{lem:sup_f}, the $\ln(\cdot)$ and $\sup(\cdot)$ operations are interchangeable, and therefore
% \begin{align*}
%     \ln\br*{\sup_{a \in A}{a}} = \sup_{a \in A}{\ln(a)} \text{ for } A = \brc*{\exp \br*{ \E_{h \sim Q}[f_S(h)] - \gamma_{\calF_S}(Q, P) } : Q \in \calM(\calH)}
% \end{align*}
% and therefore

Let $\delta \in (0,1)$, we set $t= \frac{2\sqrt{m}}{\delta}$, and plug in $f_{S}(h) \defeq m  \cdot \kld{\Lhat_{S}(h)}{L_D(h)}$ to get
\begin{align}
\mathbb{P}_{S \sim \calD^m}\left( \sup_Q \brc*{ m \Expct hQ  \kld{\Lhat_{S}(h)}{L_D(h)} - \gamma_{\calF_S}(Q, P)}\geq \ln (\nf{2\sqrt{m}}{\delta})\right) & \le \delta. \nn
\end{align}
Therefore, the complementary event satisfies
\begin{align}
\mathbb{P}_{S \sim \calD^m}\left( \sup_Q \brc*{ m \Expct hQ  \kld{\Lhat_{S}(h)}{L_D(h)} - \gamma_{\calF_S}(Q, P)} < \ln (\nf{2\sqrt{m}}{\delta})\right) & \geq 1 - \delta. \nn
\end{align}
I.e., for any $P \in \calM(\calH)$, with a probability of at least
$1-\delta$ over the samples $S \sim \calD^m$, the following inequality holds for all $Q \in \calM(\calH)$
\begin{align}
    \Expct hQ  \kld{\Lhat_{S}(h)}{L_D(h)} & < \frac{\gamma_{\calF_S}(Q, P) + \ln(\nf{{2\sqrt{m}}}{\delta})}{m}. \nn
\end{align}
By the convexity of the function $\kld{p}{q}$ in the pair of parameters $(p,q)$ and by Jensen's inequality, we have 
\begin{align}
  \kld{\Expct hQ \Lhat_{S}(h)}{\Expct hQ  L_D(h)} \leq    \Expct hQ  \kld{\Lhat_{S}(h)}{L_D(h)} . \nn  
\end{align}
Therefore we finally get that for any $P \in \calM(\calH)$, w.p.~of at least $1-\delta$ the following holds for all $Q \in \calM(\calH)$
\begin{align}
   \kld{\Lhat_{S}(Q)}{L_D(Q)} & < \frac{\gamma_{\calF_S}(Q, P) + \ln(\nf{{2\sqrt{m}}}{\delta})}{m}. \nn
\end{align}
\end{proof}

 %%%%%%%%%%%%%%%%%%%%%%%%%%%%%%%%%%%%%%%%%%%%%%%%%%%%%%%%%%%%
 \subsection{Proof of the Total-Variation PAC-Bayes Bound (Thm. \ref{thm:TVPB})}
\label{sect:proof_thm_TVPB}
 %%%%%%%%%%%%%%%%%%%%%%%%%%%%%%%%%%%%%%%%%%%%%%%%%%%%%%%%%%%%
 
 \begin{proof}
Let $\delta > 0$.
For any fixed sample $S \in \calZ^m$, we define $f_S(h) \defeq 2(m-1) \Delta^2_S(h)$.
Define $M_S \defeq \sup_{h \in \calH} \Delta^2_S(h)$.
Therefore, $0 \leq f_S(h)  \leq  2(m-1)  M_S$, i.e., $f_S(h) \in \calFinfty_{2(m-1) M_S}$.
% Assume that $\abs*{f_S(h)} \leq 2 (m-1) M_S$, for some $M_S > 0$ (we can always take $M_S=1$,
% since $0 \leq \Delta_S(h) \leq 1$). Therefore, $f_S(h) \in \calF_{\calFinfty_{2(m-1) M_S}}$.
This fact, together with the general IPM PB bound (Prop. \ref{prop:PACBayesBound_IPM}) and Eq. \eqref{eq:TVeqGamma} (equivalence of IPM to TV under the family of bounded functions)  imply that for any $\delta  \in (0,1)$
\begin{align}  
\Prob \br*{\Delta_S(Q) \le\sqrt{\frac{2 (m-1) M_S \DTV(Q,P)}{2(m-1)} + \frac{ \ln(\nf{2 m}{\delta})}{2(m-1)}}} \geq 1 - \nf{\delta}{2}, \nn
\end{align}
or equivalently,
\begin{align} \label{eq:DSQ_bnd_MS}
\Prob \br*{\Delta_S(Q) \le \sqrt{M_S \DTV(Q,P) + \frac{ \ln(\nf{2 m}{\delta})}{2(m-1)}}} \geq 1 - \nf{\delta}{2} .
\end{align}
 According to the UC assumption we have
 \begin{equation} 
    \Prob \br*{\Delta_S(h) \leq \uc\br{m, \nf{\delta}{2}}, \forall h \in \calH} \geq 1 - \nf{\delta}{2}, \nn
 \end{equation}
 and therefore, using the fact that $0 \leq \Delta_S(h) \leq 1$ we also have
  \begin{equation} 
    \Prob \br*{\Delta^2_S(h) \leq \uc^2\br{m, \nf{\delta}{2}}, \forall h \in \calH} \geq 1 - \nf{\delta}{2}, \nn
 \end{equation}
which implies that the bounds also holds for the supremum $M_S$
   \begin{equation}  \label{eq:MS_bnd}
    \Prob \br*{M_S \leq \uc^2\br{m, \nf{\delta}{2} }} \geq 1 - \nf{\delta}{2}.
 \end{equation}
 To conclude the proof, we use a union bound argument and Equations \eqref{eq:DSQ_bnd_MS} and \eqref{eq:MS_bnd}.

\end{proof}

  %%%%%%%%%%%%%%%%%%%%%%%%%%%%%%%%%%%%%%%%%%%%%%%
\subsection{Proof of the Template Wasserstein-PB Bound (Thm. \ref{thm:WPB})} \label{sect:proof_thm:WPB}
 %%%%%%%%%%%%%%%%%%%%%%%%%%%%%%%%%%
 \begin{proof}
Let $\delta > 0$.
% For any $S \in \calZ^m$, a Lipschitz constant of $\Delta_{S}^{2}(\cdot)$ satisfies
%  \begin{equation}
%      K_S \leq \sup_{(h,h') \in \calH^2} \frac{\abs*{\Delta_S^2(h) - \Delta_S^2(h')}}{\rho(h,h')}
%  \end{equation}
Let $K_S$ be some Lipschitz constant of $\Delta_{S}^{2}(\cdot)$.
Define $f_S(h) \defeq 2 (m-1) K_S$. Notice that $\Delta_{S}^{2}(h)$ is $2 (m-1) K_S$-Lipschitz, i.e., $f_S(h) \in \calFLip_{2 (m-1) K_S}$.
Using Proposition \ref{prop:PACBayesBound_IPM} we have 
\begin{align}
\Prob \br*{\Delta(Q) \le \sqrt{\frac{\gamma_{\calFLip_{2(m-1) K_S}} (P,Q)  +\ln(\nf{2 m}{\delta})}{2(m-1)}}} \geq 1 -\nf{\delta}{2}.\nn
\end{align}
By equation \eqref{eq:Duality_Short} (the Kantorovich-Rubinstein duality) the inequality can rewritten as 
\begin{align} \label{eq:wpb_proof_ks1}
\Prob \br*{\Delta(Q) \le \sqrt{\frac{2 (m-1) K_S W_1(Q,P) +\ln(\nf{2 m}{\delta})}{2(m-1)}}} \geq 1 -\nf{\delta}{2}.
\end{align}

By assumption, w.p.\ at least  $1 - \nf{\delta}{2}$,  $\Delta_{S}^{2}(h)$ is $K$-Lipschitz with $K=K\br{m, \nf{\delta}{2}}$.
Using a union bound argument with this event and the event of \eqref{eq:wpb_proof_ks1} concludes the proof.
\end{proof}
 %%%%%%%%%%%%%%%%%%%%%%%%%%%%%%%%%%%%%%%%%%%%%%%%%%%%%%%%%%%%
\subsection{Proof of the Wasserstein-PB Bound for Finite Classes (Thm.\ \ref{thm:WPB_finite})} \label{sect:proof_KS_finite_H}
%%%%%%%%%%%%%%%%%%%%%%%%%%%%%%%%%%%%%%%%%%%%%%%%%%%%%%%%%%%%

We first prove the following lemma.
\begin{lem} \label{lem:KS_finite_H}
 Let  $\calH$ be a finite hypothesis class.
 Assume that for any fixed $z \in \calZ$, $\ell(h,z)$ is a $\LossLip$-Lipschitz function in $h \in \calH$ w.r.t~the metric $\rho$.
Then for any $\delta \in (0,1)$, we have
 \begin{equation}
     \Prob \br*{\Ktild_S \leq  \frac{8}{m} \LossLip\log(\nf{2 \abs{\calH}}{\delta})} \geq 1 - \delta, \nn
 \end{equation}
  where for any fixed $S \in \calZ^m$, $\Ktild_S$ is the sharp Lipschitz constant of $\Delta^2_S(\cdot)$, i.e.
\begin{equation}
   \Ktild_S \textrel{def}{=} \sup_{h,h' \in \calH: h \neq h'} \frac{\abs*{\Delta_S^2(h) - \Delta_S^2(h')}}{\rho(h,h')} .\nn
\end{equation}
\end{lem}

\begin{proof}[Proof of lemma \ref{lem:KS_finite_H}]

Let $\delta > 0$.

Note that for any $h,h' \in \calH$ and $S \in \calZ^m$, we have
\begin{align} \label{eq:bound_ratio}
    &\frac{\abs*{\Delta_S(h) - \Delta_S(h')} }{\rho(h,h')}  \\
    &=   \frac{\abs*{ \brs*{\E_{z\sim\calD}\ell(h,z))- \frac{1}{m}\sum_{i=1}^{m}
 \ell(h,z_{i}) } 
    -  \brs*{\E_{z\sim\calD}\ell(h',z)) - \frac{1}{m}\sum_{i=1}^{m}\ell(h',z_{i}) }} }{\rho(h,h')} \nn \\
     &=   \abs*{ \frac{ \frac{1}{m}\sum_{i=1}^{m} \brs*{\ell(h',z_{i}) - \ell(h,z_{i})} - \E_{z\sim\calD}  \brs*{\ell(h',z) - \ell(h,z)}  }{\rho(h,h')}} .\nn
\end{align}
 
By assumption, for any fixed $z \in \calZ$, $\ell(h,z)$ is a $\LossLip$-Lipschitz function in $h$ 
w.r.t.~the metric $\rho$.
I.e., we have that $\abs*{\ell(h,z) - \ell(h',z)} \leq \LossLip \rho(h,h'), \forall h,h' \in \calH, z \in \calZ$.
Hence, for any pair $(h,h') \in \calH^2$, the random sequence $\brc*{\abs*{\ell(h,z_i) - \ell(h',z_i)}}_{i=1}^m$ is i.i.d.\ and bounded by $\LossLip \rho(h,h')$.

By Hoeffding's theorem, it holds with probability of at least $1 - \frac{\delta}{ 2 \abs{\calH}^2}$ that
    \begin{align} \label{eq:hoeff_bnd}
    \abs*{ \frac{1}{m}\sum_{i=1}^{m} \brs*{\ell(h',z_{i}) - \ell(h,z_{i})}  - \E_{z\sim\calD} \brs*{\ell(h',z) - \ell(h,z)} } \leq  \rho(h,h') \LossLip\sqrt{\frac{2 \ln (\frac{4 \abs{\calH}^2}{\delta})}{m} }  .
    \end{align}
    
By using a union bound over claim \eqref{eq:hoeff_bnd} for all pairs of hypotheses $(h,h') \in \calH^2$, we get that w.p.\ of at least $1-\nf{\delta}{2}$, we have for all pairs $(h,h') \in \calH^2$ \textbf{simultaneously} that,
\begin{align}   \label{eq:hoeff_bnd2}
       \abs*{ \frac{1}{m}\sum_{i=1}^{m} \brs*{\ell(h',z_{i}) - \ell(h,z_{i})}  - \E_{z\sim\calD} \brs*{\ell(h',z) - \ell(h,z)} } \leq  \rho(h,h') \LossLip\sqrt{\frac{2 \ln (\frac{4 \abs{\calH}^2}{\delta})}{m}}  .
\end{align}

It is well-known (e.g.\ \citet{SSS}, Cor.\  2.3) that for a finite hypothesis class and any $\nf{\delta}{2} > 0$ , 
\begin{equation} \label{eq:standard_finite_class_bnd}
    \Prob \br*{\sup_{h \in \calH} \abs*{\Delta_S(h)} \leq \sqrt{\frac{\ln(\nf{4 \abs{\calH}}{\delta})}{2m}}} \geq 1 - \nf{\delta}{2} .
\end{equation}
 
Notice that 
 \begin{align}
      \Ktild_S &\defeq\sup_{h,h' \in \calH: h \neq h'} \frac{\abs*{\Delta_S^2(h) - \Delta_S^2(h')}}{\rho(h,h')} \nn \\
    &= \sup_{h,h' \in \calH: h \neq h'} \frac{\abs*{\Delta_S(h) - \Delta_S(h')} \abs*{\Delta_S(h) + \Delta_S(h')}}{\rho(h,h')} \nn \\
        &\leq \sup_{h,h' \in \calH: h \neq h'} \frac{\abs*{\Delta_S(h) - \Delta_S(h')}}{\rho(h,h')} 2 \sup_{h'' \in \calH} \abs*{\Delta_S(h'')} \nn   \\
      &\textrel{\eqref{eq:bound_ratio}}{=} 4 \sup_{h,h' \in \calH: h \neq h'} \abs*{\frac{ \frac{1}{m}\sum_{i=1}^{m} \brs*{\ell(h',z_{i}) - \ell(h,z_{i})} - \E_{z\sim\calD}  \brs*{\ell(h',z) - \ell(h,z)}  }{\rho(h,h')}}  \sup_{h'' \in \calH} \abs*{\Delta_S(h'')} . \nn
 \end{align}

The proof of the lemma is concluded by using a union bound argument with claims \eqref{eq:hoeff_bnd2}
and \eqref{eq:standard_finite_class_bnd}.
We get that w.p.\ of at least $1-\delta$ we have
\begin{align}
   \Ktild_S &\leq 4 \sup_{h,h' \in \calH: h \neq h'} \brc*{ \frac{ \rho(h,h') \LossLip\sqrt{\frac{2 \ln (\nf{4 \abs{ \calH}^2}{\delta})}{m}}}{\rho(h,h')} } \sqrt{\frac{\ln(\nf{4 \abs{\calH}}{\delta})}{2m}} \nn\\
    &= \frac{4}{m} \LossLip\sqrt{\ln (\nf{4 \abs{ \calH}^2}{\delta}) \ln(\nf{4 \abs{\calH}}{\delta})}  \nn\\
    & \leq \frac{4}{m} \LossLip\ln(\nf{4 \abs{\calH}^2}{\delta}) \nn \\ 
    & \leq \frac{8}{m} \LossLip \ln(\nf{2 \abs{\calH}}{\delta}).  \nn
\end{align}
 
%  To summarize, for finite $\calH$ and $\LossLip$-Lipschitz loss function, we get that for any $\delta > 0$, 
%  \begin{equation}
%      \Prob \br*{\Ktild_S \leq  \frac{8}{m} \LossLip\log(\nf{2 \abs{\calH}}{\delta})} \geq 1 - \delta
%  \end{equation}
 \end{proof}

\begin{proof}[Proof of Theorem \ref{thm:WPB_finite}]
The proof follows directly from Lemma \ref{lem:KS_finite_H} and Theorem \ref{thm:WPB}.
\end{proof}

  %%%%%%%%%%%%%%%%%%%%%%%%%%%%%%%%%%%%%%%%%%%%%%%
\subsection{Proof of the Wasserstein-PB Bound for Differentiable Loss UC Classes (Thm. \ref{thm:WPB_uc_grad})} \label{sect:proof_thm:WPB_uc_grad}
 %%%%%%%%%%%%%%%%%%%%%%%%%%%%%%%%%%

We first prove the following lemma. 

\begin{lem} \label{lem:KS_grad_UC}
Let $(\calH, \rho)$ be a metric space such that $\calH \subset \R^d$ is a closed convex set, and $\rho$ is the $L_2$ distance.  
% $\rho(h,h') \defeq \norm{h - h'}_2$.
% Assume that for any $z \in \calZ$ and $h \in \interior(\calH)$, $\ell(h,z)$ is continuously differentiable w.r.t~$h$.
{Assume that the loss function $\ell(h,z)$ is differentiable w.r.t.~h on $\interior(\calH) \times \calZ$ and continuous w.r.t.~$h$ on $\calH \times \calZ$.
}
Assume the learning problem has a UC bound $\uc$ (Def. \ref{def:UC}), and a UC bound of the loss gradient, $\ucg$ (Def. \ref{def:uc_grad}), 
then 
 \begin{equation}
     \Prob \br*{\Ktild_S \leq  2 \cdot  \uc\br{m, \nf{\delta}{2}} \cdot  \ucg\br{m, \nf{\delta}{2}}  } \geq 1 - \delta, \nn
 \end{equation}
 where for any fixed $S \in \calZ^m$, $\Ktild_S$ is the sharp Lipschitz constant of $\Delta^2_S(\cdot)$, i.e.
\begin{equation}
   \Ktild_S \textrel{def}{=} \sup_{h,h' \in \calH: h \neq h'} \frac{\abs*{\Delta_S^2(h) - \Delta_S^2(h')}}{\rho(h,h')}. \nn
\end{equation}
\end{lem}

% $\uc^{(k)}_{m, \delta', \calH}$, in any of the learning problems $(\calD,  \calH, \ell_k(h,z)), k=1,\dots,d$, where $\ell_k(h,z) = \frac{\partial}{\partial h_k} \ell(h,z)$ is the partial derivative of

% $\ell(h,z)$ w.r.t.\ the $h_k$ coordinate, at the point $(h,z)$. 

% Assume $\calH \subset \R^d$ is a convex closed set in $d$-dimension space, and assume Euclidean metric $\rho(h,h') \defeq \norm{h - h'}_2, \forall h,h' \in \calH$. 

\begin{proof}[Proof of Lemma \ref{lem:KS_grad_UC}]

% By the mean value theorem and the convexity of $\calH$, we have
% For any $z \in \calZ$, by the mean value theorem we have that for any pair $(h, h')\in \calH$, there exists a point $w\brs{h,h',z} \in \calH$ (more specifically, on the line segment between $h$ and $h'$, that it is entirely inside $\calH$ since it is a convex set) s.t.\ 
{{\placeholdera}
Using the fact that loss function $\ell(h,z)$ is differentiable w.r.t.~h on $\interior(\calH) \times \calZ$ and continuous w.r.t.~$h$ on $\calH \times \calZ$, and
by the mean value theorem and the convexity of $\calH$, we have
}
\begin{equation} \label{eq:mean_value_thm}
    \forall z \in \calZ, (h, h')\in \calH^2, \exists w_{z, h,h'} \in \calH, \text{s.t.~}  \ell(h,z) - \ell(h',z)  =  \innorm{h - h', \nabla_h\ell \br*{w_{z, h,h'}, z} },
\end{equation}
where $\nabla_h\ell\br{w,z}$ denotes the gradient of $\ell(\cdot, \cdot)$ w.r.t the $h$ variable, at the point $(w,z)$.
% and $\innorm{\cdot, \cdot}$ denotes dot product between vectors in $\R^d$.

Notice that 
 \begin{align}   \label{eq:ks_grad_eq1}
      \Ktild_S &\textrel{def}{=} \sup_{h,h' \in \calH: h \neq h'} \frac{\abs*{\Delta_S^2(h) - \Delta_S^2(h')}}{\rho(h,h')} \\
    &= \sup_{h,h' \in \calH: h \neq h'} \frac{\abs*{\Delta_S(h) - \Delta_S(h')} \abs*{\Delta_S(h) + \Delta_S(h')}}{\rho(h,h')} \nn  \\
        &\leq \sup_{h,h' \in \calH: h \neq h'} \frac{\abs*{\Delta_S(h) - \Delta_S(h')}}{\rho(h,h')} 2 \sup_{h'' \in \calH} \abs*{\Delta_S(h'')} . \nn  
\end{align}
We have for any $(h,h') \in \calH, h \neq h$ that
\begin{align}  \label{eq:ks_grad_eq2}
      & \frac{\abs*{\Delta_S(h) - \Delta_S(h')}}{\rho(h,h')} \\
      =& \abs*{\frac{ \frac{1}{m}\sum_{i=1}^{m} \brs*{\ell(h',z_{i}) - \ell(h,z_{i})} - \E_{z\sim\calD}  \brs*{\ell(h',z) - \ell(h,z)}  }{\rho(h,h')}}  \nn \\
      \textrel{(i)}{=} & \abs*{\frac{ \frac{1}{m}\sum_{i=1}^{m} 
      \innorm{h - h', \nabla
    _h\ell \br*{w_{z_i, h,h'}, z_i} }   - \E_{z\sim\calD}    \innorm{h - h', \nabla_h\ell\br*{w_{z, h,h'}, z} }  }{\norm{h - h'}_2}}  \nn \\
    \textrel{(ii)}{=} &\abs*{\frac{ \innorm{h - h', \frac{1}{m}\sum_{i=1}^{m}
      \nabla_h\ell(w_{z_i, h,h'},z_i) - \E_{z\sim\calD}  \nabla_h\ell \br*{w_{z,h,h'}, z} }}{\norm{h - h'}_2}} \nn \\
    \textrel{(iii)}{\leq}  & \frac{ \norm{h - h'}_2  \norm{ \frac{1}{m}\sum_{i=1}^{m}
      \nabla_h\ell\br*{w_{z_i,h,h},z_i} - \E_{z\sim\calD}  \nabla_h\ell\br*{w_{z,h,h},z} }_2}{\norm{h - h'}_2} \nn \\
       =&  \norm{ \frac{1}{m}\sum_{i=1}^{m}
      \nabla_h\ell\br*{w_{z_i,h,h'}, z_i} - \E_{z\sim\calD}  \nabla_h\ell\br*{w_{z,h,h'},z} }_2 , \nn
 \end{align}
 where \textit{(i)} is by the mean value theorem (Eq.~\ref{eq:mean_value_thm}), \textit{(ii)} is by the linearity of the sum, expectation, and the inner product, and  \textit{(iii)} is by the Cauchy–Schwarz inequality.
 
 By the UC assumptions, we have
  \begin{align} \label{eq:ks_grad_proof1}
     \Prob \br*{\forall h \in \calH, \abs*{\Delta_S(h) } \leq \uc\br{m, \nf{\delta}{2}}} \geq 1 - \nf{\delta}{2},
 \end{align}
and
\begin{align} \label{eq:ks_grad_proof2}
   \Prob \br*{\forall h \in \interior(\calH), \norm{ \frac{1}{m}\sum_{i=1}^{m}
      \nabla_h\ell(h,z_i) - \E_{z\sim\calD}  \nabla_h\ell(h,z) }_2 \leq \ucg\br{m, \nf{\delta}{2}}} \geq 1 - \nf{\delta}{2}.
\end{align}
To conclude the proof, we use a union bound argument with \eqref{eq:ks_grad_proof1} and \eqref{eq:ks_grad_proof2}, 
and use inequalities \eqref{eq:ks_grad_eq1} and \eqref{eq:ks_grad_eq2} to finally get
 \begin{equation}
     \Prob \br*{ \Ktild_S \leq  2 \cdot  \uc\br{m, \nf{\delta}{2}} \cdot  \ucg\br{m, \nf{\delta}{2}}  } \geq 1 - \delta. \nn
 \end{equation}
 \end{proof}
 
\begin{proof}[Proof of Theorem \ref{thm:WPB_uc_grad}]
The proof follows from Lemma \ref{lem:KS_grad_UC} with $\nf{\delta}{2}$, Theorem \ref{thm:WPB} with $\nf{\delta}{2}$, and using the union bound.
\end{proof}
%#####################################################

 \subsection{Proof of the Wasserstein-PB Bound for Linear Regression (Cor.~\ref{cor:wpb_regression})} \label{sect:proof_cor_wpb_regression}
  %%%%%%%%%%%%%%%%%%%%%%%%%%%%%%%%%%%%%%%%%%%%%%%%%%%%%%%%%%%%

 \begin{proof}[Proof of Corollary \ref{cor:wpb_regression}]
To meet the requirements of Theorem \ref{thm:WPB_uc_grad}, we will prove a uniform convergence bound for the generalization gap  (Def.~\ref{def:UC}), and for the loss gradient (Def.~\ref{def:uc_grad}).

% Throughout the proof, we will denote $\ytild = \frac{1}{\sigma} y$. Notice that  $\ytild \in [-1,1]$.

The generalization gap function for any $h \in \calH$ can be written as
\begin{align} 
 \label{eq:reg_uc_proof0}
     &\Delta_S(h)
    =  \E  \ell(h,z)
    - \frac{1}{m}\sum_{i=1}^m  \ell(h,z_i) \\
    &=  \E \frac{1}{4} (h^\top x - y)^2
    - \frac{1}{m}\sum_{i=1}^m \frac{1}{4} (h^\top x_i - y_i)^2 \nn \\
    &= \frac{1}{4}  \br*{\E y^2  -\frac{1}{m}\sum_{i=1}^m y_i^2 } - \frac{1}{2}  \br*{\E y x^\top h - \frac{1}{m}\sum_{i=1}^m y_i x_i^\top h} + \frac{1}{4}  h^\top \br*{\E x x^\top - \frac{1}{m} \sum_{i=1}^m x_i x_i^\top  }h.  \nn
\end{align}

% {\color{red} 
% \begin{align} 
%  \label{eq:reg_uc_proof0}
%      &\Delta_S(h)
%     =  \E  \ell(h,z)
%     - \frac{1}{m}\sum_{i=1}^m  \ell(h,z_i) \\
%     &=  \E  \min \brc*{\frac{1}{4} (h^\top x - y)^2, 1}
%     - \frac{1}{m}\sum_{i=1}^m \min \brc*{\frac{1}{4} (h^\top x_i - y_i)^2, 1} \nn
% \end{align}
% }
Let $\delta > 0$.
We will now bound in high probability each of the three term above.

\textbf{First term.} The variables $y^2_1,\dots,y^2_m$ are independent random variables in the range $[0,1]$, therefore by Hoeffding's inequality
\begin{equation} 
    \Prob \br*{\abs*{\E y^2  -\frac{1}{m}\sum_{i=1}^m y_i^2} \leq \sqrt{\frac{\ln(\nf{6}{\delta})}{2 m}}} \geq 1 - \nf{\delta}{3}.
\end{equation}
I.e.,
\begin{equation} \label{eq:reg_uc_proof1}
    \Prob \br*{\frac{1}{4} \abs*{\E y^2  -\frac{1}{m}\sum_{i=1}^m y_i^2} \leq \sqrt{\frac{\ln(\nf{6}{\delta})}{32 m }}} \geq 1 - \nf{\delta}{3}.
\end{equation}

\textbf{Second term.} Note that $y x$ is $r$-sub-Gaussian random vector in $\R^d$, since, for any $s$ in the unit-sphere $\Sd_{1}$, $s^\top y x $ is $r$-sub-Gaussian (since it is a.s.~bounded in $[-r,r]$).
Therefore $\brc*{y_i x_i}_{i=1}^{m}$ are independent $r$-sub-Gaussian random vectors. 
Using Thm.\ 1 of \citet{hsu2012tail}, we have that for any $t >0$,
\begin{equation}
    \Prob \br*{\norm{\E y x  -\frac{1}{m}\sum_{i=1}^m y_i x_i }_2^2 > \frac{r^2}{m} \br*{d + 2 d \sqrt{t} +2 t}} \leq \exp(-t). \nn
\end{equation}
Therefore, we have
\begin{equation} \label{eq:reg_uc_proof_xy}
    \Prob \br*{\norm{\E y x  -\frac{1}{m}\sum_{i=1}^m y_i x_i}_2^2 < \frac{r^2}{m}\br*{ d + 2 d \sqrt{\ln(\nf{3}{\delta})} +2 \ln(\nf{3}{\delta})}} \geq  1 - \nf{\delta}{3}.
\end{equation}
Then, by the Cauchy–Schwarz inequality we have that w.p.~of at least $ 1 - \nf{\delta}{3}$, for all $h \in \Bir$
\begin{align}
 \label{eq:reg_uc_proof2}
 \frac{1}{2} \abs*{h^\top \br*{\E y x^\top  -\frac{1}{m}\sum_{i=1}^m y_i x_i^\top } }
   & \leq \frac{1}{2} \norm{h}_2 \norm{\E y x  -\frac{1}{m}\sum_{i=1}^m y_i x_i}_2  \\
   &< \frac{1}{2 \sqrt{m}} \sqrt{d + 2 d \sqrt{\ln(\nf{3}{\delta})} + 2 \ln(\nf{3}{\delta})}.\nn
\end{align}

\textbf{Third term.} By Theorem 6.5 of \citet{wainwright2019high} (constants from Thm.\ of \citet{bastani2019meta} Lem.\ 22), we have that w.p.~of at least $ 1 - \nf{\delta}{3}$
\begin{equation} \label{eq:reg_uc_proof_xx}
  \frac{1}{4} \normop{\frac{1}{m} 
\sum_{i=1}^{m} x_i x_i^{\top} -\E\br*{x x^\top}}
    \leq 8 r^2
    \max \left\{
    \sqrt{\frac{5d + 2\ln\left(\frac{6}{\delta}\right)}
    {m}},
    \frac{5d + 2\ln\left(\frac{6}{\delta}\right)}{m}
    \right\},
\end{equation}
where $\normop{\cdot}$ is the $\ell_2$ operator-norm, that can be defined by $ \normop{A} \defeq \sup_{u,v \in \Sd_1} \abs*{u^\top A v}, \forall A \in \R^{d \times d}$, where $\Sd_1$ is the the unit sphere in $\R^d$.
Therefore, we conclude that w.p.~of at least $ 1 - \nf{\delta}{3}$
\begin{align} \label{eq:reg_uc_proof3}
    \forall h \in \Bir, \frac{1}{4} \abs*{h^\top \br*{\frac{1}{m} 
\sum_{i=1}^{m} x_i x_i^{\top} -\E\br*{x x^\top}} h}
    \leq 8 
    \max \left\{
    \sqrt{\frac{5d + 2\ln\left(\frac{6}{\delta}\right)}{m}},
    \frac{5d + 2\ln\left(\frac{6}{\delta}\right)}{m}
    \right\}.
\end{align}

Taking the absolute value of \eqref{eq:reg_uc_proof0} and using the triangle inequality, the union bound, and inequalities  \eqref{eq:reg_uc_proof1}, \eqref{eq:reg_uc_proof2}, and \eqref{eq:reg_uc_proof3}, we get that  
\begin{align}
    \Prob \br*{\forall h \in \calH, \Delta_S(h) \leq \uc\br{m,\delta}}  \geq 1 - \delta, \nn
\end{align}
where we defined
\begin{align} \label{eq:uc_linear_reg}
     \uc\br{m,\delta} &\defeq  \sqrt{\frac{\ln(\nf{6}{\delta})}{32 m }}
     + \sqrt{\frac{d + 2 d \sqrt{\ln(\nf{3}{\delta})} +2 \ln(\nf{3}{\delta})} {4 m}}   \\
     &+ 8 \max \left\{
    \sqrt{\frac{5d + 2\ln\left(\frac{6}{\delta}\right)}{m}},
    \frac{5d + 2\ln\left(\frac{6}{\delta}\right)}{m}
    \right\} .\nn
\end{align}
Therefore, $\uc(m,\delta) \in O\br*{\sqrt{\frac{
    d \br*{1 + \ln(\nf{1}{\delta}})}{m}}}$.

Next, we wish to prove a uniform convergence bound for the loss gradient, in Euclidean norm. Note that for any $h \in \calH$
\begin{align} \label{eq:reg_uc_grad_loss}
       & \E_{z\sim\calD}  \nabla_h\ell(h,z) - \frac{1}{m}\sum_{i=1}^{m}
    \nabla_h\ell(h,z_i) \\
       &= \E_{z\sim\calD}  \nabla_h \frac{1}{4} (h^\top x - y)^2 - \frac{1}{m}\sum_{i=1}^{m}
    \nabla_h \frac{1}{4} (h^\top x_i - y_i)^2  \nn \\
      &= \frac{1}{2} \E  (h^\top x - y) x^\top - \frac{1}{m} \sum_{i=1}^m \frac{1}{2} (h ^\top x_i - y_i)  x_i^\top  \nn \\
     &= \frac{1}{2} h^\top \br*{\E x x^\top -  \frac{1}{m} \sum_{i=1}^m  x_i x_i^\top}
     -  \frac{1}{2} \br*{\E y x^\top - \frac{1}{m} \sum_{i=1}^m y_i x_i^\top  }. \nn 
\end{align}
To bound the $L_2$ norm of the first term of the equation above, we use similar argument as in \eqref{eq:reg_uc_proof_xx}, and the fact that the operator-norm can be defined equivalently by $\forall A \in \R^{d \times d}, \normop{A} = \sup_{v \in \Sd_1} \norm{ A v}_2$, to get that w.p.~of at least $ 1 - \nf{\delta}{2}$
\begin{align} \label{eq:reg_uc_proof_xx_2}
  \forall h \in \Bir, \norm{\frac{1}{2} h^\top \br*{\frac{1}{m} 
\sum_{i=1}^{m} x_i x_i^{\top} -\E\br*{x x^\top}}}_2
    \leq 16  r
    \max \left\{
    \sqrt{\frac{5d + 2\ln\left(\frac{4}{\delta}\right)}
    {m}},
    \frac{5d + 2\ln\left(\frac{4}{\delta}\right)}{m}
    \right\} . 
\end{align}

To bound the $L_2$ norm of the second term, we use the same argument as in \eqref{eq:reg_uc_proof_xy}, and get
\begin{equation} \label{eq:reg_uc_proof_xy_2}
    \Prob \br*{\frac{1}{2}\norm{\E y x  -\frac{1}{m}\sum_{i=1}^m y_i x_i}_2 <  \frac{r}{2 \sqrt{m}}\sqrt{ d + 2 d \sqrt{\ln(\nf{2}{\delta})} +2 \ln(\nf{2}{\delta})}} \geq  1 - \nf{\delta}{2}.
\end{equation}

Now, taking the norm of equality \eqref{eq:reg_uc_grad_loss} and using  the triangle inequality, inequalities \eqref{eq:reg_uc_proof_xx_2} and \eqref{eq:reg_uc_proof_xy_2}, and the union bound we get that
\begin{align} \label{eq:reg_uc_grad_loss_bound}
       &\Prob \br*{\forall h \in  \calH , \norm{\E_{z\sim\calD}  \nabla_h\ell(h,z) - \frac{1}{m}\sum_{i=1}^{m}
    \nabla_h\ell(h,z_i)}_2 \leq \ucg\br{m,\delta}} \geq 1 - \delta, 
\end{align}
where we defined
\begin{align} \label{eq:uc_grad_linear_reg}
    \ucg\br{m,\delta} \defeq
    16 r \max \left\{
    \sqrt{\frac{5d + 2\ln\left(\frac{4}{\delta}\right)}
    {m}},
    \frac{5d + 2\ln\left(\frac{4}{\delta}\right)}{m}
    \right\} + r \sqrt{\frac{ d + 2 d \sqrt{\ln(\nf{2}{\delta})} + 2 \ln(\nf{2}{\delta})}{4 m}}. 
\end{align}
Therefore, $\ucg(m,\delta) \in O\br*{r \sqrt{\frac{
    d \br*{1 + \ln(\nf{1}{\delta}})}{m}}}$.

Notice that the loss is bounded in $[0,1]$, since
\begin{equation*}
    \ell(x, y, h) = \frac{1}{4} (h^\top x - y)^2
    \leq \frac{1}{4} 2 \br*{ (h^\top x)^2 +  y^2 }
    \leq \frac{1}{2} \br*{\norm{h}_2^2 \norm{x}_2^2 + 1 }  \leq 1.
\end{equation*}
Therefore we can use Theorem  \ref{thm:WPB_uc_grad} to conclude the proof. 
\end{proof}

%%%%%%%%%%%%%%%%%%%%%%%%%%%%%%%%
\section{Appendix: An Example of a Seeger Type Bound} \label{sect:appendox_seeger}
%%%%%%%%%%%%%%%%%%%%%%%%%%%%%%%%

To derive an analogous Seeger's type theorem to Thm.\ \ref{thm:TVPB}, we need to prove uniform convergence of the kl-gap, $\DeltaKL_{S}(h) \defeq \kld{\Lhat_{S}(h)}{L_{D}(h)}$, rather than the usual gap $\Delta_{S}(h) = L_{D}(h) -  \Lhat_{S}(h)$.

For example, consider the binary classification and finite $\calH$ case.

For each $h \in \calH$, we bound $\DeltaKL_{S}(h) =\kld{\Lhat_{S}(h)}{L_{D}(h)}$ using the concentration inequality from \citet{dembo2009ldp} Thm.\ 2.2.3.  (see also \citet{Mardia19} Lem.\ 8), which holds since $\Lhat_{S}(h)$ is an empirical average of $m$ Bernoulli i.i.d variables with mean $L_D(h)$. For any $\varepsilon > 0$, we have
\begin{align}
   \Prob \br*{\DeltaKL_{S}(h) < \varepsilon} \geq 1 - 2 \exp \br*{-m \varepsilon}. \nn
\end{align}

Using a union bound argument we get
\begin{align}
   \Prob \br*{\forall h \in \calH, \DeltaKL_{S}(h) < \varepsilon} \geq 1 - 2 \abs*{\calH} \exp \br*{-m \varepsilon} . \nn
\end{align}
Therefore, for any $\delta \in (0,1)$ we can get
\begin{equation} \label{eq:UCon_DeltaKL}
       \Prob \br*{\forall h \in \calH, \DeltaKL_{S}(h) < \frac{\ln \br*{\nf{2 \abs{\calH}}{\delta}}}{m}  } \geq 1 -\delta.
\end{equation}

Let $\calFinfty_{\ln \br*{\nf{4 \abs{\calH}}{\delta}}}$  be the family of functions as defined in Eq. \ref{eq:F_bound_innfty}, i.e., functions that are bounded in the $\infty$-norm by $\ln \br*{\nf{4 \abs{\calH}}{\delta}}$, for which the IPM is $\ln \br*{\nf{4 \abs{\calH}}{\delta}} \DTV(Q,P)$.

By \eqref{eq:UCon_DeltaKL}, w.p.\ at least $1 - \nf{\delta}{2}$ we have that $ m  \DeltaKL_{S}(h) \in \calFinfty_{\ln \br*{\nf{4 \abs{\calH}}{\delta}}}$  

Now we can use the Seeger's type IPM-PB bound (Prop. \ref{prop:Seeger_IPM_PB}) and a union bound argument,
to get that with probability at least
$1-\delta$ over the samples $S \sim \calD^m$, the following inequality holds for all $Q \in \calM(\calH)$
\begin{align}
\Delta_{S}(Q) & \le \sqrt{2 \Lhat_{S}(Q) 
\frac{\ln \br*{\nf{4 \abs{\calH}}{\delta}} \DTV(Q,P) + \ln(\nf{{4\sqrt{m}}}{\delta})}{m}}
+ 2 \frac{\ln \br*{\nf{4 \abs{\calH}}{\delta}} \DTV(Q,P) + \ln(\nf{{4\sqrt{m}}}{\delta})}{m}. \nn
\end{align}

  %%%%%%%%%%%%%%%%%%%%%%%%%%%%%%%%
%%%%%%%%%%%%%%%%%%%%%%%%%%%%%%%%
 {{\placeholdera}
 \section{Appendix: Numerical Demonstration Details
 } \label{sect:expriment}
 %%%%%%%%%%%%%%%%%%%%%%%%%%%%%%%%
 %%%%%%%%%%%%%%%%%%%%%%%%%%%%%%%%
This section describes the experiment that implements the setting of Corollary \ref{cor:wpb_regression}  (Wasserstein-PB Bound for Linear Regression). 
The code is available at:  
\url{https://github.com/ron-amit/pac_bayes_reg}. 

\paragraph{The sample distribution.} 
The unknown data distribution $\calD$ is determined by a latent vector $g \in \R^d$, drawn once per experiment instance from a uniform distribution over $\Ball_{0.1}$. The dimension is $d=10$.
For each sample $(x,y) \sim \calD$, $x$ is drawn uniformly from $\Ball_{0.1}$ and $y=f(x)$ is set by $f(x) = \clip_{[-1,1]}\brc{g ^\top x  + \xi}$ where, 
$$
\clip_{[a,b]}(t) \defeq 
\begin{cases}
 a, &t<a \\
 t, & a \leq t \leq b \\
 b & t > b ,
\end{cases}
$$ 
for any $a,b,t \in \R$, and $\xi$ is drawn uniformly from $[-0.5, 0.5]$.
{{\placeholdera}The motivation for this choice of $\calD$ is to have an underlying linear structure in the data, corrupted by noise. The clipping ensures that the loss values are in the range $[0,1]$.}

 \paragraph{The prior and posterior distributions.}
 The hypothesis space is an $r$-radius ball $\calH = \Ball_r$, with $r = 1$.
 The prior and posterior distributions over $\calH$ are set as projected Gaussian distributions.
  Let $P_{\Ball} :\R^d \rightarrow \Ball_r$ be a projection operator onto $\Ball_r$
     Let $\Ptild$ be a Gaussian  measure over $\R^d$, $\calN(\mu_P, \sigma^2_P I)$, where $\mu_P = \zeroVec$, and $\sigma_P${{\placeholdera} is a fixed constant that} will be specified later.
     The prior is defined as  $P = P_{\Ball} \sharp \Ptild$, i.e.~, as the push-forward measure of $\Ptild$ under the projection $P_{\Ball}$.
 The family of posteriors we are considering are projected Gaussian distributions, $\calQ \defeq  \brc*{P_{\Ball} \sharp \tildQ: \tildQ = \calN(\mu_Q, \sigma^2_Q I), \mu_Q, \in \Ball_{r_Q}}$, where $\sigma_Q${{\placeholdera} is a fixed constant that} will be specified later and the maximal norm of $\mu_Q$ is $r_Q = 0.05$.

 \textbf{The Wasserstein distance.}
 Since there is no closed-form formula for the $1^{\mbox{st}}$ order Wasserstein distance between Gaussian distributions projected onto a ball $W_1(Q,P)$, we will instead use an upper bound.
 We use Lemma \ref{lem:W_gauss_proj} (Sect. \ref{sect:tech_lemmas}) to bound this distance with the distance of the corresponding pre-projection measures, $W_1(\Qtild,\Ptild)$,   where $\tildQ$ and $\tildP$ are the corresponding pre-projection measures.
Note that our choice of parameters ensures that the lemma condition holds:
 \begin{align*}
     r^2 \geq  \max \brc*{\norm{\mu_Q}_2^2 + \norm{\Sigma_Q}_F^2,  \norm{\mu_P}_2^2 + \norm{\Sigma_P}_F^2}  = \max \brc*{ \norm{\mu_Q}_2^2 + d \sigma^2_Q,  \norm{\mu_P}_2^2 + d \sigma^2_P}.
 \end{align*}
 We also use the fact that $W_1(\Qtild,\Ptild) \leq W_2(\Qtild,\Ptild)$ ( \citet{givens1984class}, Prop.~3) and the analytic formula for the $2^{\mbox{nd}}$ order Wasserstein distance between two Gaussian distributions (\citet{givens1984class}, Prop.~7) to finally get a closed-form upper bound, 
 \begin{align} \label{eq:W_Gauss_upper}
    W_1(Q,P) &\textrel{Lem.~\ref{lem:W_gauss_proj}}{\leq}  \sqrt{\norm{\mu_Q - \mu_P}_2^2 + \Tr \br*{\Sigma_Q + \Sigma_P - 2 \br*{\Sigma_Q^{\nf{1}{2}} \Sigma_P \Sigma_Q^{\nf{1}{2}} }^{\nf{1}{2}} }} \\
    & + \sqrt{\frac{\pi}{2} \norm{\Sigma_Q}_{2,2}} \erfc \br*{\frac{r - \sqrt{\norm{\mu_Q}_2^2 + \norm{\Sigma_Q}_F^2}}{\sqrt{2 \norm{\Sigma_Q}_{2,2}} }} \nn\\
    & + \sqrt{\frac{\pi}{2} \norm{\Sigma_P}_{2,2}} \erfc \br*{\frac{r - \sqrt{ \norm{\mu_P}_2^2 + \norm{\Sigma_P}_F^2}}{\sqrt{2 \norm{\Sigma_P}_{2,2}} }} \nn\\
    &= \sqrt{\norm{\mu_Q - \mu_P}_2^2 + d \br*{\sigma_Q-\sigma_P}^2 } \nn\\
    & + \sqrt{\frac{\pi}{2}} \sigma_Q \erfc \br*{\frac{r - \sqrt{\norm{\mu_Q}_2^2 + d \sigma_Q^2}}{\sqrt{2} \sigma_Q}} 
     + \sqrt{\frac{\pi}{2}} \sigma_P  \erfc\br*{\frac{r - \sqrt{ \norm{\mu_P}_2^2 + d \sigma_P^2}}{\sqrt{2} \sigma_P }}\nn \\
    &\defeq W_{\text{bound}}(\mu_Q). \nn
 \end{align}
 Notice that in the limit of $\sigma_Q,\sigma_P \to 0$, the bound becomes $\norm{\mu_Q - \mu_P}_2$, which is equivalent to the Wasserstein distance between two Dirac measures at $\mu_Q$ and $\mu_P$. 
 
  \paragraph{The empirical risk term.} 
 To compute the expectation of the empirical risk w.r.t.~the posterior, $\E_{h\sim Q} \Lhat(h)$, we derive a closed-form formula using the structure and of the loss and the posterior distribution
 \footnote{In cases where the loss is a more complicated function (but still differentiable), one can approximate the expectation over the posterior with the reparametrization trick \citep{kingma2013auto}, similarly to  \citet{dziugaite2017computing, Amit_Meir_18}.}
 .
 Given a dataset $S=\brc*{(x_i,y_i)}_{i=1}^m$, denote $X \in \R^{m \times d}$ as a matrix whose rows are the vectors $x_i$, and denote $Y \in \R^{m \times 1}$ as a vector whose entries are $y_i$.
  Denote $\Jhat_{(X,Y)}(\mu_Q) \defeq \E_{h\sim Q} \Lhat(h)$.
  Then we have
  \begin{align} \label{eq:empric_risk_lin_reg}
      \Jhat_{(X,Y)}(\mu_Q) &= \E_{h \sim \calN(\mu_Q, \sigma^2_Q  I )}\frac{1}{m} \sum_{i=1}^{m}\frac{1}{4} (h^\top x_i - y_i)^2 \\
      &= \frac{1}{4 m} \E_{h \sim \calN(\mu_Q, \sigma^2_Q  I )} \norm{X h^\top - Y}_2^2 
      \nn \\
     &= \frac{1}{4 m} \E_{\epsilon \sim \calN(0,I)} \norm{X \br*{\mu_Q + \sigma_Q \epsilon}^\top - Y}_2^2  \nn \\
  &= \frac{1}{4 m} \E_{\epsilon \sim \calN(0,I)} \norm{\sigma_Q X\epsilon^\top + X \mu_Q^\top  -  Y}_2^2  \nn \\
   &= \frac{1}{4 m} \br*{\norm{ X \mu_Q^\top  -  Y}_2^2 + \E_{\epsilon \sim \calN(0,I)}  \sigma_Q^2 \Tr \br{ \epsilon X^\top X\epsilon^\top }}  \nn \\
   &= \frac{1}{4 m} \br*{\norm{ X \mu_Q^\top  -  Y}_2^2 + \sigma_Q^2 \Tr \br{ X^\top X}}  \nn \\
     &= \frac{1}{4 m} \br*{\norm{ X \mu_Q^\top  -  Y}_2^2 + \sigma_Q^2 \norm{X}_F^2} . \nn  
  \end{align}
    %   &= \frac{1}{4 m} \br*{\norm{X \mu_Q^\top - Y}_2^2 + \Tr \br*{\sigma^2_Q X X^\top}}

 \paragraph{The explicit Wasserstein-PB bound.}
 According to Cor.~\ref{cor:wpb_regression}, given a training set $S = (X,Y)$, the upper bound on the expected risk $L(Q) \defeq  \E_{h \sim Q} L(h)$  is
 \begin{equation} \label{eq:explicit_WPB_linReg}
     J^{\text{WPB}}_{(X,Y)}(\mu_Q) \defeq
     \Jhat_{(X,Y)}(\mu_Q) + \sqrt{2 \uc\br{m, \nf{\delta}{4}} \cdot  \ucg\br{m, \nf{\delta}{4}} \cdot W_{\text{bound}}(\mu_Q) + \frac{\ln(\nf{2 m}{\delta})}{2(m-1)}},
 \end{equation}
where  $\uc \br{m,\delta}$ is defined in
\eqref{eq:uc_linear_reg}, $\ucg\br{m,\delta}$ is defined in
\eqref{eq:uc_grad_linear_reg}, and $W_{\text{bound}}(\mu_Q)$ is the upper bound over $W_1(Q,P)$ defined in \eqref{eq:W_Gauss_upper}.
 
 \textbf{The explicit KL-PB bound.}
  For the KL-PB bound, we use the classic PB bound (Prop.~\ref{prop:PACBayesBound}) with the KL-divergence replaced by an upper bound that has a closed-form expression. By the data-processing inequality we have that $\KL{Q}{P} = \KL{P_{\Ball} \sharp \Qtild}{P_{\Ball} \sharp \Ptild} \leq \KL{\Qtild}{ \Ptild}$, and $\KL{\Qtild}{\Ptild}$ can be computed using the analytic formula for the KL-divergence between two Gaussian distributions. Therefore, the upper bound we use is
 \begin{align} \label{eq:explicit_KLPB_linReg}
       J^{\text{KL-PB}}_{(X,Y)}(\mu_Q) &\defeq
     \Jhat_{(X,Y)}(\mu_Q) + \sqrt{\frac{ \frac{\norm{\mu_Q - \mu_P}_2^2}{2 \sigma^2_P} + d \br*{\ln \br*{\frac{\sigma_P}{\sigma_Q}} + \frac{\sigma^2_Q }{2 \sigma^2_P} - \frac{1}{2}} + \ln(\nf{m}{\delta})}{2(m-1)}} .
 \end{align}

 \paragraph{Experiment Procedure:}
 We repeat the experiment for $10$ repetitions, to account for the randomness of the data and optimization in each run. 
 In each run, \textit{(i)} the task data distribution $\calD$ is generated as described above, \textit{(ii)} A training set of $m$ samples is generated.
  \textit{(iii)} The posterior mean vector $\mu_Q$ is learned using the Adam Optimizer \citep{kingma2014adam} that minimizes either $J^{\text{KL-PB}}_{(X,Y)}(\mu_Q)$ or $J^{\text{WPB}}_{(X,Y)}(\mu_Q)$ (as will be specified later), where the learning rate is set as $10^{-3}$, and the maximal batch size is $256$.
  The gradients are computed using automatic differentiation by the PyTorch framework \citep{paszke2019pytorch}. After each gradient step, the parameter $\mu_Q$ is projected to $\Ball_{r_Q}$.

\paragraph{Results.}
Table \ref{table:results_sigmaP_10_2} show the results when we set the prior parameter $\sigma_P$ as $10^{-2}$, and Table \ref{table:results_sigmaP_10_4} shows the results for $\sigma_P = 10^{-4}$, both use $\sigma_Q=10^{-3}$.
The optimization objective for those two setups is the KLPB bound, $J^{\text{KL-PB}}_{(X,Y)}(\mu_Q)$.

The third setup (Table \ref{table:results_sigmaP_0}) investigates Dirac posteriors (``a deterministic model''). In this setup we set $\sigma_Q = \sigma_P = 0$, and the optimization objective is set to be $J^{\text{WPB}}_{(X,Y)}(\mu_Q)$.
Note that since $\sigma_P =0$ then the KL-divergence is undefined, while the $W_1$ distance equals exactly $\norm{\mu_Q - \mu_P}_2$.

Figures \ref{fig:lin_reg_a}, \ref{fig:lin_reg_b} and \ref{fig:lin_reg_c} show the corresponding plots.
The `Training loss' column shows the empirical risk \eqref{eq:empric_risk_lin_reg}, i.e.,~the averaged loss of the learned posterior on the training data.
 The `Test loss' column shows the average loss of the learned posterior on a separate `test' set of $10000$ samples drawn from $\calD$.
In all the evaluated bounds, we use the confidence parameter $\delta=0.05$.
The `UC bound' shows the sum of the empirical risk and the UC generalization gap bound \eqref{eq:uc_linear_reg}.  
The `WPB bound' is the Wasserstein-PB bound evaluated by equation \eqref{eq:explicit_WPB_linReg}, and the `KLPB bound' is evaluated by equation \eqref{eq:explicit_KLPB_linReg}.
The results clearly show the improved tightness of the WPB bound over the UC bound, for the two choices of a prior distribution.
The KLPB bound, also shows relatively tight values, as expected from an algorithm- and data-dependent bound. 
However, for the narrower prior distribution ($\sigma_P=10^{-4}$), the KLPB bound is significantly looser than the WPB bound. That is expected from the properties of the KL-divergence, which can tend to $\infty$ if $\sigma_P \to 0$, as opposed to the Wasserstein distance.
In the extreme case of $\sigma_P = 0$ the KLPB bound is undefined, while the WPB exhibits a considerable improvement over the UC bound. 
The results confirm that the WPB generally improves over UC bounds, and may be tighter than the KLPB bound, depending on the prior and posterior distributions.

\begin{table}[t]
\caption{Linear regression experiment with $\mathbf{\sigma_P =10^{-2}, \sigma_Q=10^{-3}}$. Each cell shows the mean over $10$ independent runs and the 95\% confidence interval in parenthesis. }
\label{table:results_sigmaP_10_2}
\centering
\begin{tabular}{llllll}
\toprule
\# samples  & Train risk & Test risk & UC bound & WPB bound & KLPB bound \\
\midrule
100 &    0.0211 (0.0010) &    0.0208 (0.0001) &    6.6176 (0.0010) &    2.2652 (0.0010) &    0.3861 (0.0010) \\
200 &    0.0206 (0.0009) &    0.0208 (0.0001) &    4.6850 (0.0009) &    1.6080 (0.0009) &    0.2814 (0.0009) \\
300 &    0.0214 (0.0006) &    0.0209 (0.0001) &    3.8298 (0.0006) &    1.3177 (0.0006) &    0.2357 (0.0006) \\
400 &    0.0205 (0.0005) &    0.0208 (0.0001) &    3.3187 (0.0005) &    1.1433 (0.0005) &    0.2070 (0.0005) \\
\bottomrule
\end{tabular}
\end{table}

\begin{table}[t]
\caption{Linear regression experiment with $\mathbf{\sigma_P =10^{-4}, \sigma_Q=10^{-3}}$. Each cell shows the mean over $10$ independent runs and the 95\% confidence interval in parenthesis.}
\centering
\label{table:results_sigmaP_10_4}
\begin{tabular}{llllll}
\toprule
\# samples & Train risk & Test risk & UC bound & WPB bound & KLPB bound \\
\midrule
100 &    0.0211 (0.0010) &    0.0208 (0.0001) &    6.6176 (0.0010) &    0.7569 (0.0010) &    1.5787 (0.0010) \\
200 &    0.0206 (0.0009) &    0.0208 (0.0001) &    4.6850 (0.0009) &    0.5424 (0.0009) &    1.1199 (0.0009) \\
300 &    0.0214 (0.0006) &    0.0209 (0.0001) &    3.8298 (0.0006) &    0.4482 (0.0006) &    0.9186 (0.0006) \\
400 &    0.0205 (0.0005) &    0.0208 (0.0001) &    3.3187 (0.0005) &    0.3906 (0.0005) &    0.7974 (0.0005) \\
\bottomrule
\end{tabular}
\end{table}

\begin{table}[t]
\caption{Linear regression experiment with $\mathbf{\sigma_P =0, \sigma_Q=0}$. Each cell shows the mean over $10$ independent runs and the 95\% confidence interval in parenthesis.}
\centering
\label{table:results_sigmaP_0}
\begin{tabular}{llllll}
\toprule
 \# samples & Train risk & Test risk & UC bound & WPB bound & KLPB bound \\
\midrule
100 &    0.0211 (0.0010) &    0.0208 (0.0001) &    6.6176 (0.0010) &    0.3175 (0.0177) &       undefined \\
200 &    0.0206 (0.0009) &    0.0208 (0.0001) &    4.6850 (0.0009) &    0.2363 (0.0136) &       undefined \\
300 &    0.0214 (0.0006) &    0.0209 (0.0001) &    3.8298 (0.0006) &    0.1989 (0.0087) &       undefined \\
400 &    0.0205 (0.0005) &    0.0208 (0.0001) &    3.3187 (0.0005) &    0.1824 (0.0127) &       undefined \\
\bottomrule
\end{tabular}
\end{table}

\begin{figure}[ht]
\centering
\begin{subfigure}{.33\textwidth}
  \centering
  \includegraphics[width=1.\linewidth]{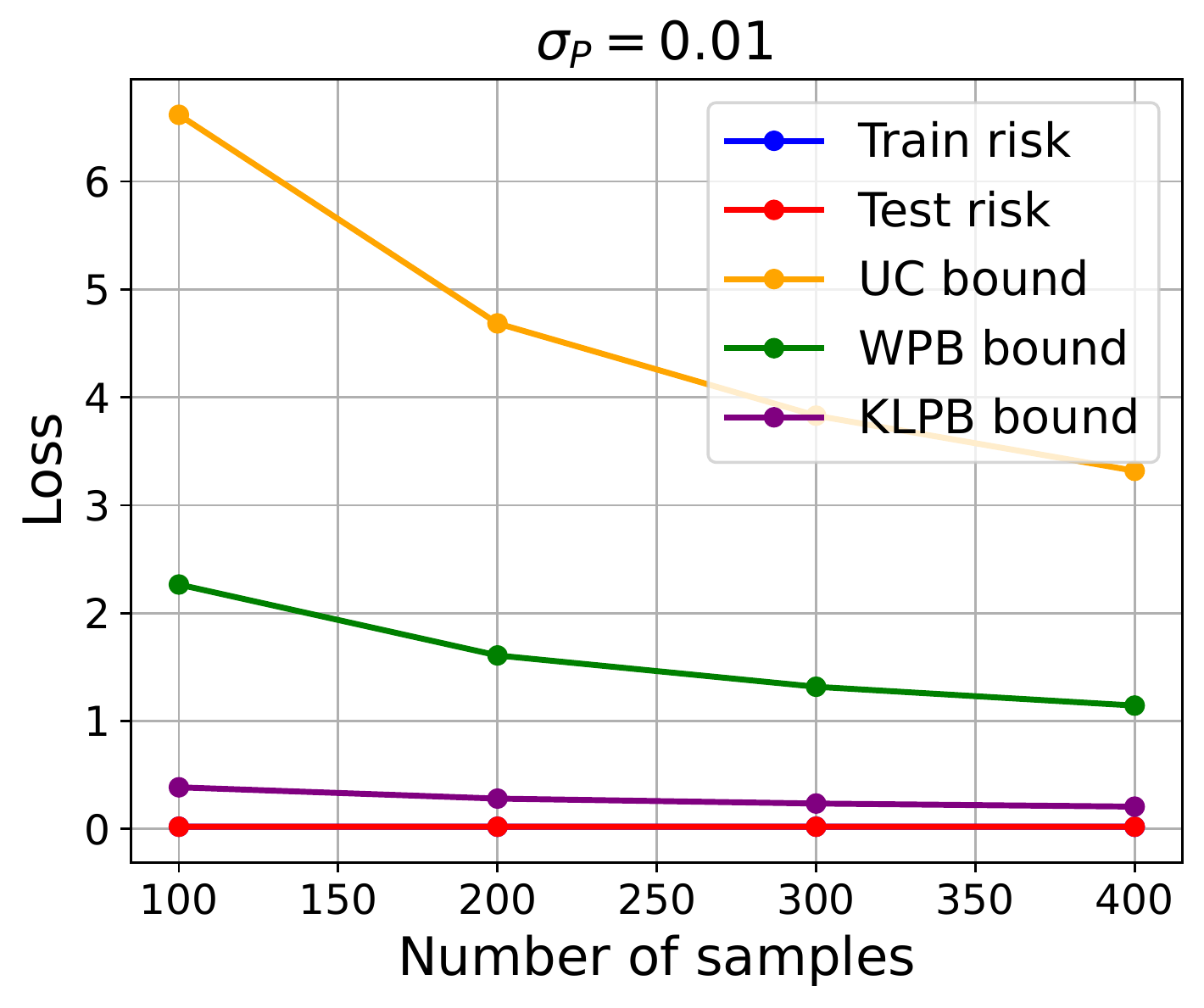}
  \caption{$\mathbf{\sigma_P =10^{-2}, \sigma_Q=10^{-3}}$}
\label{fig:lin_reg_a}
\end{subfigure}%
\begin{subfigure}{.33\textwidth}
  \centering
  \includegraphics[width=1.\linewidth]{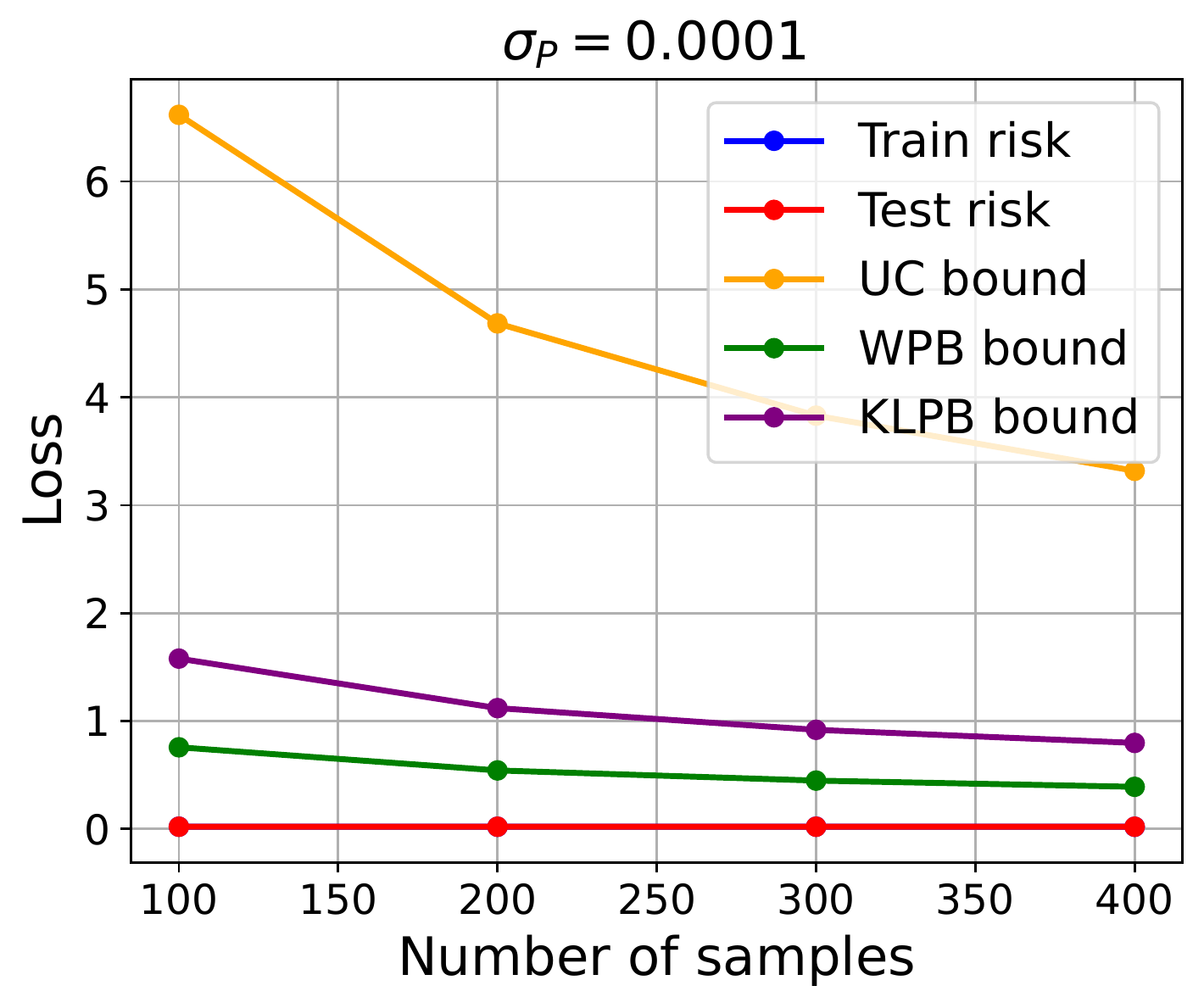}
  \caption{$\mathbf{\sigma_P =10^{-4}, \sigma_Q=10^{-3}}$}
  \label{fig:lin_reg_b}
\end{subfigure}
\begin{subfigure}{.33\textwidth}
  \centering
  \includegraphics[width=1.\linewidth]{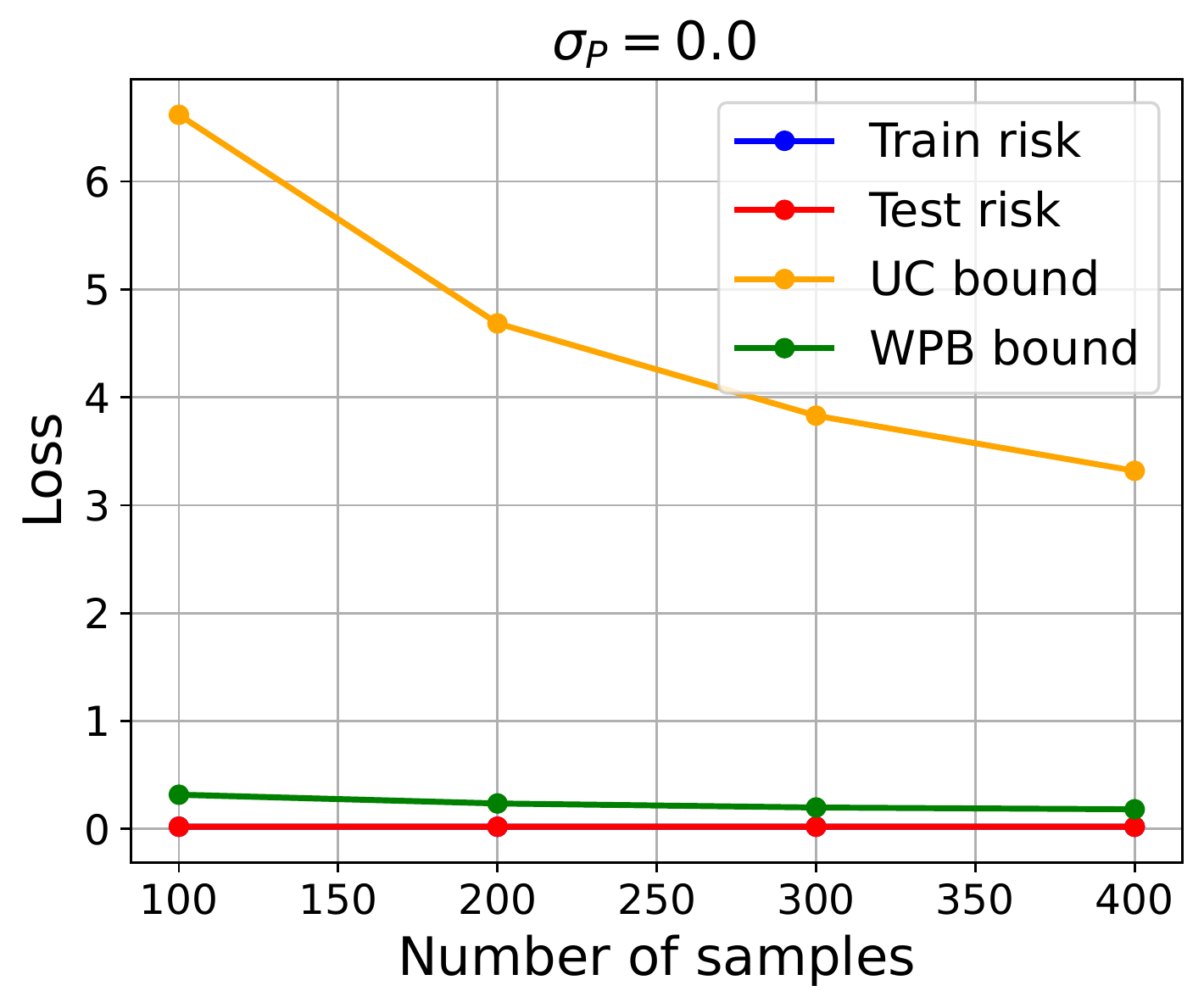}
  \caption{$\mathbf{\sigma_P =\sigma_Q=0}$}
  \label{fig:lin_reg_c}
\end{subfigure}
\caption{Linear regression experiment. Note that the the 95\% confidence interval is too small to be discernible in the plots,{{\placeholdera} and that blue train risk plot is not visible since it is very close to the test risk.}}
\label{fig:lin_reg}
\end{figure}

 %%%%%%%%%%%%%%%%%%%%%%%%%%%%%%%%%%%%%%%%%%%%%%%%%%%%%%%%%%%%%%%%%%%%%%%%%%%%%%%%%%%%%%%%%%%%%% 

\section{Appendix: Technical Lemmas} \label{sect:tech_lemmas}

\begin{lem} \label{lem:sup_f}
 Let $A\subset \mathbb{R}$ be bounded and non-empty, and let $f:(0,\infty) \to \mathbb{R}$ be continuous and monotone non-decreasing.
Then $f(\sup A) = \sup f(A)$, where we defined $\sup A \defeq \sup_{a \in A}a$, and $\sup f(A) \defeq  \sup_{a\in A} f(a)$; that is, $f(A)$ is the image of the set $A$ under $f$.
\end{lem}
\begin{proof}
First notice that $f(\sup A) \geq \sup f(A)$ by monotonicity, because $a \leq \sup A$ for all $a \in A$.

For the other inequality, $f(\sup A) \leq \sup f(A)$ we need to also use continuity:
let $\varepsilon > 0$; by continuity there exists $\delta>0$ such that
for every $a \in A$ such that $a \ge \sup A - \delta$ it holds that $f(a) \geq f(\sup A) - \varepsilon$ (there exists such $a$ by the definition of the supremum).
By monotonicity this implies that $\sup f(A) \geq f(\sup A) - \varepsilon$.
Since the latter inequality holds for every $\varepsilon > 0$,
we conclude that $\sup f(A) \geq f(\sup A)$ as required.
\end{proof}

\begin{lem}[Wasserstein distance between truncated Gaussian distributions] \label{lem:W_gauss_proj}
Let $X^{(1)}$ and $X^{(2)}$ be the Gaussian random vectors in $\R^d$, with distributions $\calN(\mu_1, \Sigma_1)$, and $\calN(\mu_1, \Sigma_2)$ respectively.
Let $P_{\Ball_r} :\R^d \rightarrow \Ball_r$ be a projection operator onto the an $r$-radius ball around the origin, $P_{\Ball_r}(x) \defeq \argmin_{x' \in \Ball_r} \norm{x - x'}_2$, where $\Ball_r = \brc*{x \in \R^d:\norm{x}_2 \leq r}$. 
Assume that $r \geq \sqrt{ \norm{\mu_j}_2^2 + \norm{\Sigma_j}_F^2}$ for $j=1,2$.
Denote the distribution measures of $X^{(1)}$ and $X^{(2)}$  as $\nu_1$ and $\nu_2$ respectively. 
Let $P_{\Ball_r}\sharp\nu_1$ and $P_{\Ball_r}\sharp\nu_2$ be the push-forward measures of $\nu_1$ and $\nu_2$, respectively, under the operator $P_{\Ball_r}$.
Then 
\begin{equation}
    W_1(P_{\Ball_r}\sharp\nu_1, P_{\Ball_r}\sharp\nu_2) \leq W_1(\nu_1,\nu_2) + \sum_{j=1}^{2} \sqrt{\frac{\pi}{2} \norm{\Sigma_j}_{2,2}} \erfc \br*{\frac{r - \sqrt{ \norm{\mu_j}_2^2 + \norm{\Sigma_j}_F^2}}{\sqrt{2 \norm{\Sigma_j}_{2,2}} }},\nn
\end{equation}
where $W_1(Q,P)$ denotes the $1^{\mbox{st}}$ order
Wasserstein distance with the $L_2$ metric.
\end{lem}

\begin{proof}
Using the triangle inequality of Wasserstein distances \citep{clement2008elementary, thorpe2018introduction} twice we get
\begin{equation} \label{eq:W_triang_1}
    W_1(P_{\Ball_r}\sharp\nu_1, P_{\Ball_r}\sharp\nu_2) \leq   W_1(P_{\Ball_r}\sharp\nu_1, \nu_1) + W_1(\nu_1,\nu_2) +    W_1(\nu_2, P_{\Ball_r}\sharp\nu_2) . 
\end{equation}

Notice that 
\begin{align}\label{eq:W_dist_proj_bnd_1}
    W_1(\nu_2, P_{\Ball_r}\sharp\nu_2)  &=  \inf_{\gamma\in\Gamma( \nu_2, P_{\Ball_r}\sharp\nu_2)} \int_{\R^d \times \R^d} \norm{x - x'}_2 \dd \gamma(x, x') \\
    &\leq \int_{\R^d} \norm{x - P_{\Ball_r}(x)}_2  \dd \nu_2(x) \nn\\
    &\textrel{(i)}{=}\int_{\R^d} \brs*{ \norm{x}_2 - r}_{+}    \dd \nu_2(x)  \nn\\
      &= \E \brc*{ \brs*{\norm{X_2}_2 - r}_{+} } \nn\\
         &\textrel{(ii)}{=} \E \brc*{ \brs*{V - r}_{+} } \nn\\
  &\textrel{(iii)}{=}   \int_{t=0}^{\infty} \Prob \brc*{\brs*{V - r}_{+} > t} \dd t\nn\\
 &\textrel{(iv)}{=}  \int_{t=0}^{\infty} \Prob \brc*{V > r + t} \dd t , \nn
\end{align}
where in \textit{(i)} we used the notation $\brs*{t}_{+}  \defeq
\begin{cases} 
 t & t > 0\\
 0 & t \leq 0 
\end{cases}$ 
  and the equality holds since $P_{\Ball_r}(x) = x$ for $\norm{x} \leq r$, and $\norm{P_{\Ball_r}(x) - x} = \norm{x}_2 - r $ for $\norm{x} > r$ (by the properties of the projection onto the a $L_2$ ball).
In \textit{(ii)} we use the definition of the random variable
$ V \defeq \norm{X^{(2)}}_2$.
 In \textit{(iii)} we used the
 tail sum formula for the expectation.
Equality \textit{(iv)} holds since the corresponding events are equivalent.

Let $Z$ be a standard Gaussian random vector in $\R^d$, i.e.,~$Z \sim \calN(\zeroVec, I)$. 

By \citet{wainwright2019high}, Thm.~2.26 (one-sided variant) we have that 
\begin{align*}
 \Prob \brc*{f(Z) - \E \brs*{f(Z) } \geq s} \leq \exp \br*{- \frac{s^2}{2 L^2}},
\end{align*}
for any $s \geq 0$ and for any function $f: \R^d \to \R$ that is $L$-Lipschitz w.r.t.~the $L_2$ metric.

In particular, for the function $f(z) \defeq \norm{\mu_2 + \Sigma_2^{\nf{1}{2}} z}_2$ we have for any $s \geq 0$ 
\begin{align} \label{eq:gauss_deviat}
 \Prob \brc*{\norm{\mu_2 + \Sigma_2^{\nf{1}{2}} Z}_2- \E \brs*{\norm{\mu_2 + \Sigma_2^{\nf{1}{2}} Z}_2} \geq s} & \leq \exp \br*{- \frac{s^2}{2 \norm{\Sigma_2^{\nf{1}{2}}}_{2,2}^2}}\\ 
 &\textrel{(i)}{=} \exp \br*{- \frac{s^2}{2 \norm{\Sigma_2}_{2,2}}}, \nn
\end{align}
since $f$ is Lipschitz with constant $\norm{\Sigma_2^{\nf{1}{2}}}_{2,2}$ where  $\norm{\cdot}_{2,2}$ is the operator norm defined by $\norm{A}_{2,2} = \sup_{\norm{x}_2 = 1} 
\norm{A x}_2$ for any $A \in \R^{d \times d}$.
Equality \textit{(i)} holds since 
\begin{equation*}
   \norm{\Sigma_2^{\nf{1}{2}}}_{2,2}^2 = \br*{\sup_{\norm{x}_2 = 1} 
\norm{\Sigma_2^{\nf{1}{2}} x}_2}^2  
=  \sup_{\norm{x}_2 = 1} 
\norm{\Sigma_2^{\nf{1}{2}} x}^2_2   =  \sup_{\norm{x}_2 = 1} 
x^\top  \Sigma_2 x =   \norm{\Sigma_2}_{2,2}.
\end{equation*}

Notice that
\begin{align*}
  \E \brs*{\norm{\mu_2 + \Sigma_2^{\nf{1}{2}} Z}_2} &= 
   \E \brs*{\sqrt{\norm{\mu_2 + \Sigma_2^{\nf{1}{2}} Z}^2_2}} \\
   &\textrel{(i)}{\leq} \sqrt{\E \brc*{ \norm{\mu_2 + \Sigma_2^{\nf{1}{2}} Z}^2_2}} \\
  &= \sqrt{\E \brc*{\br*{\mu_2 + \Sigma_2^{\nf{1}{2}} Z}^\top \br*{\mu_2 + \Sigma_2^{\nf{1}{2}} Z} }} \\
   &\textrel{(ii)}{=} \sqrt{ \norm{\mu_2}_2^2 + \norm{\Sigma_2}_F^2},
\end{align*}
where \textit{(i)} is by Jensen's inequality, and in \textit{(ii)} we used the fact that $Z \sim \calN(\zeroVec, I)$ and $\norm{\cdot}_F$ denotes the Frobenius Norm, defined by $\norm{A}_F \defeq \sqrt{\sum_{i,j} A_{i,j}^2}$.

Therefore, using \eqref{eq:gauss_deviat}, we get
\begin{align*}
 \Prob \brc*{\norm{\mu_2 + \Sigma_2^{\nf{1}{2}} Z}_2-  \sqrt{ \norm{\mu_2}_2^2 + \norm{\Sigma_2}_F^2} \geq s} \leq \exp \br*{- \frac{s^2}{2 \norm{\Sigma_2}_{2,2}}}.
\end{align*}

Note that the random variable $V \defeq \norm{X^{(2)}}_2$ is equal, in distribution, to the random variable $\norm{\mu_2 + \Sigma_2^{\nf{1}{2}} Z}_2 $, 
and therefore we also have for $s \geq 0$ that
\begin{align*}
 \Prob \brc*{V - \sqrt{ \norm{\mu_2}_2^2 + \norm{\Sigma_2}_F^2} \geq s} \leq \exp \br*{- \frac{s^2}{2 \norm{\Sigma_2}_{2,2}}}.
\end{align*}
For any $t \geq 0$, set $s := t + r - \sqrt{ \norm{\mu_2}_2^2 + \norm{\Sigma_2}_F^2} $. Since we assume that $r \geq \sqrt{ \norm{\mu_2}_2^2 + \norm{\Sigma_2}_F^2}$, we have that $s \geq 0$.
Therefore we have
\begin{align*}
 \Prob \brc*{V > r + t} \leq \exp \br*{- \frac{\br*{t + r - \sqrt{ \norm{\mu_2}_2^2 + \norm{\Sigma_2}_F^2} }^2}{2 \norm{\Sigma_2}_{2,2}}}.
\end{align*}
Hence, by \eqref{eq:W_dist_proj_bnd_1} we have
\begin{align*}
     W_1(\nu_2, P_{\Ball_r}\sharp\nu_2) &\leq \int_{t=0}^{\infty} \Prob \brc*{V > r + t} \dd t  \\
     &\leq \int_{t=0}^{\infty} \exp \br*{- \frac{\br*{t + r - \sqrt{ \norm{\mu_2}_2^2 + \norm{\Sigma_2}_F^2} }^2}{2 \norm{\Sigma_2}_{2,2}}} \dd t \\
     &=  \sqrt{\frac{\pi}{2} \norm{\Sigma_2}_{2,2}} \erfc \br*{\frac{r - \sqrt{ \norm{\mu_2}_2^2 + \norm{\Sigma_2}_F^2}}{\sqrt{2 \norm{\Sigma_2}_{2,2}} }}.
\end{align*}
By symmetry we have a similar bound for $W_1(P_{\Ball_r}\sharp\nu_1, \nu_1)$.
To conclude, using \eqref{eq:W_triang_1} we get
\begin{align*}
    W_1(P_{\Ball_r}\sharp\nu_1, P_{\Ball_r}\sharp\nu_2) &\leq   W_1(P_{\Ball_r}\sharp\nu_1, \nu_1) + W_1(\nu_1,\nu_2) +    W_1(\nu_2, P_{\Ball_r}\sharp\nu_2)  \\
    &\leq W_1(\nu_1,\nu_2) +  \sum_{j=1}^{2} \sqrt{\frac{\pi}{2} \norm{\Sigma_j}_{2,2}} \erfc \br*{\frac{r - \sqrt{ \norm{\mu_j}_2^2 + \norm{\Sigma_j}_F^2}}{\sqrt{2 \norm{\Sigma_j}_{2,2}} }}.
\end{align*}
\end{proof}

 } % end color
 %====================================================================
 %%%%%%%%%%%%%%%%%%%%%%%%%%%%%%%%%%%%%%%%%%%%%%%%%%%%%%%%%%%%55
%%%%%%%%%%%%%%%%%%%%%%%%%%%%%%%%%%%%%%%
\end{document}